\documentclass[twoside,11pt]{article}

\usepackage{tikz}
\usetikzlibrary{matrix,chains,positioning,decorations.pathreplacing,arrows}
\usetikzlibrary{shapes, snakes, positioning, arrows,calc}

\newtheorem{fact}{Fact}
\usepackage{algorithm}
\usepackage{algpseudocode}

\usepackage{wrapfig}
\usepackage{dblfloatfix}

\usepackage{caption}
\usepackage[labelformat=simple]{subcaption}

\usepackage{jmlr2e}

\usepackage{hyperref}
\usepackage{url}
\usepackage{amsmath}
\usepackage{amsfonts}
\usepackage{dsfont}
\usepackage{bm}
\usepackage{xcolor}
\usepackage{cleveref}

\usepackage{etoolbox}
\usepackage{nicefrac}

\newcommand{\mbf}[1]{\mathbf{#1}}

\newcommand{\zb}{\mbf{z}}
\newcommand{\yb}{\mbf{y}}

\newcommand{\mub}{\bm{\mu}}

\newcommand{\lambdabhat}{\bm{\rho}}

\newcommand{\bb}{\mbf{b}}
\newcommand{\ub}{\mbf{u}}
\newcommand{\lb}{\mbf{l}}
\newcommand{\xb}{\mbf{x}}

\newcommand{\balpha}{\bm{\alpha}}
\newcommand{\bbeta}{\bm{\beta}}

\newcommand{\bzeta}{\bm{\zeta}}

\newcommand{\ubrv}{\check{U}}
\newcommand{\lbrv}{\check{L}}

\newcommand{\xbhat}{\mbf{\hat{x}}}
\newcommand{\mtt}{\mathtt}

\newcommand{\sig}[1]{\sigma\left(#1\right)}

\DeclareMathOperator*{\argmax}{\arg\!\max}
\DeclareMathOperator*{\argmin}{\arg\!\min}

\newcommand{\xelim}{\xb_{k-1}^0[0]}
\newcommand{\welim}{\boldsymbol{w}_{i, k}[0]}
\newcommand{\wxrem}{\sum_{j>1} \boldsymbol{w}_{i, k}[j] \xb_{k-1}^0[j]}
\newcommand{\wcurr}{\boldsymbol{w}_{i, k}}

\allowdisplaybreaks

\usepackage{enumitem}

\usepackage{stmaryrd}

\usepackage{booktabs}
\usepackage{adjustbox}
\newcommand\Tstrut{\rule{0pt}{2.6ex}}       
\newcommand\Bstrut{\rule[-0.9ex]{0pt}{0pt}} 
\newcommand{\TBstrut}{\Tstrut\Bstrut} 
\usepackage{siunitx}

\usepackage{lastpage} 
\jmlrheading{25}{2024}{1-\pageref{LastPage}}{1/21; Revised 11/22}{1/24}{21-0076}{Alessandro De Palma, Harkirat Singh Behl, Rudy Bunel, Philip H.S. Torr and M. Pawan Kumar} 
\ShortHeadings{Scaling the Convex Barrier with Sparse Dual Algorithms}{De Palma, Behl, Bunel, Torr, and Kumar}

\firstpageno{1}

\begin{document}

\title{Scaling the Convex Barrier with Sparse Dual Algorithms}

\author{\name Alessandro De Palma \email adepalma@robots.ox.ac.uk\\
			\name Harkirat Singh Behl \email harkirat@robots.ox.ac.uk\\
			\name Rudy Bunel \email bunel.rudy@gmail.com\\
			\name Philip H.S. Torr \email phst@robots.ox.ac.uk \\
			\name M. Pawan Kumar \email pawan@robots.ox.ac.uk \\
	        \addr Department of Engineering Science\\
	        University of Oxford\\
	        Oxford OX1 3PJ}

\editor{Simon Lacoste-Julien}

\maketitle

\begin{abstract}
Tight and efficient neural network bounding is crucial to the scaling of neural network verification systems. 
Many efficient bounding algorithms have been presented recently, but they are often too loose to verify more challenging properties.
This is due to the weakness of the employed relaxation, which is usually a linear program of size linear in the number of neurons.
While a tighter linear relaxation for piecewise-linear activations exists, it comes at the cost of exponentially many constraints and currently lacks an efficient customized solver.
We alleviate this deficiency by presenting two novel dual algorithms: 
one operates a subgradient method on a small active set of dual variables, the other exploits the sparsity of Frank-Wolfe type optimizers to incur only a linear memory cost.  
Both methods recover the strengths of the new relaxation: tightness and a linear separation oracle. 
At the same time, they share the benefits of previous dual approaches for weaker relaxations: massive parallelism, GPU implementation, low cost per iteration and valid bounds at any time.
As a consequence, we can obtain better bounds than off-the-shelf solvers in only a fraction of their running time, attaining significant formal verification~speed-ups.
\end{abstract}


\section{Introduction}
Verification requires formally proving or disproving that a given property of a neural network holds over all inputs in a specified domain. We consider properties in their canonical form~\citep{Bunel2018}, which requires us to either: (i) prove that no input results in a negative output (property is true); or (ii) identify a counter-example (property is false). The search for counter-examples is typically performed by efficient methods such as random sampling of the input domain~\citep{Webb2019}, or projected gradient descent~\citep{Kurakin2017,Carlini2017,Madry2018}. In contrast, establishing the veracity of a property requires solving a suitable convex relaxation to obtain a lower bound on the minimum output. 
If the lower bound is positive, the given property is true. 
If the bound is negative and no counter-example is found, either: 
(i) we make no conclusions regarding the property (incomplete verification); or (ii) we further refine the counter-example search and lower bound computation within a branch-and-bound framework until we reach a concrete conclusion (complete verification).

The main bottleneck of branch and bound is the computation of the lower bound for each node of the enumeration tree via convex optimization. While earlier works relied on off-the-shelf solvers~\citep{Ehlers2017,Bunel2018}, it was quickly established that such an approach does not scale-up elegantly with the size of the neural network. This has motivated researchers to design specialized dual solvers~\citep{Dvijotham2019,BunelDP20}, thereby providing initial evidence that verification can be realized in practice. However, the convex relaxation considered in the dual solvers is itself very weak~\citep{Ehlers2017}, hitting what is now commonly referred to as the ``convex barrier''~\citep{Salman2019}. In practice, this implies that either several properties remain undecided in incomplete verification, or take several hours to be verified exactly.

Multiple works have tried to overcome the convex barrier for piecewise linear activations~\citep{Raghunathan2018,Singh2019b}.
Here, we focus on the single-neuron Linear Programming (LP) relaxation by \citet{Anderson2020}.  
Unfortunately, its tightness comes at the price of exponentially many (in the number of variables) constraints. Therefore, existing dual solvers~\citep{Dvijotham2018,BunelDP20} are not easily applicable, limiting the scaling of the new relaxation.

We address this problem by presenting two specialized dual solvers for the relaxation by \citet{Anderson2020}, which realize its full potential by meeting the following desiderata:
\begin{itemize}
	\item Relying on an \emph{active set} of dual variables, we present a unified dual treatment that includes both a linearly sized LP relaxation~\citep{Ehlers2017} and the tighter formulation. 
	As a consequence, we obtain an inexpensive dual initializer, named Big-M, which is competitive with dual approaches on the looser relaxation \citep{Dvijotham2018,BunelDP20}. Moreover, by dynamically extending the active set, we obtain a subgradient-based solver, named Active Set, which rapidly overcomes the convex barrier and yields much tighter bounds if a larger computational budget is available.
	\item The tightness of the bounds attainable by Active Set depends on the memory footprint through the size of the active set. 
	By exploiting the properties of Frank-Wolfe style optimizers~\citep{Frank1956}, we present Saddle Point, a solver that deals with the exponentially many constraints of the relaxation by \citet{Anderson2020} while only incurring a linear memory cost.
	Saddle Point eliminates the dependency on memory at the cost of a potential reduction of the dual feasible space, but is nevertheless very competitive with Active Set in settings requiring tight bounds.
	\item Both solvers are \emph{sparse} and recover the strengths of the original primal problem \citep{Anderson2020} in the dual domain. In line with previous dual solvers, both methods yield valid bounds at anytime, leverage convolutional network structure and enjoy \emph{massive parallelism} within a GPU implementation, resulting in better bounds in an order of magnitude less time than off-the-shelf solvers \citep{gurobi-custom}. Owing to this, we show that both solvers can yield large complete verification gains compared to primal approaches \citep{Anderson2020} and previous dual algorithms.
\end{itemize}

Code implementing our algorithms is available as part of the OVAL neural network verification framework: \url{https://github.com/oval-group/oval-bab}.

A preliminary version of this work appeared in the proceedings of the Ninth International Conference on Learning Representations~\citep{DePalma2021}.
The present article significantly extends it by:
\begin{enumerate}
	\item Presenting Saddle Point ($\S$\ref{sec:sp}), a second solver for the relaxation by \citet{Anderson2020}, which is more memory efficient than both Active Set and the original cutting plane algorithm by \citet{Anderson2020}.
	\item Providing a detailed experimental evaluation of the new solver, both for incomplete and complete verification.
	\item Presenting an adaptive and more intuitive scheme to switch from looser to tighter bounding algorithms within branch and bound ($\S$\ref{sec:stratified}). 
	\item Investigating the effect of different speed-accuracy trade-offs from the presented solvers in the context of complete verification.
\end{enumerate}

\section{Preliminaries: Neural Network Relaxations} \label{sec:preliminaries}
We denote vectors by bold lower case letters (for example, $\xb$) and matrices by upper case letters (for example, $W$). 
We use $\odot$ for the Hadamard product, $\llbracket \cdot \rrbracket$ for integer ranges, $\mathds{1}_{{\mbf{a}}}$ for the indicator vector on condition $\mbf{a}$ and brackets for intervals ($[\lb_k,\ub_k]$) and vector or matrix entries ($\xb[i]$ or $W[i,j]$). In addition, $\text{col}_i(W)$ and $\text{row}_i(W)$ respectively denote the $i$-th column and the $i$-th row of matrix~$W$.
Finally, given $W \in \mathbb{R}^{m \times n}$ and $\xb \in \mathbb{R}^{m}$, we will employ $W \diamond \xb$ and $W \oblong \xb$ as shorthands for respectively $\sum_i \text{col}_i(W) \odot \xb$ and $\sum_i \text{col}_i(W)^T\xb$.

\vspace{3pt}
Let $\mathcal{C}$ be the network input domain. Similar to \citet{Dvijotham2018,BunelDP20}, we assume that the minimization of a linear function over $\mathcal{C}$ can be performed efficiently.
For instance, this is the case for $\ell_{\infty}$ and $\ell_2$ norm perturbations.
Our goal is to compute bounds on the scalar output of a piecewise-linear feedforward neural network.
The tightest possible lower bound can be obtained by solving the following optimization~problem:
\begin{subequations}
	\begin{alignat}{2}
	\smash{\min_{\xb, \xbhat}}\quad \hat{x}_{n} \qquad
	\text{s.t. }\quad& \xb_0 \in \mathcal{C},\label{eq:verif-inibounds}\span\\
	& \xbhat_{k+1} = W_{k+1} \xb_{k} + \mathbf{b}_{k+1} \quad && k \in \left\llbracket0,n-1\right\rrbracket,\label{eq:verif-linconstr}\\
	& \xb_{k} = \sig{\xbhat_k} \quad &&k \in \left\llbracket1,n-1\right\rrbracket,\label{eq:verif-relu}
	\end{alignat}
	\label{eq:noncvx_pb}
\end{subequations} 
where the activation function $\sig{\xbhat_k}$ is piecewise-linear, $\xbhat_{k}, \xb_{k} \in \mathbb{R}^{n_k}$ denote the outputs of the~$k$-th linear layer (fully-connected or convolutional) and activation function respectively, $W_{k}$ and $\mathbf{b}_{k}$ denote its weight matrix and bias, $n_k$ is the number of activations at layer k. 
We will focus on the ReLU case ($\sig{\xb} = \max\left( \xb, 0 \right)$), as common piecewise-linear functions can be expressed as a composition of ReLUs~\citep{Bunel2020}.

Problem \eqref{eq:noncvx_pb} is non-convex due to the activation function's non-linearity, that is, due to constraint \eqref{eq:verif-relu}.
As solving it is NP-hard~\citep{Katz2017}, it is commonly approximated by a convex relaxation (see $\S$\ref{sec:related-work}). 
The quality of the corresponding bounds, which is fundamental in verification, depends on the tightness of the relaxation. 
Unfortunately, tight relaxations usually correspond to slower bounding procedures.
We first review a popular ReLU relaxation in $\S$\ref{sec:relax-planet}. We then consider a tighter one in $\S$\ref{sec:relax-anderson}. 

\subsection{Planet Relaxation} \label{sec:relax-planet}
The so-called Planet relaxation~\citep{Ehlers2017} has enjoyed widespread use due to its amenability to efficient customized solvers~\citep{Dvijotham2018,BunelDP20} and is the ``relaxation of choice" for many works in the area~\citep{Salman2019,Singh2019a,Bunel2020,Balunovic2020,Lu2020Neural}.
Here, we describe it in its non-projected form $\mathcal{M}_k$, corresponding to the LP relaxation of the Big-M Mixed Integer Programming (MIP) formulation~\citep{Tjeng2019}. 
Applying $\mathcal{M}_k$ to problem \eqref{eq:noncvx_pb} results in the following linear program:
\begin{equation}
\begin{aligned}
\smash{\min_{\xb, \xbhat, \zb}}\quad \hat{x}_{n} \quad \text{s.t. }\quad& \xb_0
\in \mathcal{C}\span\\
& \xbhat_{k+1} = W_{k+1} \xb_{k} + \bb_{k+1}  && k \in \left\llbracket0, n-1 \right\rrbracket,\\
&\hspace{-8pt} \left. \begin{array}{l} 
\xb_{k} \geq \xbhat_{k}, \quad \xb_{k} \leq \hat{\ub}_k \odot \zb_{k}, \\
\xb_{k} \leq \xbhat_{k} - \hat{\lb}_k \odot  (1 - \zb_{k}),\\
(\xb_{k}, \xbhat_{k}, \zb_{k}) \in [\lb_k, \ub_k]  \times [\hat{\lb}_k, \hat{\ub}_k] \times [\mbf{0},\ \mbf{1}]  \end{array} \right\}  := \mathcal{M}_k &&  k \in \left\llbracket1, n-1\right\rrbracket,
\end{aligned}
\label{eq:primal-bigm}
\end{equation}
where $\hat{\lb}_k, \hat{\ub}_k$ and $\lb_k, \ub_k$ are \emph{intermediate bounds} respectively on pre-activation variables $\xbhat_{k}$ and post-activation variables $\xb_{k}$.
These constants play an important role in the structure of $\mathcal{M}_k$ and, together with the relaxed binary constraints on $\zb$, define box constraints on the variables.
We detail how to compute intermediate bounds in appendix \ref{sec:intermediate}.
Projecting out auxiliary variables $\zb$ results in the Planet relaxation (cf. appendix \ref{sec:planet-eq} for details), which replaces \eqref{eq:verif-relu} by its convex hull.

Problem \eqref{eq:primal-bigm}, which is linearly-sized, can be easily solved via commercial black-box LP solvers~\citep{Bunel2018}. 
This does not scale-up well with the size of the neural network, motivating the need for specialized solvers.
Customized dual solvers have been designed by relaxing constraints \eqref{eq:verif-linconstr}, \eqref{eq:verif-relu}~\citep{Dvijotham2018} or replacing \eqref{eq:verif-relu} by the Planet relaxation and employing Lagrangian Decomposition~\citep{BunelDP20}.
Both approaches result in bounds very close to optimality for problem \eqref{eq:primal-bigm} in only a fraction of the runtime of off-the-shelf solvers.  

\subsection{A Tighter Relaxation} \label{sec:relax-anderson}

A much tighter approximation of problem \eqref{eq:noncvx_pb} than the Planet relaxation ($\S$\ref{sec:relax-planet}) can be obtained by representing the convex hull of the composition of constraints \eqref{eq:verif-linconstr} and \eqref{eq:verif-relu} rather than the convex hull of constraint \eqref{eq:verif-relu} alone.
A formulation of this type was recently introduced by~\citet{Anderson2020}.
In order to represent the interaction between $\xb_k$ and all possible subsets of the activations of the previous layer, the formulation relies on a number of auxiliary matrices, which we will now introduce. 
Weight matrices are masked entry-wise via $I_k$, binary masks belonging to the following set: $\smash{2^{W_k} = \{0, 1\}^{n_k \times n_{k-1}}}$. We define $\mathcal{E}_k := 2^{W_k} \setminus \{0, 1\}$, to exclude the all-zero and all-one masks, which we treat separately (see $\S$\ref{sec:preacts}).
In addition, the formulation requires bounds on the subsets of $\xb_{k-1}$ selected via the masking. 
In order to represent these bounds, the formulation exploits matrices $\smash{\lbrv_{k-1}, \ubrv_{k-1} \in \mathbb{R}^{n_k \times n_{k-1}}}$, which are also masked via $I_k$ and computed via interval arithmetic as follows:
\begin{equation*}
\begin{aligned}
\lbrv_{k-1}[i,j] &= \lb_{k-1}[j] \mathds{1}_{W_k[i,j] \geq 0}  +  \ub_{k-1}[j] \mathds{1}_{W_k[i,j]  < 0}, \\
\ubrv_{k-1}[i,j] &= \ub_{k-1}[j] \mathds{1}_{W_k[i,j] \geq 0}  +  \lb_{k-1}[j] \mathds{1}_{W_k[i,j]  < 0}.
\end{aligned}
\vspace{5pt}
\end{equation*}
The new representation results in the following primal problem:
\begin{equation} 
\begin{aligned}
\min_{\xb, \xbhat, \zb}&  \ \hat{x}_{n}  \enskip
\text{s.t. }  \\[-5pt]
&\xb_0 \in \mathcal{C}\span\\
& \xbhat_{k+1} = W_{k+1} \xb_{k} + \bb_{k+1}  && k \in \left\llbracket0, n-1 \right\rrbracket,\\
&\hspace{-8pt} \left. \begin{array}{l}  (\xb_{k}, \xbhat_{k}, \zb_{k}) \in \mathcal{M}_k \\
\xb_k \leq \left( \hspace{-4pt} \begin{array}{l} \left(W_k  \odot I_{k}\right)  \xb_{k-1}  + \zb_k \odot \bb_k\ + \\
- \left(W_k  \odot I_{k} \odot \lbrv_{k-1}\right) \diamond (1 - \zb_k)\ + \\
+ \left(W_k \odot  \left(1 - I_{k}\right) \odot \ubrv_{k-1}\right) \diamond \zb_k \end{array} \hspace{-4pt} \right) 
\ \forall I_{k} \in \mathcal{E}_k \end{array} \right\}  := \mathcal{A}_k   && k \in \left\llbracket1, n-1\right\rrbracket.
\end{aligned}
\label{eq:primal-anderson}
\end{equation}

Both $\mathcal{M}_k$ and $\mathcal{A}_k$ yield valid MIP formulations for problem \eqref{eq:noncvx_pb} when imposing integrality constraints on $\zb$. However, the LP relaxation of $\mathcal{A}_k$ will yield tighter bounds. 
In the worst case, this tightness comes at the cost of exponentially many constraints: one for each $I_k \in {\cal E}_k$. 
On the other hand, given a set of primal assignments $(\xb, \zb)$ that are not necessarily feasible for problem \eqref{eq:primal-anderson}, one can efficiently compute the most violated constraint (if any) at that point. 
Denoting by $\mathcal{A}_{\mathcal{E}, k} = \mathcal{A}_{k} \setminus \mathcal{M}_{k}$ the exponential family of constraints, the mask associated to the most violated constraint in $\mathcal{A}_{\mathcal{E}, k}$ can be computed in linear-time~\citep{Anderson2020} as:
\begin{equation}
I_k[i, j] = \mathds{1}^T_{ \left( \left( 1 - \zb_{k}[i] \right)\odot \lbrv_{k-1}[i,j] + \zb_{k}[i] \odot \ubrv_{k-1}[i,j] - \xb_{k-1}[j] \right) W_k[i, j]  \geq 0 }.
\label{eq:oracle}
\end{equation}
The most violated constraint in $\mathcal{A}_k$ is then obtained by comparing the constraint violation from the output of oracle \eqref{eq:oracle} to those from the constraints in $\mathcal{M}_k$.

\subsubsection{Pre-activation Bounds} \label{sec:preacts}
Set $\mathcal{A}_k$ defined in problem \eqref{eq:primal-anderson} slightly differs from the original formulation of~\cite{Anderson2020}, as the latter does not exploit pre-activation bounds $\hat{\lb}_k, \hat{\ub}_k$ within the exponential family. In particular, constraints $\xb_{k} \leq \hat{\ub}_k \odot \zb_{k}$, and $\xb_{k} \leq \xbhat_{k} - \hat{\lb}_k \odot  (1 - \zb_{k})$, which we treat via $\mathcal{M}_k$, are replaced by looser counterparts.
While this was implicitly addressed in practical applications~\citep{Botoeva2020}, 
not doing so has a strong negative effect on bound tightness, possibly to the point of yielding looser bounds than problem \eqref{eq:primal-bigm}. 
In appendix~\ref{sec:intermediate-anderson}, we provide an example in which this is the case and extend the original derivation by~\citet{Anderson2020} to recover $\mathcal{A}_k$ as in problem \eqref{eq:primal-anderson}. 

\subsubsection{Cutting Plane Algorithm}
Owing to the exponential number of constraints, problem \eqref{eq:primal-anderson} cannot be solved as it is.
As outlined by~\citet{Anderson2020}, the availability of a linear-time separation oracle \eqref{eq:oracle} offers a natural primal \emph{cutting plane}  algorithm, which can then be implemented in off-the-shelf solvers:
solve the Big-M LP \eqref{eq:primal-bigm}, then iteratively add the most violated constraints from $\mathcal{A}_k$ at the optimal solution.
When applied to the verification of small neural networks via off-the-shelf MIP solvers, this leads to substantial gains with respect to the looser Big-M relaxation~\citep{Anderson2020}.

\section{A Dual Formulation for the Tighter Relaxation} \label{sec:solver}

Inspired by the success of dual approaches on looser relaxations~\citep{BunelDP20, Dvijotham2019}, we show that the formal verification gains by \citet{Anderson2020} (see $\S$\ref{sec:relax-anderson}) scale to larger networks if we solve the tighter relaxation in the dual space.
Due to the particular structure of the relaxation, a customized solver for problem \eqref{eq:primal-anderson} needs to meet a number of requirements.

\begin{fact}
	In order to replicate the success of previous dual algorithms on looser relaxations, we need a solver for problem \eqref{eq:primal-anderson}  with the following properties: 
	(i) \emph{sparsity}: a memory cost linear in the number of network activations in spite of exponentially many constraints, 
	(ii) \emph{tightness}: the bounds should reflect the quality of those obtained in the primal space, 
	(iii) \emph{anytime}: low cost per iteration and valid bounds at each step.   
	\label{fact}
\end{fact} 

The anytime requirement motivates dual solutions: any dual assignment yields a valid bound due to weak duality.
Unfortunately, as shown in appendix~\ref{sec:no-previous}, neither of the two dual derivations by~\citet{BunelDP20,Dvijotham2018} readily satisfy all desiderata at once. 
Therefore, we need a completely different approach.
Starting from primal \eqref{eq:primal-anderson}, we relax all constraints in $\mathcal{A}_k$ except~box constraints (see $\S$\ref{sec:relax-planet}). 
In order to simplify notation, we employ dummy variables $\balpha_{0}= 0$, $\bbeta_0 = 0$, $\balpha_{n} = I$, $\bbeta_{n} = 0$, obtaining the following dual problem (derivation in appendix \ref{sec:dual-derivation}):
\begin{gather} 
\begin{aligned}
\max_{(\balpha, \bbeta)\geq 0}& d(\balpha, \bbeta) \qquad \text{where:} \qquad d(\balpha, \bbeta) :=\min_{\xb, \zb} \enskip \mathcal{L}(\xb, \zb, \balpha, \bbeta),  \\[-5pt]
&\hspace{-15pt}\mathcal{L}(\xb, \zb, \balpha, \bbeta)  = 
\left[\begin{array}{l}
\sum_{k=1}^{n-1}\bb_k^T \balpha_{k}  - \sum_{k=0}^{n-1}\boldsymbol{f}_{k}(\balpha, \bbeta)^T \xb_k 
- \sum_{k=1}^{n-1}\boldsymbol{g}_{k}(\bbeta)^T \zb_k \\
+ \sum_{k=1}^{n-1}\left(\sum_{I_{k}  \in\mathcal{E}_k} (W_k \odot I_{k} \odot \lbrv_{k-1}) \oblong \bbeta_{k, I_{k}} + \bbeta_{k, 1}^T (\hat{\lb}_k - \bb_{k})\right)\end{array}\right.\\[1pt]
\text{s.t. } \qquad & \xb_0 \in \mathcal{C}, \qquad  (\xb_{k}, \zb_{k}) \in [\lb_k, \ub_k] \times [\mbf{0},\ \mbf{1}] \qquad k \in \llbracket1, n-1\rrbracket,
\end{aligned}  
\label{eq:dual-anderson}
\end{gather} 
where functions $\boldsymbol{f}_{k}, \boldsymbol{g}_{k}$ are defined as follows:
\begin{equation}
\begin{aligned}
&\hspace{-15pt}\boldsymbol{f}_{k}(\balpha, \bbeta) = \begin{array}{l}
\balpha_{k} - W_{k+1}^T\balpha_{k+1} - \sum_{I_{k}}\bbeta_{k, I_{k}} + \smash{\sum_{I_{k+1}} (W_{k+1} \odot I_{k+1})^T \bbeta_{k+1, I_{k+1}}},
\end{array}  \\
&\hspace{-15pt} \boldsymbol{g}_{k}(\bbeta) = \left[\begin{array}{l}  \sum_{I_{k}  \in\mathcal{E}_k} \left(W_k \odot (1 - I_{k}) \odot \ubrv_{k-1}\right) \diamond \bbeta_{k, I_{k}} + \bbeta_{k, 0} \odot \hat{\ub}_k + \bbeta_{k, 1} \odot \hat{\lb}_k\\
+ \sum_{I_{k}  \in\mathcal{E}_k} \left(W_k \odot I_{k} \odot \lbrv_{k-1}\right) \diamond \bbeta_{k, I_{k}} + \sum_{I_{k} \in \mathcal{E}_k} \bbeta_{k, I_{k}} \odot \bb_k.
\end{array}\right.
\end{aligned}
\label{eq:dual-anderson-functions}
\end{equation}
We employ $\sum_{I_{k}}$ as a shorthand for $\sum_{I_{k} \in 2^{W_k}}$. 
Both functions therefore include sums over an exponential number of $\bbeta_{k, I_{k}}$ variables.

This is again a challenging problem: the exponentially many constraints in the primal~\eqref{eq:primal-anderson} are associated to an exponential number of variables, as $\bbeta = \{\bbeta_{k, I_k} \forall I_k \in 2^{W_k}, k \in \llbracket1, n-1\rrbracket\}$. 
Nevertheless, we show that the requirements of Fact \ref{fact} can be met by operating on restricted versions of the dual.
We present two specialized algorithms for problem \eqref{eq:dual-anderson}: one considers only a small active set of dual variables ($\S$\ref{sec:active}), the other restricts the dual domain while exploiting the sparsity of Frank-Wolfe style iterates ($\S$\ref{sec:sp}).
Both algorithms are sparse, anytime and yield bounds reflecting the tightness of the new relaxation.

We conclude this section by pointing out that, as stated in $\S$\ref{sec:preliminaries}, the only assumption for $\mathcal{C}$ in problem~\eqref{eq:dual-anderson} is that it allows for efficient linear minimization. This is the case for both $\ell_{\infty}$ and $\ell_2$ norm perturbations, which are hence supported by our solvers.

\section{Active Set} \label{sec:active}

We present Active Set (Algorithm~\ref{alg:active-set}), a supergradient-based solver that operates on a small active set of dual variables $\bbeta_\mathcal{B}$. 
Starting from the dual of problem \eqref{eq:primal-bigm}, Active Set iteratively adds variables to $\bbeta_\mathcal{B}$ and solves the resulting reduced version of problem \eqref{eq:dual-anderson}.
We first describe our solver on a fixed $\bbeta_\mathcal{B}$ ($\S$\ref{sec:as-solver}) and then outline how to iteratively modify the active set ($\S$\ref{sec:active-set-descr}).

\begin{figure}[t]
\centering
\begin{minipage}{.99\textwidth}
	\begin{algorithm}[H]
		\caption{Active Set}\label{alg:active-set}
		\small
		\begin{algorithmic}[1]
			\algrenewcommand\algorithmicindent{1.0em}
			\Function{activeset\_compute\_bounds}{Problem \eqref{eq:dual-anderson}}
			\State Initialize duals $\balpha^0, \bbeta_{0}^0, \bbeta_{1}^0$ using Algorithm \eqref{alg:bigm}
			\State Set $\bbeta^0_{k, I_{k}} = 0, \forall \ I_{k} \in \mathcal{E}_k$
			\State $\mathcal{B} = \emptyset$
			\For{$\mtt{nb\_additions}$}
			\For{$t \in \llbracket0, T-1\rrbracket$}
			\State $\xb^{*, t}, \zb^{*, t} \in$ $\argmin_{\xb, \zb}\mathcal{L}_{\mathcal{B}}(\xb, \zb, \balpha^t, \bbeta^t_{\mathcal{B}}) $ using \eqref{eq:as-primalk-min},\eqref{eq:as-primal0-min} \Comment{inner minimization}
			\If{$t \leq \mtt{nb\_vars\_to\_add}$}
			\State For each layer $k$, add output of \eqref{eq:oracle} called at $(\xb^*, \zb^*)$ to $\mathcal{B}_k$  \Comment{active set extension}
			\EndIf
			\State $\balpha^{t+1}, \bbeta^{t+1}_{\mathcal{B}}\leftarrow (\balpha^{t}, \bbeta^{t}) + H(\nabla_{\balpha}d(\balpha, \bbeta), \nabla_{\bbeta_{\mathcal{B}}}d(\balpha, \bbeta) )$ \Comment{supergradient step, using \eqref{eq:as-supergradient}}
			\State $\balpha^{t+1}, \bbeta^{t+1}_{\mathcal{B}}\leftarrow \max(\balpha^{t+1}, 0), \max(\bbeta^{t+1}_{\mathcal{B}}, 0) \qquad$ \Comment{dual projection}
			\EndFor
			\EndFor
			\State\Return $\min_{\xb, \zb}\mathcal{L}_{\mathcal{B}}(\xb, \zb, \balpha^T, \bbeta^T_{\mathcal{B}}) $
			\EndFunction
		\end{algorithmic}
	\end{algorithm}
\end{minipage}
\end{figure}
\vspace{10pt}

\subsection{Solver} \label{sec:as-solver}

We want to solve a version of problem \eqref{eq:dual-anderson} for which $\mathcal{E}_k$, the exponentially-sized set of $I_k$ masks for layer $k$,
is restricted to some constant-sized set\footnote{As dual variables $\bbeta_{k, I_k}$ are indexed by $I_k$, $\mathcal{B} = \cup_k \mathcal{B}_k$ implicitly defines an active set of variables $\bbeta_\mathcal{B}$.} $\mathcal{B}_k \subseteq \mathcal{E}_k$, with $\mathcal{B} = \cup_{k \in \llbracket1, n-1\rrbracket} \mathcal{B}_k$. 
By keeping $\mathcal{B} = \emptyset$, we recover a novel dual solver for the Big-M relaxation \eqref{eq:primal-bigm} (explicitly described in appendix \ref{sec:dual-init}), which is employed as initialization.
Setting $\bbeta_{k, I_{k}} = 0, \forall \ I_{k} \in \mathcal{E}_k \setminus \mathcal{B}_k$ in \eqref{eq:dual-anderson-functions}, \eqref{eq:dual-anderson} and removing these from the formulation, we obtain: 

\begin{equation*}
\begin{aligned}
&\hspace{-15pt}\boldsymbol{f}_{\mathcal{B}, k}(\balpha, \bbeta_{\mathcal{B}}) = \left[\begin{array}{l}
\balpha_{k} - W_{k+1}^T\balpha_{k+1} - \sum_{I_{k} \in \mathcal{B}_k \cup \{0,1\}}\bbeta_{k, I_{k}}, \\
+ \smash{\sum_{I_{k+1}\in \mathcal{B}_{k+1} \cup \{0,1\}} (W_{k+1} \odot I_{k+1})^T \bbeta_{k+1, I_{k+1}}} \end{array}\right.  \\
&\hspace{-15pt} \boldsymbol{g}_{\mathcal{B}, k}(\bbeta_{\mathcal{B}}) = \left[\begin{array}{l}\sum_{I_{k}  \in\mathcal{B}_k} \left(W_k \odot (1 - I_{k}) \odot \ubrv_{k-1}\right) \diamond \bbeta_{k, I_{k}} + \bbeta_{k, 0} \odot \hat{\ub}_k + \bbeta_{k, 1} \odot \hat{\lb}_k\\
+ \sum_{I_{k}  \in\mathcal{B}_k} \left(W_k \odot I_{k} \odot \lbrv_{k-1}\right) \diamond \bbeta_{k, I_{k}}  +  \sum_{I_{k} \in \mathcal{B}_k} \bbeta_{k, I_{k}} \odot \bb_k,  \end{array}\right.
\end{aligned}
\end{equation*}
along with the reduced dual problem: 
\begin{gather} 
\begin{aligned}
\max_{(\balpha, \bbeta_{\mathcal{B}})\geq 0} &d_{\mathcal{B}}(\balpha, \bbeta_{\mathcal{B}}) \qquad  \text{where:} \qquad d_{\mathcal{B}}(\balpha, \bbeta_{\mathcal{B}}) :=\min_{\xb, \zb} \enskip \mathcal{L}_{\mathcal{B}}(\xb, \zb, \balpha, \bbeta_{\mathcal{B}}),  \\[-5pt]
&\hspace{-20pt}\mathcal{L}_{\mathcal{B}}(\xb, \zb, \balpha, \bbeta_{\mathcal{B}})  = 
\left[\begin{array}{l}
\sum_{k=1}^{n-1}\bb_k^T \balpha_{k} - \sum_{k=0}^{n-1}\boldsymbol{f}_{\mathcal{B}, k}(\balpha, \bbeta_{\mathcal{B}})^T \xb_k - \sum_{k=1}^{n-1}\boldsymbol{g}_{\mathcal{B}, k}(\bbeta_{\mathcal{B}}) ^T \zb_k \\
+ \sum_{k=1}^{n-1} \left(\sum_{I_{k}  \in\mathcal{B}_k} (W_k \odot I_{k} \odot \lbrv_{k-1}) \oblong \bbeta_{k, I_{k}} + \bbeta_{k, 1}^T (\hat{\lb}_k - \bb_{k})\right) \end{array}\right.\\[3pt]
\text{s.t. } \qquad & \xb_0 \in \mathcal{C}, \qquad  (\xb_{k}, \zb_{k}) \in [\lb_k, \ub_k] \times [\mbf{0},\ \mbf{1}] \qquad k \in \llbracket1, n-1\rrbracket.
\end{aligned}
\label{eq:dual-active-set}
\end{gather} 
We can maximize $d_{\mathcal{B}}(\balpha, \bbeta_{\mathcal{B}})$, which is concave and non-smooth, via projected supergradient ascent or variants thereof, such as Adam~\citep{Kingma2015}.
In order to obtain a valid supergradient, we need to perform the inner minimization over the primals. Thanks to the structure of problem \eqref{eq:dual-active-set}, the optimization decomposes over the layers. 
For $k \in \llbracket1, n-1\rrbracket$, we can perform the minimization in closed-form by driving the primals to their upper or lower bounds depending on the sign of their coefficients:
\begin{equation}
\begin{gathered}
\label{eq:as-primalk-min}
\xb_{k}^* = \mathds{1}_{\boldsymbol{f}_{\mathcal{B}, k}(\balpha, \bbeta_{\mathcal{B}})  \geq 0} \odot \hat{\ub}_k + \mathds{1}_{\boldsymbol{f}_{\mathcal{B}, k}(\balpha, \bbeta_{\mathcal{B}})  < 0} \odot \hat{\lb}_k, \qquad
\zb_{k}^* = \mathds{1}_{\boldsymbol{g}_{\mathcal{B}, k}(\bbeta_{\mathcal{B}})   \geq 0} \odot \mbf{1}.
\end{gathered}
\end{equation}
The subproblem corresponding to $\xb_0$ is different, as it involves a linear minimization over $\xb_0 \in \mathcal{C}$:
\begin{equation}
\xb^*_0 \in \begin{array}{l} \argmin_{\xb_0}\end{array} \quad   \boldsymbol{f}_{\mathcal{B}, 0}(\balpha, \bbeta_{\mathcal{B}})^T\xb_{0} \qquad \quad \text{s.t. } \quad  \xb_0 \in \mathcal{C}.
\label{eq:as-primal0-min} 
\end{equation} 
We assumed in $\S$\ref{sec:preliminaries} that \eqref{eq:as-primal0-min} can be performed efficiently. We refer the reader to~\citet{BunelDP20} for descriptions of the minimization when $\mathcal{C}$ is a $\ell_\infty$ or $\ell_2$ ball, as common for adversarial examples.

Given $(\xb^*, \zb^*)$ as above, the supergradient of $d_{\mathcal{B}}(\balpha, \bbeta_{\mathcal{B}}) $ is a subset of the one for $d(\balpha, \bbeta)$, given~by: 
\begin{equation}
\begin{gathered}
\label{eq:as-supergradient}
\nabla_{\balpha_k}d(\balpha, \bbeta) = W_k \xb^*_{k-1} + \bb_k - \xb^*_k, \qquad
\nabla_{\bbeta_{k, 0}}d(\balpha, \bbeta)  = \xb^*_k - \zb^*_k \odot \hat{\ub}_{k},\\
\nabla_{\bbeta_{k, 1}}d(\balpha, \bbeta) = \xb^*_k - \left(W_k\xb^*_{k-1} + \bb_{k}\right) + \left(1 -\zb^*_k \right)\odot \hat{\lb}_{k},\\ 
\nabla_{\bbeta_{k, I_k}}d(\balpha, \bbeta) = \left( \begin{array}{l} \xb^*_{k} - \left(W_k  \odot I_{k}\right)  \xb^*_{k-1} 
+ \left(W_k  \odot I_{k} \odot \lbrv_{k-1}\right) \diamond (1 - \zb^*_k) \ + \\
- \zb^*_k \odot \bb_k +  \left(W_k \odot  \left(1 - I_{k}\right) \odot \ubrv_{k-1}\right) \diamond \zb^*_k \end{array} \right)   I_k \in \mathcal{B}_k,\\
\end{gathered}
\end{equation}
for each $k \in \llbracket1, n-1\rrbracket$ (dual ``variables" $\balpha_{0}, \balpha_{n}, \bbeta_{0}, \bbeta_{n}$ are constants employed to simplify the notation: see appendix \ref{sec:dual-derivation}).
At each iteration, after taking a step in the supergradient direction, the dual variables are projected to the non-negative orthant by clipping negative~values.

\subsection{Extending the Active Set} \label{sec:active-set-descr}

We initialize the dual \eqref{eq:dual-anderson} with a tight bound on the Big-M relaxation by solving for $d_{\emptyset}(\balpha, \bbeta_{\emptyset})$ in~\eqref{eq:dual-active-set} (appendix \ref{sec:dual-init}).
To satisfy the tightness requirement in Fact \ref{fact}, we then need to include constraints (via their Lagrangian multipliers) from the exponential family of $\mathcal{A}_k$ into $\mathcal{B}_k$. 
Our goal is to tighten them as much as possible while keeping the active set small to save memory and compute. 
The active set strategy is defined by a \emph{selection criterion} for the $I^*_{k}$ to be added\footnote{adding a single $I^*_{k}$ mask to $\mathcal{B}_k$ extends $\bbeta_{\mathcal{B}}$ by $n_k$ variables: one for each neuron at layer $k$.} to $\mathcal{B}_k$, and the \emph{frequency} of addition.
In practice, \textit{we add the variables maximising the entries of supergradient $\nabla_{\bbeta_{k, I_k}}d(\balpha, \bbeta)$ after a fixed number of dual iterations.}
We now provide motivation for both choices.

\subsubsection{Selection Criterion}
The selection criterion needs to be computationally efficient. 
Thus, we proceed greedily and focus only on the immediate effect at the current iteration.
Let us map a restricted set of dual variables $\bbeta_{\mathcal{B}}$ to a set of dual variables  $\bbeta$ for the full dual \eqref{eq:dual-anderson}. We do so by setting variables not in the active set to $0$: $\bbeta_{\bar{\mathcal{B}}} = 0$, and $\bbeta = \bbeta_{\mathcal{B}} \cup \bbeta_{\bar{\mathcal{B}}}$.
Then, for each layer $k$, we add the set of variables $\bbeta_{k,I^*_k}$ maximising the corresponding entries of the supergradient of the full dual problem~\eqref{eq:dual-anderson}, excluding those pertaining to $\mathcal{M}_k$:
\begin{equation}
	\bbeta_{k,I^*_k} \in \argmax_{\bbeta_{k,I_k} \in\ \bbeta_{k} \setminus \bbeta_{\emptyset, k} } \{\nabla_{\bbeta_{k, I_k}}d(\balpha, \bbeta)^T \mathbf{1}\}.
	\label{eq:selection}
\end{equation}
Therefore, we use the subderivatives as a proxy for short-term improvement on the full dual objective $d(\balpha, \bbeta)$. 
Under a primal interpretation, our selection criterion involves a call to the separation oracle \eqref{eq:oracle} by \citet{Anderson2020}.

\begin{proposition}
	$\bbeta_{k,I^*_k}$ as defined in equation \eqref{eq:selection} represents the Lagrangian multipliers associated to the most violated constraints from $\mathcal{A}_{\mathcal{E}, k}$ at $\left(\xb^*, \zb^*\right) \in \argmin_{\xb, \zb} \mathcal{L}_{\mathcal{B}}(\xb, \zb, \balpha, \bbeta_{\mathcal{B}})$, the primal minimiser of the current restricted Lagrangian.
	\label{prop:as}
\end{proposition} 
\begin{proof}
	The result can be obtained by noticing that $\nabla_{\bbeta_{k, I_k}}d(\balpha, \bbeta)$ (by definition of Lagrangian multipliers) quantifies the corresponding constraint's violation at $\left(\xb^*, \zb^*\right)$. Therefore, maximizing $\nabla_{\bbeta_{k, I_k}}d(\balpha, \bbeta)$ will amount to maximizing constraint violation.
	We demonstrate analytically that the process will, in fact, correspond to a call to oracle \eqref{eq:oracle}.
	Recall the definition of  $\nabla_{\bbeta_{k, I_k}}d(\balpha, \bbeta)$ in equation \eqref{eq:as-primalk-min}, which applies beyond the current active set:
	$$\nabla_{\bbeta_{k, I_k}}d(\balpha, \bbeta) =\left( \begin{array}{l} \xb^*_{k} - \left(W_k  \odot I_{k}\right)  \xb^*_{k-1} 
	+ \left(W_k  \odot I_{k} \odot \lbrv_{k-1}\right) \diamond (1 - \zb^*_k) \ + \\
	- \zb^*_k \odot \bb_k +  \left(W_k \odot  \left(1 - I_{k}\right) \odot \ubrv_{k-1}\right)\diamond\zb^*_k \end{array} \right)  \quad I_k \in \mathcal{E}_k.$$
	
	We want to compute $I^*_k \in \argmax_{I_k \in \mathcal{E}_k} \{\nabla_{\bbeta_{k, I_k}}d(\balpha, \bbeta)^T \mathbf{1}\}$, that is:
	$$I^*_k \in \argmax_{I_k \in \mathcal{E}_k} \left( \begin{array}{l} \xb^*_{k} - \left(W_k  \odot I_{k}\right)  \xb^*_{k-1} 
	+ \left(W_k  \odot I_{k} \odot \lbrv_{k-1}\right) \diamond(1 - \zb^*_k)  \ +\\
	- \zb^*_k \odot \bb_k +  \left(W_k \odot  \left(1 - I_{k}\right) \odot \ubrv_{k-1}\right)\diamond\zb^*_k \end{array} \right)^T \mathbf{1}.$$
	
	By removing the terms that do not depend on $I_k$, we obtain:  
	$$\max_{I_k \in \mathcal{E}_k} \left( \begin{array}{l} - \left(W_k  \odot I_{k}\right)  \xb^*_{k-1} 
	+ \left(W_k  \odot I_{k} \odot \lbrv_{k-1}\right)\diamond (1 - \zb^*_k) \ +\\
	+ \left( W_k \odot I_{k} \odot \ubrv_{k-1}\right)\diamond \zb^*_k\end{array} \right)^T \mathbf{1}.$$
	Let us denote the $i$-th row of $W_k$ and $I_k$ by $\boldsymbol{w}_{i, k}$ and $\boldsymbol{i}_{i, k}$, respectively, and define  $\mathcal{E}_k[i]  = 2^{\boldsymbol{w}_{i, k}} \setminus \{0, 1\}$.
	The optimization decomposes over each such row: we thus focus on the optimization problem for the supergradient's $i$-th entry. Collecting the mask, we get:
	$$\max_{\boldsymbol{i}_{i, k} \in \mathcal{E}_k[i]} \begin{array}{l} \sum_{j} \bigg( \Big(\left( 1 - \zb^*_{k}[i] \right)\odot \lbrv_{k-1}[i,j] + \zb^*_{k}[i] \odot \ubrv_{k-1}[i,j] -\xb^*_{k-1}[i] \Big) W_{k}[i,j] \bigg)  I_{k}[i,j] \end{array} .$$
	
	As the solution to the problem above is obtained by setting $I_{k}^*[i,j]=1$ if its coefficient is positive and $I_{k}^*[i,j]=0$ otherwise, we can see that the optimal $I_k$ corresponds to calling oracle \eqref{eq:oracle} by \citet{Anderson2020} on $\left(\xb^*, \zb^*\right)$. Hence, in addition to being the mask associated to $\bbeta_{k,I^*_k}$, the variable set maximising the supergradient, $I_{k}^*$ corresponds to the most violated constraint from $\mathcal{A}_{\mathcal{E}, k}$ at $\left(\xb^*, \zb^*\right)$.
\end{proof}

\subsubsection{Frequency of Addition}
Finally, we need to decide the frequency at which to add variables to the active set. 
\begin{fact}
	Assume we obtained a dual solution $(\balpha^\dagger, \bbeta^\dagger_{\mathcal{B}}) \in \argmax d_{\mathcal{B}}(\balpha, \bbeta_{\mathcal{B}})$ using Active Set on the current $\mathcal{B}$.
	Then $\left(\xb^*, \zb^*\right) \in \argmin_{\xb, \zb} \mathcal{L}_{\mathcal{B}}(\xb, \zb, \balpha^\dagger, \bbeta^\dagger_{\mathcal{B}})$ is not in general an optimal primal solution for the primal of the current variable-restricted dual problem \citep{Sherali1996}.
	\label{fact:sherali}
\end{fact} 
The primal of $d_{\mathcal{B}}(\balpha, \bbeta_{\mathcal{B}})$ (restricted primal) is the problem obtained by setting $\mathcal{E}_k \leftarrow \mathcal{B}_k$ in problem~\eqref{eq:primal-anderson}.
While the primal cutting plane algorithm by \citet{Anderson2020} calls the separation oracle \eqref{eq:oracle} at the optimal solution of the current restricted primal, Fact \ref{fact:sherali} shows that our selection criterion leads to a different behaviour even at dual optimality for $d_{\mathcal{B}}(\balpha, \bbeta_{\mathcal{B}})$.
Therefore, as we have no theoretical incentive to reach (approximate) subproblem convergence, we add variables after a fixed tunable number of supergradient iterations.
Furthermore, we can add more than one variable ``at once" by running the oracle \eqref{eq:oracle} repeatedly for a number of iterations.

We conclude this section by pointing out that, provided the algorithm is run to dual convergence on each variable-restricted dual problem~\eqref{eq:dual-active-set}, primal optima can be recovered by suitable modifications of the optimization routine~\citep{Sherali1996}. 
Then, if the dual variables corresponding to the most violated constraint at the primal optima are added to~$\mathcal{B}_k$, Active Set mirrors the primal cutting plane algorithm, guaranteeing convergence to the solution of problem~\eqref{eq:dual-anderson}.
In practice, as the main advantage of dual approaches \citep{Dvijotham2018,BunelDP20} is their ability to quickly achieve tight bounds (rather than formal optimality), we rely on the heuristic strategy in Algorithm~\ref{alg:active-set}.

\section{Saddle Point} \label{sec:sp}

For the Active Set solver ($\S$\ref{sec:active}), we only consider settings in which $\bbeta_{\mathcal{B}}$ is composed of a (small) constant number of variables. 
In fact, both its memory cost and time complexity per iteration are proportional to the cardinality of the active set. This mirrors the properties of the primal cutting algorithm by~\citet{Anderson2020}, for which memory and runtime will increase with the number of added constraints.
As a consequence, the tightness of the attainable bounds will depend both on the computational budget and on the available memory.
We remove the dependency on memory by presenting Saddle Point (Algorithm \ref{alg:sp-fw}), a Frank-Wolfe type solver. By restricting the dual feasible space, Saddle Point is able to deal with all the exponentially many variables from problem \eqref{eq:dual-anderson}, while incurring only a linear memory cost.
We first describe the rationale behind the reduced dual domain ($\S$\ref{sec:sp-fw-theory}), then proceed to describe solver details ($\S$\ref{sec:sp-fw-solver}).

\subsection{Sparsity via Sufficient Statistics} \label{sec:sp-fw-theory}

In order to achieve sparsity (Fact \ref{fact}) without resorting to active sets, 
it is crucial to observe that all the appearances of $\bbeta$ variables in $\mathcal{L}(\xb, \zb, \balpha, \bbeta)$, the Lagrangian of the full dual \eqref{eq:dual-anderson}, can be traced back to the following linearly-sized \emph{sufficient statistics}:

\begin{equation}
\bzeta_{k}(\bbeta_k) = 
\left[\begin{array}{l}
\sum_{I_{k}} \bbeta_{k, I_{k}} \\ \sum_{I_{K+1}} (W_{k+1} \odot I_{k+1})^T \bbeta_{k+1, I_{k+1}} \\
 \sum_{I_{k}  \in\mathcal{E}_k} (W_k \odot I_{k} \odot \lbrv_{k-1}) \diamond \bbeta_{k, I_{k}} + \bbeta_{k, 1} \odot (\hat{\lb}_k - \bb_{k}) \\
 \sum_{I_{k}  \in\mathcal{E}_k} (W_k \odot I_{k} \odot \ubrv_{k-1}) \diamond \bbeta_{k, I_{k}} + \bbeta_{k, 0} \odot (\hat{\ub}_k - \bb_{k})
\end{array}\right].
\label{eq:suff-stats}
\end{equation}
Therefore, by designing a solver that only requires access to $\bzeta_{k}(\bbeta_k)$ rather than to single $\bbeta_k$ entries, we can incur only a linear memory cost. 
\vspace{7pt}

In order for the resulting algorithm to be computationally efficient, we need to meet the anytime requirement in Fact~\ref{fact} with a low cost per iteration.
Let us refer to the evaluation of a neural network at a given input point $\xb_{0}$ as a forward pass, and the backpropagation of a gradient through the network as a backward pass.
\begin{fact}
	The full dual objective $d(\balpha, \bbeta)$ can be computed at the cost of a backward pass over the neural network if sufficient statistics $\bzeta(\bbeta) = \cup_{k \in \llbracket1, n-1\rrbracket}\bzeta_k(\bbeta_k)$ have been pre-computed.
	\label{fact:suffstats}
\end{fact}
\begin{proof}
	If $\bzeta(\bbeta)$ is up to date, the Lagrangian $\mathcal{L}(\xb, \zb, \balpha, \bbeta)$ can be evaluated using a single backward pass: this can be seen by replacing the relevant entries of \eqref{eq:suff-stats} in equations \eqref{eq:dual-anderson-functions} and \eqref{eq:dual-anderson}. 
	Similarly to gradient backpropagation, the bottleneck of the Lagrangian computation is the layer-wise use of transposed linear operators over the $\balpha$ dual variables.
	The minimization of the Lagrangian over primals can then be computed in linear time (less than the cost of a backward pass) by using equations \eqref{eq:as-primalk-min}, \eqref{eq:as-primal0-min} with $\mathcal{B}_k = \mathcal{E}_k$ for each layer~$k$.
\end{proof}
From Fact~\ref{fact:suffstats}, we see that the full dual can be efficiently evaluated via $\bzeta(\bbeta)$. On the other hand, in the general case, $\bzeta(\bbeta)$ updates have an exponential time complexity.
Therefore, we need a method that updates the sufficient statistics while computing a minimal number of terms of the exponentially-sized sums in \eqref{eq:suff-stats}. In other words, we need \emph{sparse updates} in the $\bbeta$ variables.
With this goal in mind, we consider methods belonging to the Frank-Wolfe family~\citep{Frank1956}, whose iterates are known to be sparse~\citep{Jaggi2013}. 
In particular, we now illustrate that sparse updates can be obtained by applying the Saddle-Point Frank-Wolfe (SP-FW) algorithm by \citet{Gidel17} to a suitably modified version of problem \eqref{eq:dual-anderson}.  
Details of SP-FW and the solver resulting from its application are then presented in $\S$\ref{sec:sp-fw-solver}.

\begin{fact} \label{fact:saddle}
	Dual problem \eqref{eq:dual-anderson} can be seen as a bilinear saddle point problem. 
	By limiting the dual feasible region to a compact set, a dual optimal solution for this domain-restricted problem can be obtained via SP-FW \citep{Gidel17}.
	Moreover, a valid lower bound to \eqref{eq:primal-anderson} can be obtained at anytime by evaluating $d(\balpha, \bbeta)$ at the current dual point from SP-FW.
\end{fact}

We make the feasible region of problem \eqref{eq:dual-anderson} compact by capping the cumulative price for constraint violations at some constants $\mub$.  In particular, we bound the $\ell_1$ norm for the sets of $\bbeta$ variables associated to each neuron. As the $\ell_1$ norm is well-known to be sparsity inducing \citep{Cands2008}, our choice reflects the fact that,  in general, only a fraction of the $\mathcal{A}_k$ constraints will be active at the optimal solution.
Denoting by $\triangle(\mub) = \cup_{k \in \llbracket1, n-1\rrbracket} \triangle_k(\mub_{k})$ the resulting dual domain, we obtain domain-restricted dual  $\max_{(\balpha, \bbeta) \in \triangle(\mub)} d(\balpha, \bbeta)$, which can be written as the following saddle point problem:
\begin{equation}
\begin{aligned}
\max_{\balpha, \bbeta}& \min_{\xb, \zb}  \enskip \mathcal{L}(\xb, \zb, \balpha, \bbeta) \\
\text{s.t. } \enskip &\xb_0 \in \mathcal{C}, \\
&(\xb_k, \zb_k) \in [\lb_k, \ub_k] \times [\mbf{0},\ \mbf{1}]  && k \in \left\llbracket1, n-1\right\rrbracket,  \\
& \left. \hspace{-7pt} \begin{array}{l}
\balpha_{k} \in [\mbf{0}, \mub_{\alpha, k}] \\
\hspace{-2pt} \bbeta_{k} \geq \mbf{0}, \ \left\lVert \bbeta_{k}\right\rVert_1 \leq  \mub_{\beta, k}  \end{array}\right\}  := \triangle_k(\mub_{k}) \quad && k \in \left\llbracket1, n-1\right\rrbracket,
\end{aligned}
\label{eq:sp-fw-dual}
\end{equation}
where $\mathcal{L}(\xb, \zb, \balpha, \bbeta)$ was defined in equation \eqref{eq:dual-anderson} and $\left\lVert \cdot\right\rVert_1$ denotes the $\ell_1$ norm. 
Frank-Wolfe type algorithms move towards the vertices of the feasible region. 
Therefore, the shape of $\triangle_k(\mub_{k})$ is key to the efficiency of $\bzeta_k$ updates. In our case, $\triangle_k(\mub_{k})$ is the Cartesian product of a box constraint on $\balpha_k$ and $n_k$ exponentially-sized simplices: one for each set $\bbeta_{k}[i] = \{\bbeta_{k, \text{row}_i(I_k)}[i]\ \forall\ \text{row}_i(I_{k})\in 2^{\text{row}_i(W_k)}\}$.
As a consequence, each vertex of $\triangle_k(\mub_{k})$ is sparse in the sense that at most $n_k + 1$ variables out of exponentially many will be non-zero.
In order for the resulting solver to be useful in practice, we need to efficiently select a vertex towards which to move:
we show in section \ref{sec:sp-fw-solver} that our choice for $\triangle_k(\mub_{k})$ allows us to recover the linear-time primal oracle \eqref{eq:oracle} by \citet{Anderson2020}. 

\vspace{7pt}
Before presenting the technical details of our Saddle Point solver, it remains to comment on the consequences of the dual space restriction on the obtained bounds.
Let us define $d_{\mub}^*=\max_{(\balpha, \bbeta) \in \triangle(\mub)} d(\balpha, \bbeta)$, the optimal value of the restricted dual problem associated to saddle point problem \eqref{eq:sp-fw-dual}. 
Value $d_{\mub}^*$ is attained at the dual variables from a saddle point of problem~\eqref{eq:sp-fw-dual}.
As we restricted the dual feasible region, $d_{\mub}^*$ will in general be smaller than the optimal value of problem~\eqref{eq:dual-anderson}. 
However, owing to the monotonicity of $d_{\mub}^*$ over $\mub$ and the concavity of $d(\balpha, \bbeta)$, we can make sure $\triangle(\mub)$ contains the optimal dual variables by running a binary search on $\mub$. 
In practice, we heuristically determine the values of $\mub$ from our dual initialization procedure (see appendix~\ref{sec:sp-fw-caps}).

\begin{figure}[t]
	\centering
	\begin{minipage}{.99\textwidth}
		\begin{algorithm}[H]
			\caption{Saddle Point}\label{alg:sp-fw}
			\small
			\begin{algorithmic}[1]
				\algrenewcommand\algorithmicindent{1.0em}
				\Function{saddlepoint\_compute\_bounds}{Problem \eqref{eq:dual-anderson}}
				\State Initialize duals $\balpha^0, \bbeta_{\mathcal{B}}^0$ using algorithm \eqref{alg:bigm} or algorithm \eqref{alg:active-set}
				\State Set $\bbeta_{\bar{\mathcal{B}}}^0 = 0$, $\bbeta^0 = \bbeta_{\mathcal{B}}^0 \cup \bbeta_{\bar{\mathcal{B}}}^0$, and replace $\bbeta^0$ by its sufficient statistics $\bzeta(\bbeta_k^0)$ using \eqref{eq:suff-stats}
				\State Initialize primals $\xb^0, \zb^0$ according to $\S$\ref{sec:sp-fw-primalinit}
				\State Set price caps $\mub$ heuristically as outlined in $\S$\ref{sec:sp-fw-caps}
				\For{$t \in \llbracket0, T-1\rrbracket$}
				\State $\bar{\xb}^t, \bar{\zb}^t \leftarrow$ using \eqref{eq:as-primalk-min},\eqref{eq:as-primal0-min} with $\mathcal{B}_k = \mathcal{E}_k$ \Comment{compute primal conditional gradient}
				\State $\bar{\balpha}^t, \bzeta(\bar{\bbeta}^t) \leftarrow$ \eqref{eq:alphak-condg}, \eqref{eq:betak-condg} + \eqref{eq:suff-stats} \Comment{compute dual conditional gradient}
				\State $\xb^{t+1}, \zb^{t+1}, \balpha^{t+1}, \bzeta(\bbeta^{t+1}) = (1 - \gamma_{t}) [\xb^{t}, \zb^{t}, \balpha^{t}, \bzeta(\bbeta^{t})] + \gamma_t [\bar{\xb}^{t}, \bar{\zb}^{t}, \bar{\balpha}^{t}, \bzeta(\bar{\bbeta}^{t})]$
				\EndFor
				\State\Return $\min_{\xb, \zb}\mathcal{L}(\xb, \zb, \balpha^T, \bzeta(\bbeta^T)) $
				\EndFunction
			\end{algorithmic}
		\end{algorithm}
	\end{minipage}
\end{figure}

\subsection{Solver} \label{sec:sp-fw-solver}
Algorithms in the Frank-Wolfe family proceed by taking convex combinations between the current iterate and a vertex of the feasible region. This ensures feasibility of the iterates without requiring projections.
For SP-FW~\citep{Gidel17}, the convex combination is performed at once, with the same coefficient, for both primal and dual variables.

In the general case, denoting primal variables as $\xb$, and dual variables as $\yb$, each iteration of the SP-FW algorithm proceeds as follows:
first, we compute the vertex $[\bar{\xb}, \bar{\yb}]$ towards which we take a step (\emph{conditional gradient}). This is done by maximizing the inner product between the gradient and the variables over the feasible region for the dual variables, and by minimizing the inner product between the gradient and the variables over the feasible region for the primal variables. This operation is commonly referred to as the linear maximization oracle for dual variables, and linear minimization oracle for primal variables.
Second, a step size $\gamma_t \in [0, 1]$ is determined according to the problem specification.
Finally, the current iterate is updated as $[\xb, \yb] \leftarrow (1 - \gamma_{t}) [\xb, \yb]  + \gamma_{t} [\bar{\xb}, \bar{\yb}]$.
We will now provide details for the instantiation of SP-FW in the context of problem \eqref{eq:sp-fw-dual}, along with information concerning the solver initialization.

While Saddle Point relies on a primal-dual method operating on problem \eqref{eq:sp-fw-dual}, our main goal is to compute anytime bounds to problem~\eqref{eq:primal-anderson}. 
As explained in $\S$\ref{sec:solver}, this is typically achieved in the dual domain.
Therefore, as per Fact~\ref{fact:saddle}, we discard the primal variables from SP-FW and use the current dual iterate to evaluate $d(\balpha, \bbeta)$ from problem~\eqref{eq:dual-anderson}. 

\subsubsection{Conditional Gradient Computations} 

Due to the bilinearity of $\mathcal{L}(\xb, \zb, \balpha, \bbeta)$, the computation of the conditional gradient for the primal variables coincides with the inner minimization in equations \eqref{eq:as-primalk-min}-\eqref{eq:as-primal0-min} with $\mathcal{B}_k = \mathcal{E}_k\ \forall\ k \in \llbracket1, n-1\rrbracket$.

Similarly to the primal variables, the linear maximization oracle for the dual variables decomposes over the layers. 
The gradient of the Lagrangian over the duals, $\nabla_{\balpha, \bbeta}\ \mathcal{L}$, is given by the supergradient in equation \eqref{eq:as-supergradient} if $\mathcal{B}_k = \mathcal{E}_k$ and the primal minimiser $(\xb^*, \zb^*)$ is replaced by the primals at the current iterate.
As dual variables $\balpha$ are box constrained, the linear maximization oracle will drive them to their lower or upper bounds depending on the sign of their gradient. Denoting conditional gradients by bars, for each $k \in \llbracket1, n-1\rrbracket$:
\begin{equation}
	\bar{\balpha}_{k} = \mub_{\alpha, k} \odot \mathds{1}_{\left( W_k \xb_{k-1} + \bb_k - \xb_k \right) \geq 0}.
	\label{eq:alphak-condg}
\end{equation}
The linear maximization for the exponentially many $\bbeta_{k}$ variables is key to the solver's efficiency and is carried out on a Cartesian product of simplex-shaped sub-domains (see definition of $\triangle_k(\mub_k)$ in \eqref{eq:sp-fw-dual}). 
Therefore, conditional gradient $\bar{\bbeta}_{k}$ can be non-zero only for the entries associated to the largest gradients of each simplex sub-domain. For each $k \in \llbracket1, n-1\rrbracket$, we have:
\begin{equation}
\begin{aligned}
	& \bar{\bbeta}_{k} = \left\{ \def\arraystretch{1.5} \begin{array}{l}
	\bar{\bbeta}_{k, I^\dagger_k} = \mub_{\beta, k} \odot \mathds{1}_{\nabla_{\bbeta_{k, I^\dagger_k}} \mathcal{L}(\xb, \zb, \balpha, \bbeta) \geq 0} \\ \bar{\bbeta}_{k, I_k} = \mbf{0} \quad \forall\ I_k \in 2^{W_k} \setminus I^\dagger_k \end{array}\right\},\\
	\text{where:} &\quad \bbeta_{k,I^\dagger_k} \in \argmax_{\bbeta_{k,I_k} \in \bbeta_k} \left\{\nabla_{\bbeta_{k, I_k}}\ \mathcal{L}(\xb, \zb, \balpha, \bbeta)^T \mathbf{1}\right\}.
\end{aligned}
\label{eq:betak-condg}
\end{equation}
We can then efficiently represent $\bar{\bbeta}_{k}$ through sufficient statistics as $\bar{\bzeta}_{k} = \bzeta_{k}(\bar{\bbeta}_{k})$, which will require the computation of a single term of the sums in \eqref{eq:suff-stats}.
\begin{proposition}
	$\bbeta_{k,I^\dagger_k}$ as defined in \eqref{eq:betak-condg} represents the Lagrangian multipliers associated to the most violated constraints from $\mathcal{A}_k$ at $\left(\xb, \zb\right)$, the current SP-FW primal iterate. Moreover, the conditional gradient $\bar{\bbeta}_{k}$ can be computed at the cost of a single call to the linear-time oracle \eqref{eq:oracle} by \citet{Anderson2020}.
\end{proposition} 
\begin{proof}
	Let us define $\bbeta_{k,I^*_k}$ as:
	$$\bbeta_{k,I^*_k} \in \argmax_{\bbeta_{k,I_k} \in\ \bbeta_{k} \setminus \bbeta_{\emptyset, k} } \{\nabla_{\bbeta_{k, I_k}}\mathcal{L}(\xb, \zb, \balpha, \bbeta)^T \mathbf{1}\}.$$
	Proceeding as the proof of proposition \ref{prop:as}, with $\left(\xb, \zb\right)$ in lieu of $(\xb^*, \zb^*)$, we obtain that $\bbeta_{k,I^*_k}$ is the set of Lagrangian multipliers for the most violated constraint from $\mathcal{A}_{\mathcal{E}, k}$ at $\left(\xb, \zb\right)$ and can be computed through the oracle \eqref{eq:oracle} by \citet{Anderson2020}. 
	
	Then, $\bbeta_{k,I^\dagger_k}$ is computed as: $$\bbeta_{k,I^\dagger_k} \in \argmax_{\bbeta_{k, I_k} \in\ \{\bbeta_{k,I^*_k},\ \bbeta_{k,0},\ \bbeta_{k,1}\}} \nabla_{\bbeta_{k, I_k}}\mathcal{L}(\xb, \zb, \balpha, \bbeta)^T \mathbf{1}.$$
	As pointed out in the proof of proposition \ref{prop:as}, the dual gradients of the Lagrangian correspond (by definition of Lagrangian multiplier) to constraint violations. Hence, $\bbeta_{k,I^\dagger_k}$ is associated to the most violated constraint in $\mathcal{A}_{k}$ .
\end{proof}

\subsubsection{Convex Combinations}
The $(t+1)$-th SP-FW iterate will be given by a convex combination of the $t$-th iterate and the current conditional gradient.
Due to the linearity of $\bzeta(\bbeta)$, we can perform the operation via sufficient statistics. Therefore, all the operations of Saddle Point occur in the linearly-sized space of $(\xb, \zb, \balpha, \bzeta(\bbeta))$:
\begin{equation*}
[\xb^{t+1}, \zb^{t+1}, \balpha^{t+1}, \bzeta(\bbeta^{t+1})] = (1 - \gamma_{t}) [\xb^{t}, \zb^{t}, \balpha^{t}, \bzeta(\bbeta^{t})] + \gamma_t [\bar{\xb}^{t}, \bar{\zb}^{t}, \bar{\balpha}^{t}, \bar{\bzeta}(\bbeta^{t})],
\end{equation*}
where, for our bilinear objective, SP-FW prescribes $\gamma_t = \frac{1}{1 + t}$~\citep[section 5]{Gidel17}.  

Extensions of SP-FW, such as its away or pairwise step variants, require a worst-case memory cost that is linear in the number of iterations. In other words, as for Active Set, the attainable tightness would depend on the available memory, voiding one of the main advantages of Saddle Point (we provide empirical evidence of its memory efficiency in $\S$\ref{sec:exp-memory}).
Furthermore, the worst-case memory cost would increase more rapidly than for Active Set, which infrequently adds a few variables to $\mathcal{B}$ (in our experiments, $\mathcal{B}$ contains at most $7$ variables per neuron: see $\S$\ref{sec:incomp_verif}).
Finally, due to the bilinearity of the objective, these SP-FW variants do not correspond to an improved convergence rate~\citep{Gidel17}.

\subsubsection{Initialization} 
As for the Active Set solver ($\S$\ref{sec:active}), dual variables can be initialized via supergradient ascent on the set of dual variables associated to the Big-M relaxation (cf. appendix \ref{sec:dual-init}). 
Additionally, if the available memory permits it, the initialization can be tightened by running Active Set (algorithm \ref{alg:active-set}) for a small fixed number of iterations.

We mirror this strategy for the primal variables, which are initialized by performing subgradient descent on the primal view of saddle point problem \eqref{eq:sp-fw-dual}. 
Analogously to the dual case, the primal view of problem \eqref{eq:sp-fw-dual} can be restricted to the Big-M relaxation for a cheaper initialization. Our primal initialization strategy is detailed in in appendix \ref{sec:sp-fw-primalinit}.

\section{Implementation Details, Technical Challenges} \label{sec:implementation-details}

In this section, we present details concerning the implementation of our solvers. In particular, we first outline our parallelization scheme and the need for a specialized convolutional operator ($\S$\ref{sec:parallel-masked}), then describe how to efficiently employ our solvers within branch and bound ($\S$\ref{sec:stratified}).

\subsection{Parallelism, Masked Forward/Backward Passes} \label{sec:parallel-masked}
Analogously to previous dual algorithms~\citep{Dvijotham2018,BunelDP20}, our solvers can leverage the massive parallelism offered by modern GPU architectures in three different ways.
First, we execute in parallel the computations of lower and upper bounds relative to all the neurons of a given layer. Second, in complete verification, we can batch over the different Branch and Bound (BaB) subproblems. 
Third, as most of our solvers rely on standard linear algebra operations employed during the forward and backward passes of neural networks, we can exploit the highly optimized implementations commonly found in modern deep learning frameworks. 

\vspace{5pt}
An exception are what we call ``masked" forward and backward passes. Writing convolutional operators in the form of their equivalent linear operator (as done in previous sections, see $\S$\ref{sec:preliminaries}), masked passes take the following form:
\begin{equation*}
	\left(W_k \odot I_{k}\right) \mbf{a}_{k}, \quad \left(W_k \odot I_{k}\right)^T \mbf{a}_{k+1},
\end{equation*}
where $\mbf{a}_{k} \in \mathbb{R}^{n_k}, \mbf{a}_{k+1} \in \mathbb{R}^{n_{k+1}}$.
Both operators are needed whenever dealing with constraints from $\mathcal{A}_k$. In fact, they appear in both Saddle Point (for instance, in the sufficient statistics \eqref{eq:suff-stats}) and Active Set (see equations \eqref{eq:dual-active-set}, \eqref{eq:as-supergradient}), except for the dual initialization procedure based on the easier Big-M problem \eqref{eq:primal-bigm}.

Masked passes can be easily implemented for fully connected layers via Hadamard products.
However, a customized lower-level implementation is required for a proper treatment within convolutional layers. 
In fact, from the high-level perspective, masking a convolutional pass corresponds to altering the content of the convolutional filter while it is being slided through the image.
Details of our implementation can be found in appendix~\ref{sec:masked-passes}.

\subsection{Stratified Bounding for Branch and Bound} \label{sec:stratified}

In complete verification, we aim to solve the non-convex problem~\eqref{eq:noncvx_pb} directly, rather than an approximation like problem~\eqref{eq:primal-anderson}. 
In order to do so, we rely on branch and bound, which operates by dividing the problem domain into subproblems (branching) and bounding the local minimum over those domains. 
The lower bound on the minimum is computed via a bounding algorithm, such as our solvers ($\S$\ref{sec:active}, $\S$\ref{sec:sp}). The upper bound, instead, can be obtained by evaluating the network at an input point produced by the lower bound computation.\footnote{For subgradient-type methods like Active Set, we evaluate the network at $\xb^{*, T}_{0}$ (see algorithm \ref{alg:active-set}), while for Frank-Wolfe-type methods like Saddle Point at $\xb_{0}^{T}$ (see algorithm \ref{alg:sp-fw}). Running the bounding algorithm to get an upper bound would result in a much looser bound, as it would imply having an upper bound on a version of problem \eqref{eq:noncvx_pb} with maximization instead of minimization.}
Any domain that cannot contain the global lower bound is pruned away, whereas the others are kept and branched over.
The graph describing branching relations between sub-problems is referred to as the enumeration tree.
As tight bounding is key to pruning the largest possible number of domains, the bounding algorithm plays a crucial role. Moreover, it usually acts as a computational bottleneck for branch and bound~\citep{Lu2020Neural}.

\vspace{7pt} 
In general, tighter bounds come at a larger computational cost. The overhead can be linked to the need to run dual iterative algorithms for more iterations, or to the inherent complexity of tighter relaxations like problem \eqref{eq:primal-anderson}. For instance, such complexity manifests itself in the masked passes described in appendix \ref{sec:masked-passes}, which increase the cost per iteration of Active Set and Saddle Point with respect to algorithms operating on problem \eqref{eq:primal-bigm}.
These larger costs might negatively affect performance on easier verification tasks, where a small number of domain splits with loose bounds suffices to verify the property.
Therefore, as a general complete verification system needs to be efficient regardless of problem difficulty, we employ a \emph{stratified} bounding system within branch and bound.
Specifically, we devise a simple adaptive heuristic to determine whether a given subproblem is ``easy" (therefore looser bounds are sufficient) or whether it is instead preferable to rely on tighter bounds. 

\vspace{7pt}
Given a bounding algorithm $a$, let us denote its lower bound for subproblem $s$ as $l_a(s)$.
Assume two different bounding algorithms, $a_l$ and  $a_t$, are available: one inexpensive yet loose, the other tighter and more costly.
At the root $r$ of the branch and bound procedure, we estimate $l_{a_t-a_l} = l_{a_t}(r) - l_{a_l}(r)$, the extent to which the lower bounds returned by $a_l$ can be tightened by $a_t$. 
While exploring the enumeration tree, we keep track of the lower bound increase from parent to child (that is, after splitting the subdomain) through an exponential moving average. We write ${i(s)}$ for the average parent-to-child tightening until subproblem $s$.
Then, under the assumption that each subtree is complete, we can estimate $|s|_{a_l}$ and $|s|_{a_t}$, the sizes of the enumeration subtrees rooted at $s$ that would be generated by each bounding algorithm. Recall that, for verification problems the canonical form \citep{Bunel2018}, subproblems are discarded when their lower bound is positive. Given $p$, the parent of subproblem $s$, we perform the estimation as:
$|s|_{a_l} = 2^{\frac{-l_{a_l}(p)}{i(s)}+1}-1$, $|s|_{a_t} = 2^{\frac{ -\left(l_{a_l}(p) + l_{a_t-a_l} \right)}{i(s)}+1}-1$.
Then, relying on $c_{a_t/a_l}$, a rough estimate of the relative overhead of running $a_t$ over $a_l$, we mark the subtree rooted at $s$ as hard if the reduction in tree size from using $a_t$ exceeds its overhead. That is, if $\frac{|s|_{a_l}}{|s|_{a_t}} > c_{a_t/a_l}$, the lower bound for $s$ and its children will be computed via algorithm $a_t$ rather than~$a_l$.

\section{Related Work} \label{sec:related-work}
In addition to those described in $\S$\ref{sec:preliminaries}, many other relaxations have been proposed in the literature.
In fact, all bounding methods are equivalent to solving some convex relaxation of a neural network.
This holds for conceptually different ideas such as bound propagation~\citep{Gowal2018}, specific dual assignments~\citep{Wong2018}, dual formulations based on Lagrangian Relaxation~\citep{Dvijotham2018} or Lagrangian Decomposition~\citep{BunelDP20}.
The degree of tightness varies greatly: from looser relaxations associated to closed-form methods~\citep{Gowal2018,Weng2018,Wong2018} to tighter formulations based on Semi-Definite Programming (SDP)~\citep{Raghunathan2018}. 

The speed of closed-form approaches results from simplifying the triangle-shaped feasible region of the Planet relaxation ($\S$\ref{sec:relax-planet}) \citep{Singh2018a,Wang2018b}.
On the other hand, tighter relaxations are more expressive than the linearly-sized LP by \citet{Ehlers2017}.
The SDP formulation by \citet{Raghunathan2018} can represent interactions between activations in the same layer. Similarly, \citet{Singh2019b} tighten the Planet relaxation by considering the convex hull of the union of polyhedra relative to $k$ ReLUs of a given layer at once. 
Alternatively, tighter LPs can be obtained by considering the ReLU together with the affine operator before it:
standard MIP techniques~\citep{Jeroslow} lead to a formulation that is quadratic in the number of variables (see appendix \ref{sec:preact-anderson-derivation}).
The relaxation by \citet{Anderson2020} detailed in $\S$\ref{sec:relax-anderson} is a more convenient representation of the same~set. 

By projecting out the auxiliary $\zb$ variables, \citet{Tjandraatmadja} recently introduced another formulation equivalent to the one by \citet{Anderson2020}, with half as many variables and a linear factor more constraints compared to what described in $\S$\ref{sec:relax-anderson}. 
Therefore, the relationship between the two formulations mirrors the one between the Planet and Big-M relaxations (see appendix \ref{sec:planet-eq}). 
Our dual derivation and solvers can be adapted to operate on the projected relaxations.
Furthermore, the formulation by \citet{Tjandraatmadja} allows for a propagation-based method (``FastC2V"). However, such an algorithm tackles only two constraints per neuron at once and might hence yield looser bounds than the Planet relaxation. In this work, we are interested in designing solvers that can operate on strict subsets of the feasible region from problem \eqref{eq:primal-bigm}.

Specialized dual solvers significantly improve in bounding efficiency with respect to off-the-shelf solvers for both LP~\citep{BunelDP20} and SDP formulations~\citep{Dvijotham2019}.
Therefore, the design of similar solvers for other tight relaxations is an interesting line of future research.
We contribute with two specialized dual solvers for the relaxation by~\citet{Anderson2020}. 
In what follows, we demonstrate empirically that by meeting the requirements of Fact \ref{fact} we can obtain large incomplete and complete verification improvements.

\section{Experiments} \label{sec:exp}
We empirically demonstrate the effectiveness of our methods under two settings. First, we assess the speed and quality of our bounds compared to other bounding algorithms on incomplete verification ($\S$\ref{sec:incomp_verif}). Then, we examine whether our speed-accuracy trade-offs pay off within branch and bound ($\S$\ref{sec:comp_verif}).
The implementation of our algorithms, available at \linebreak \url{https://github.com/oval-group/oval-bab} as part of the OVAL neural network verification framework, is based on Pytorch \citep{Paszke2017}.

\subsection{Experimental Setting}

We compare both against dual iterative methods and Gurobi, which we use as gold standard for LP solvers. 
The latter is employed for the following two baselines:
\begin{itemize}
	\item \textbf{Gurobi Planet} means solving the Planet~\cite{Ehlers2017} relaxation of the network (a version of problem \eqref{eq:primal-bigm} for which $\zb$ have been projected out).
	\item \textbf{Gurobi cut} starts from the Big-M relaxation and adds constraints from $\mathcal{A}_k$ in a cutting-plane fashion, as the original primal algorithm by ~\citet{Anderson2020}. 
\end{itemize}
Both Gurobi-based methods make use of LP incrementalism (warm-starting) when possible. 
In the experiments of $\S$\ref{sec:incomp_verif}, where each image involves the computation of 9 different output upper bounds, we warm-start each LP from the LP of the previous neuron. For “Gurobi 1 cut”, which involves two LPs per neuron, we first solve all Big-M LPs, then proceed with the LPs containing a single cut.

In addition, our experimental analysis comprises the following dual iterative methods:
\begin{itemize}
	\item \textbf{BDD+}, the recent proximal-based solver by~\citet{BunelDP20}, operating on a Lagrangian Decomposition dual of the Planet relaxation.
	\item \textbf{Active Set} denotes our supergradient-based solver for problem \eqref{eq:primal-anderson}, described in~$\S$\ref{sec:active}.
	\item \textbf{Saddle Point}, our Frank-Wolfe-based solver for problem \eqref{eq:primal-anderson} (as presented in $\S$\ref{sec:sp}).
	\item By keeping $\mathcal{B}=\emptyset$, Active Set reduces to \textbf{Big-M}, a solver for the non-projected Planet relaxation (appendix \ref{sec:dual-init}), which is employed as dual initializer to both Active Set and Saddle~Point. 
	\item \textbf{AS-SP} is a version of Saddle Point whose dual initialization relies on a few iterations of Active Set rather than on the looser Big-M solver, hence combining both our dual approaches. 
\end{itemize}
As we operate on the same data sets employed by~\citet{BunelDP20}, we omit both their supergradient-based approach and the one by~\citet{Dvijotham2018}, as they both perform worse than BDD+~\citep{BunelDP20}. 
For the same reason, we omit cheaper (and looser) methods, like interval propagation~\citep{Gowal2018} and the one by~\citet{Wong2018}.
In line with previous bounding algorithms \citep{BunelDP20}, we employ Adam updates~\citep{Kingma2015} for supergradient-type methods due to their faster empirical convergence.
While dual iterative algorithms are specifically designed to take advantage of GPU acceleration (see $\S$\ref{sec:parallel-masked}), we additionally provide a CPU implementation of our solvers in order to complement the comparison with Gurobi-based methods.

\vspace{7pt}
Unless otherwise stated, experiments were run under the following setup: Ubuntu 16.04.2 LTS, on a single Nvidia Titan Xp GPU, except those based on Gurobi and the CPU version of our solvers. 
The latter were run on i7-6850K CPUs, utilising $4$ cores for the incomplete verification experiments, and $6$ cores for the more demanding complete verification setting.

\subsection{Incomplete Verification} \label{sec:incomp_verif}
We evaluate the efficacy of our bounding algorithms in an incomplete verification setting by upper bounding the vulnerability to adversarial perturbations~\citep{Szegedy2014}, measured as the difference between the logits associated to incorrect classes and the one corresponding to the ground truth, on the CIFAR-10 test set~\citep{CIFAR10}. 
If the upper bound is negative, we can certify the network's robustness to adversarial perturbations. 

\subsubsection{Speed-Accuracy Trade-Offs} \label{sec:exp-inc-sgd8}
Here, we replicate the experimental setting from~\citet{BunelDP20}. The networks correspond to the small network architecture from~\citet{Wong2018}, and to the ``Wide" architecture, also employed for complete verification experiments in $\S$\ref{sec:exp-comp-sub}, found in Table \ref{tab:problem_size}.
Due to the additional computational cost of bounds obtained via the tighter relaxation \eqref{eq:primal-anderson}, we restricted the experiments to the first $2567$ CIFAR-10 test set images for the experiments on the SGD-trained network (Figures \ref{fig:sgd8}, \ref{fig:sgd8-pointwise}), and to the first $4129$ images for the network trained via the method by \citet{Madry2018} (Figures \ref{fig:madry8}, \ref{fig:madry8-pointwise}).

Here, we present results for a network trained via standard SGD and cross entropy loss, with no modification to the objective for robustness. Perturbations for this network lie in a $\ell_{\infty}$ norm ball with radius $\epsilon_{ver} = 1.125/255$ (which is hence lower than commonly employed radii for robustly trained networks).
In appendix \ref{sec:exp-appendix}, we provide additional CIFAR-10 results on an adversarially trained network using the method by~\citet{Madry2018}, and on MNIST~\citep{LeCun1998}, for a network trained with the verified training algorithm by~\citet{Wong2018}. 

Solver hyper-parameters were tuned on a small subset of the CIFAR-10 test set.
BDD+ is run with the hyper-parameters found by \citet{BunelDP20} on the same data sets, for both incomplete and complete verification.
For all supergradient-based methods (Big-M, Active Set), we employed the Adam update rule~\citep{Kingma2015}, which showed stronger empirical convergence. 
For Big-M, replicating the findings by \citet{BunelDP20} on their supergradient method, we linearly decrease the step size from $10^{-2}$ to $10^{-4}$. 
Active Set is initialized with $500$ Big-M iterations, after which the step size is reset and linearly scaled from $10^{-3}$ to $10^{-6}$. 
We found the addition of variables to the active set to be effective before convergence: we add variables every $450$ iterations, without re-scaling the step size again. Every addition consists of $2$ new variables (see algorithm \ref{alg:active-set}) and we cap the maximum number of cuts to $7$. This was found to be a good compromise between fast bound improvement and computational cost. 
For Saddle Point, we use $1/(t + 10)$ as step size to stay closer to the initialization points. The primal initialization algorithm (see appendix \ref{sec:sp-fw-primalinit}) is run for $100$ steps on the Big-M variables, with step size linearly decreased from $10^{-2}$ to $10^{-5}$. 

\begin{figure*}[t!]
	\vspace{-10pt}
	\centering
	\begin{subfigure}{\textwidth}
		\includegraphics[width=1\textwidth]{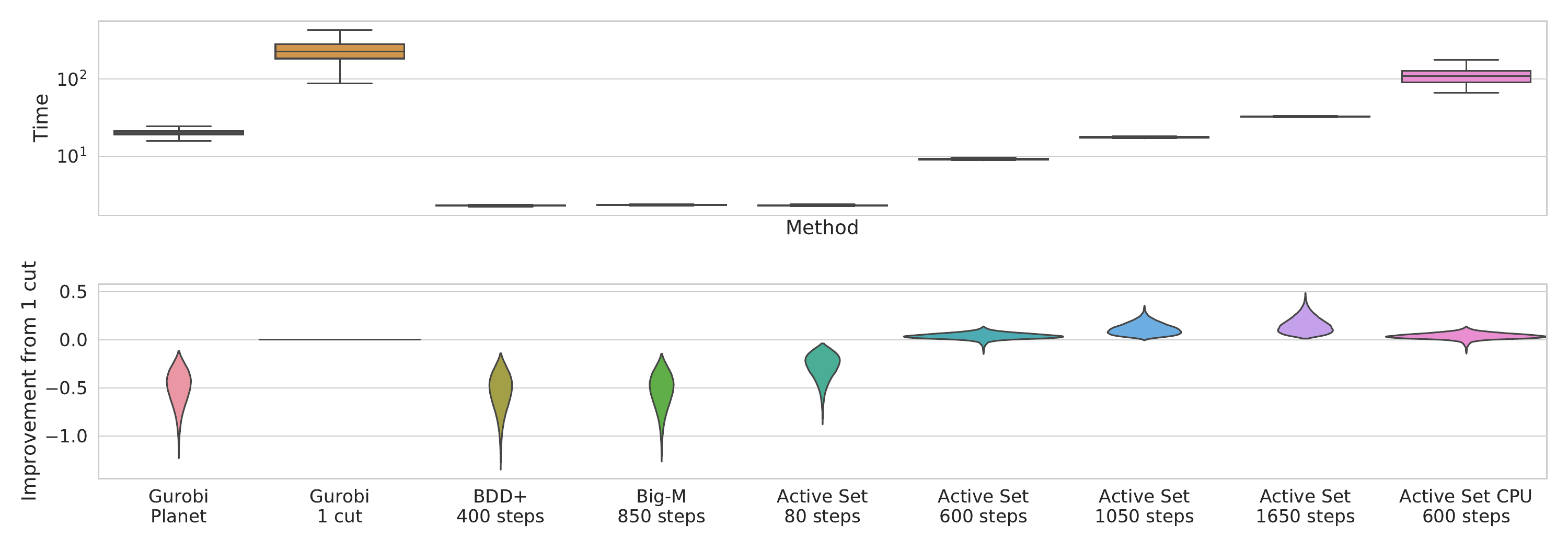}
		\vspace{-20pt}
		\subcaption{Speed-accuracy trade-offs of Active Set for different iteration ranges and computing devices.}
		\vspace{5pt}
	\end{subfigure}
	\begin{subfigure}{\textwidth}
		\includegraphics[width=1\textwidth]{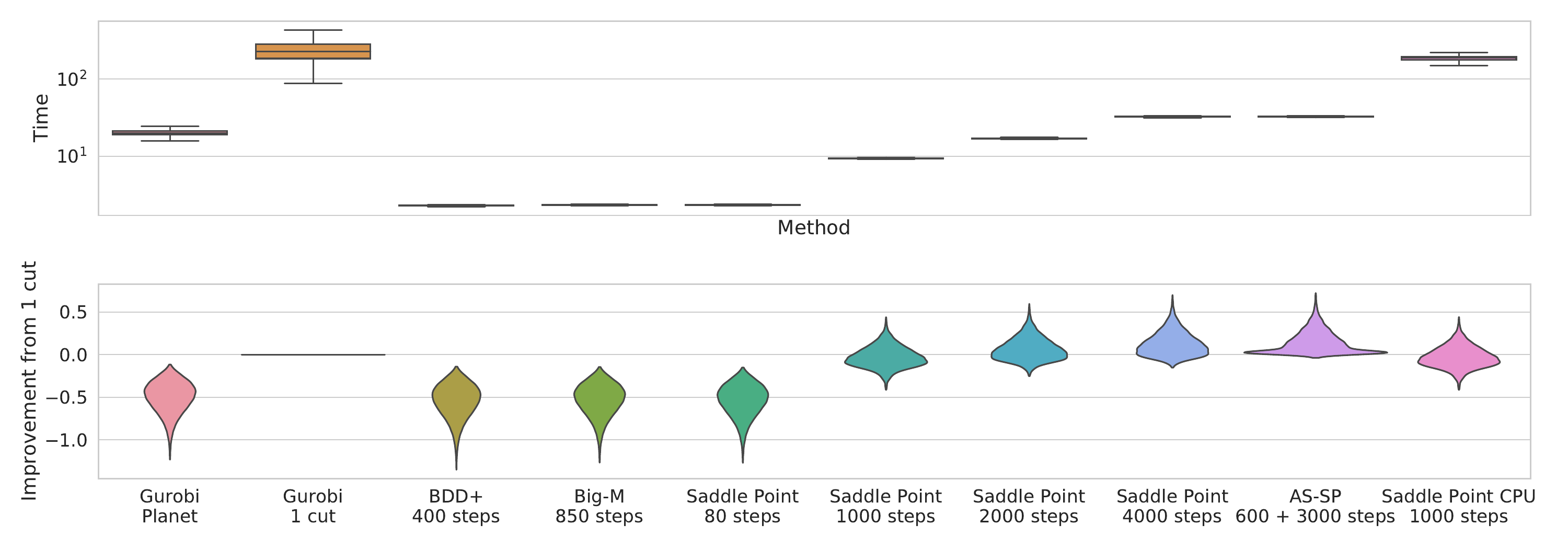}
		\vspace{-20pt}
		\subcaption{Speed-accuracy trade-offs of Saddle Point for different iteration ranges and computing devices.}
	\end{subfigure}
	
	\caption{\label{fig:sgd8} 
		Upper bounds to the adversarial vulnerability for the SGD-trained network from \citet{BunelDP20}. 
		Box plots: distribution of runtime in seconds.
		Violin plots: difference with the bounds obtained by Gurobi with a cut from $\mathcal{A}_k$ per neuron; higher is better, the width at a given value represents the
		proportion of problems for which this is the result. 
		On average, both Active Set and Saddle Point achieve bounds tighter than Gurobi 1 cut with a smaller runtime.
	}
	\vspace{-10pt}
\end{figure*}

\begin{figure*}[b!]
	\vspace{-8pt}
	\begin{subfigure}{.99\textwidth}
		\begin{minipage}{.48\textwidth}
			\centering
			\includegraphics[width=\textwidth]{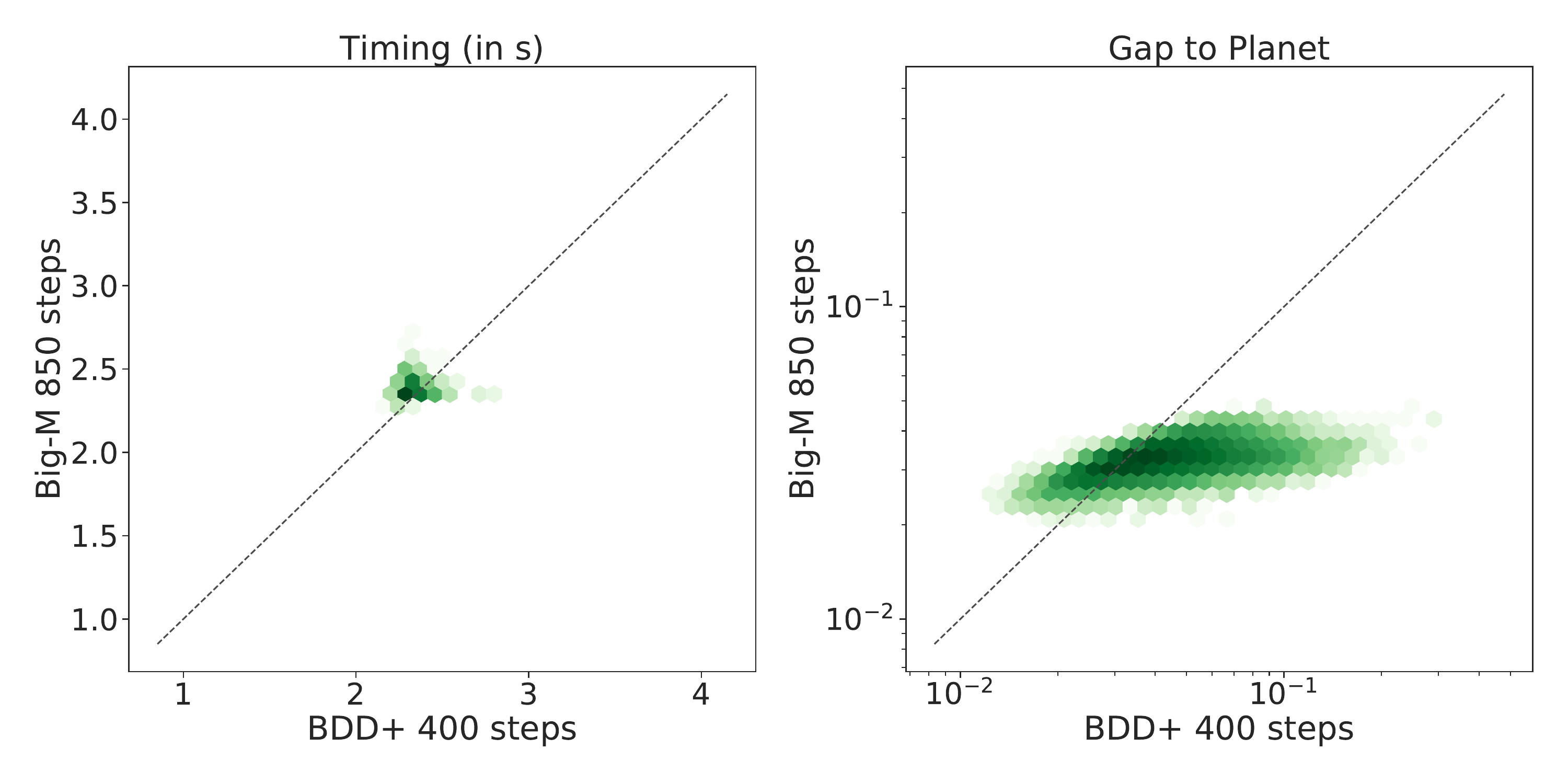}
			\vspace{-20pt}
		\end{minipage}
		\hspace{5pt}
		\begin{minipage}{.4\textwidth}
			\subcaption{ \label{fig:sgd-planet-gap} \small Comparison of runtime (left) and gap to Gurobi Planet bounds (right). For the latter, lower is better.}
		\end{minipage}
		\vspace{8pt}
	\end{subfigure}
	\begin{subfigure}{.99\textwidth}
		\begin{minipage}{.48\textwidth}
			\centering
			\includegraphics[width=\textwidth]{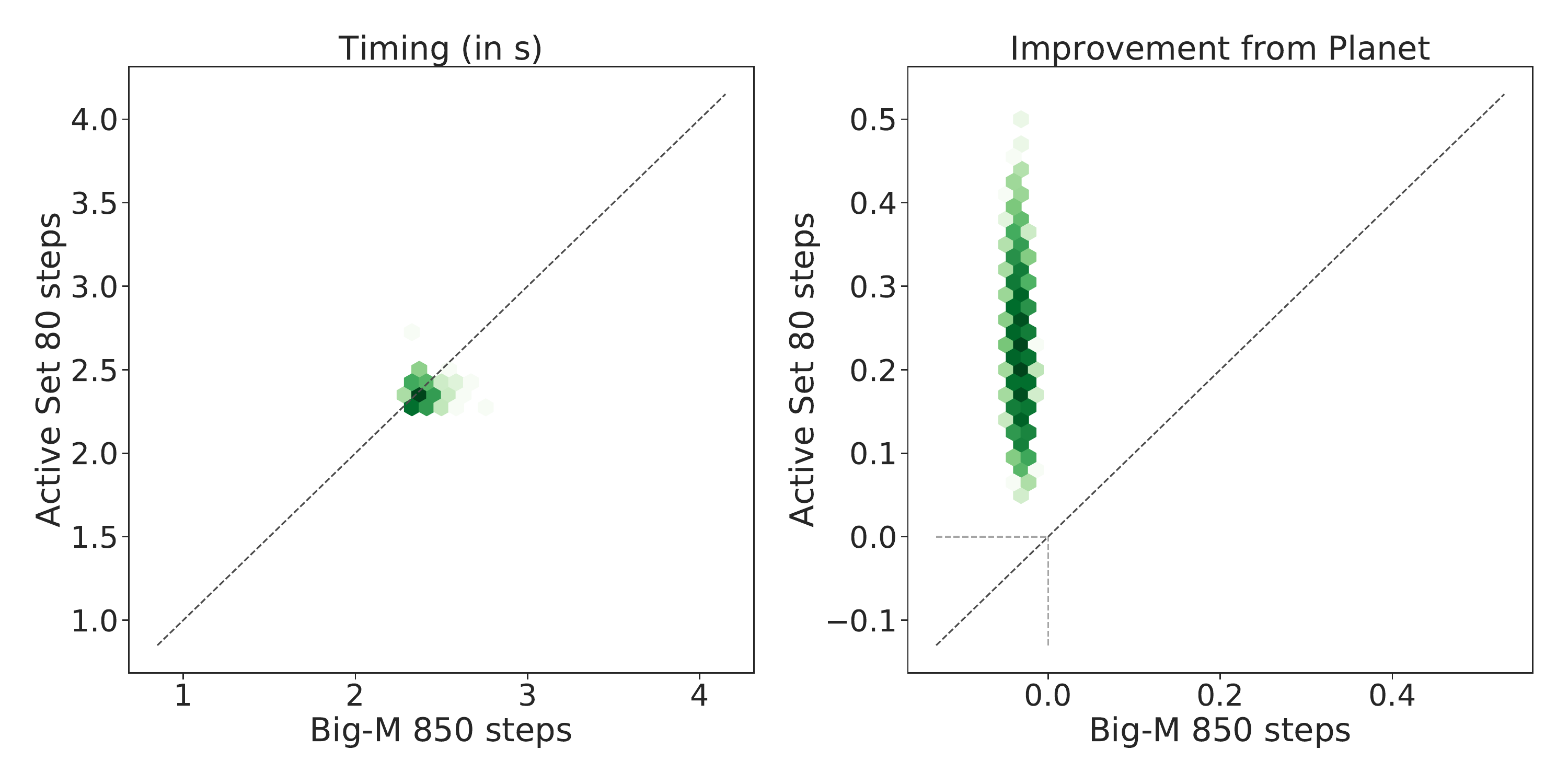}
		\end{minipage}
		\hspace{5pt}
		\begin{minipage}{.48\textwidth}
			\centering
			\includegraphics[width=\textwidth]{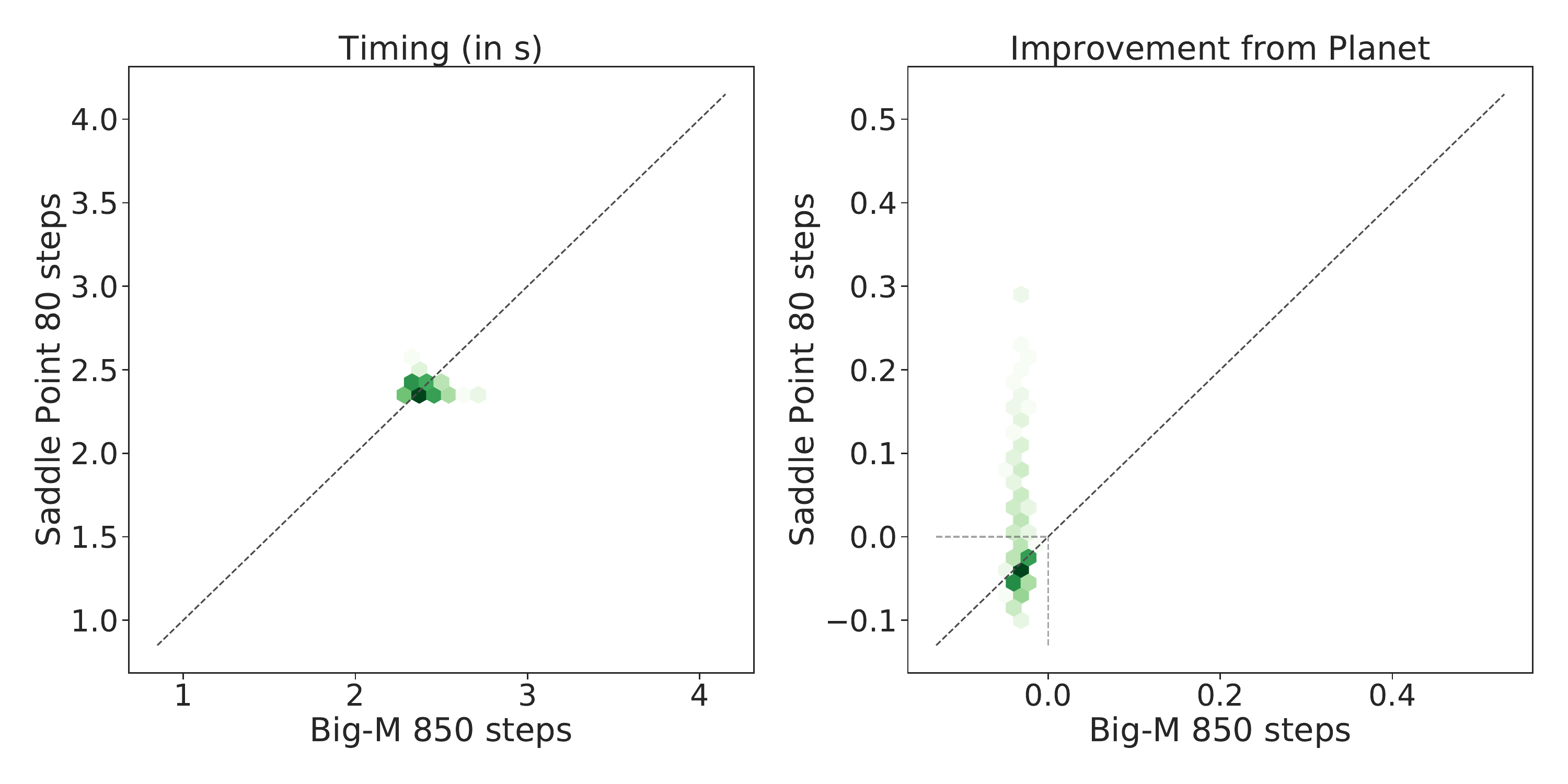}
		\end{minipage}
		\vspace{-8pt}
		\subcaption{Comparison of runtime (Timing) and difference with the Gurobi Planet bounds (Improvement from Planet). For the latter, higher is better.}
		\label{fig:sgd-planet-improvement}
	\end{subfigure}
	\vspace{-5pt}
	\caption{Pointwise comparison for a subset of the methods on the data presented in Figure \ref{fig:sgd8}. Darker colour shades mean higher point density (on a logarithmic scale). The oblique dotted line corresponds to the equality.}
	\label{fig:sgd8-pointwise}
\end{figure*}

\vspace{7pt}

Figure \ref{fig:sgd8} shows the distribution of runtime and the bound improvement with respect to Gurobi cut for the SGD-trained network.
For Gurobi cut, we only add the single most violated cut from $\mathcal{A}_k$ per neuron, due to the cost of repeatedly solving the LP.
We tuned BDD+ and Big-M, the dual methods operating on the weaker relaxation \eqref{eq:primal-bigm}, to have the same average runtime. They obtain bounds comparable to Gurobi Planet in one order less time. 
Initialized from $500$ Big-M iterations, at $600$ iterations, Active Set already achieves better bounds on average than Gurobi cut in around $1/20^{{th}}$ of the time. With a computational budget twice as large ($1050$ iterations) or four times as large ($1650$ iterations), the bounds significantly improve over Gurobi cut in a fraction of the time.
Similar observations hold for Saddle Point which, especially when using fewer iterations, also exhibits a larger variance in terms of bounds tightness.
In appendix \ref{sec:exp-appendix}, we empirically demonstrate that the tightness of the Active Set bounds is strongly linked to our active set strategy (see $\S$\ref{sec:active-set-descr}).
Remarkably, even if our solvers are specifically designed to take advantage of GPU acceleration, executing them on CPU proves to be strongly competitive with Gurobi cut. Active Set produces better bounds in less time for the benchmark of Figure \ref{fig:sgd8}, while Saddle Point yields comparable speed-accuracy trade-offs. \\ 

Figure \ref{fig:sgd8-pointwise} shows pointwise comparisons for the less expensive methods from Figure~\ref{fig:sgd8}, on the same data. 
Figure \ref{fig:sgd-planet-gap} shows the gap to the (Gurobi) Planet bound for BDD+ and our Big-M solver. Surprisingly, our Big-M solver is competitive with BDD+, achieving on average better bounds than BDD+, in the same time. 
Figure \ref{fig:sgd-planet-improvement} shows the improvement over Planet bounds for Big-M compared to those of few ($80$) Active Set and Saddle Point iterations. Active Set returns markedly better bounds than Big-M in the same time, demonstrating the benefit of operating (at least partly) on the tighter dual \eqref{eq:dual-anderson}. On the other hand, Saddle Point is rarely beneficial with respect to Big-M when running it for a~few~iterations. \\

\begin{figure*}[t!]
	\centering
	\begin{minipage}{.48\textwidth}
		\centering
		\includegraphics[width=\textwidth]{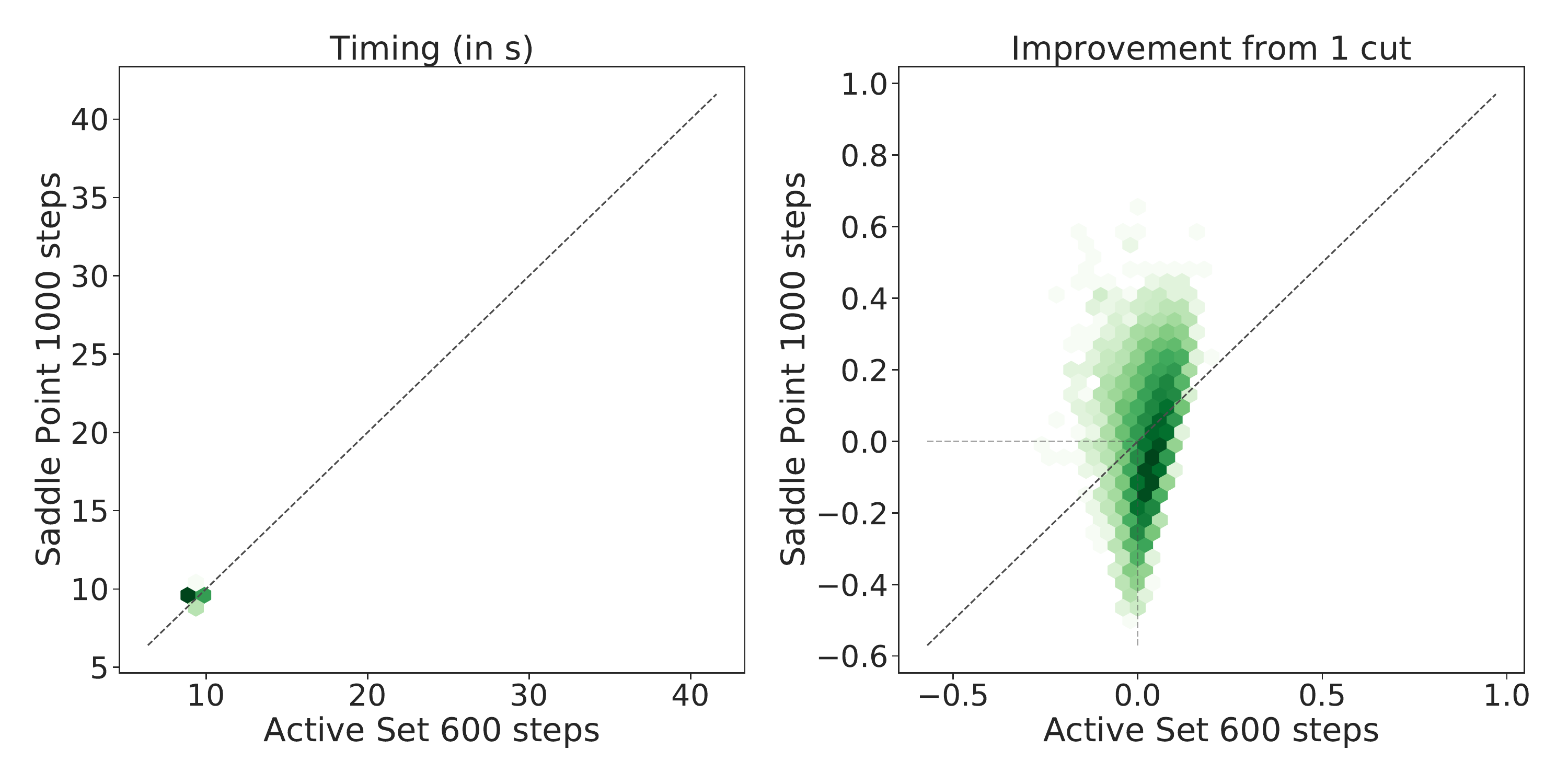}
	\end{minipage}
	\hspace{5pt}
	\begin{minipage}{.48\textwidth}
		\centering
		\includegraphics[width=\textwidth]{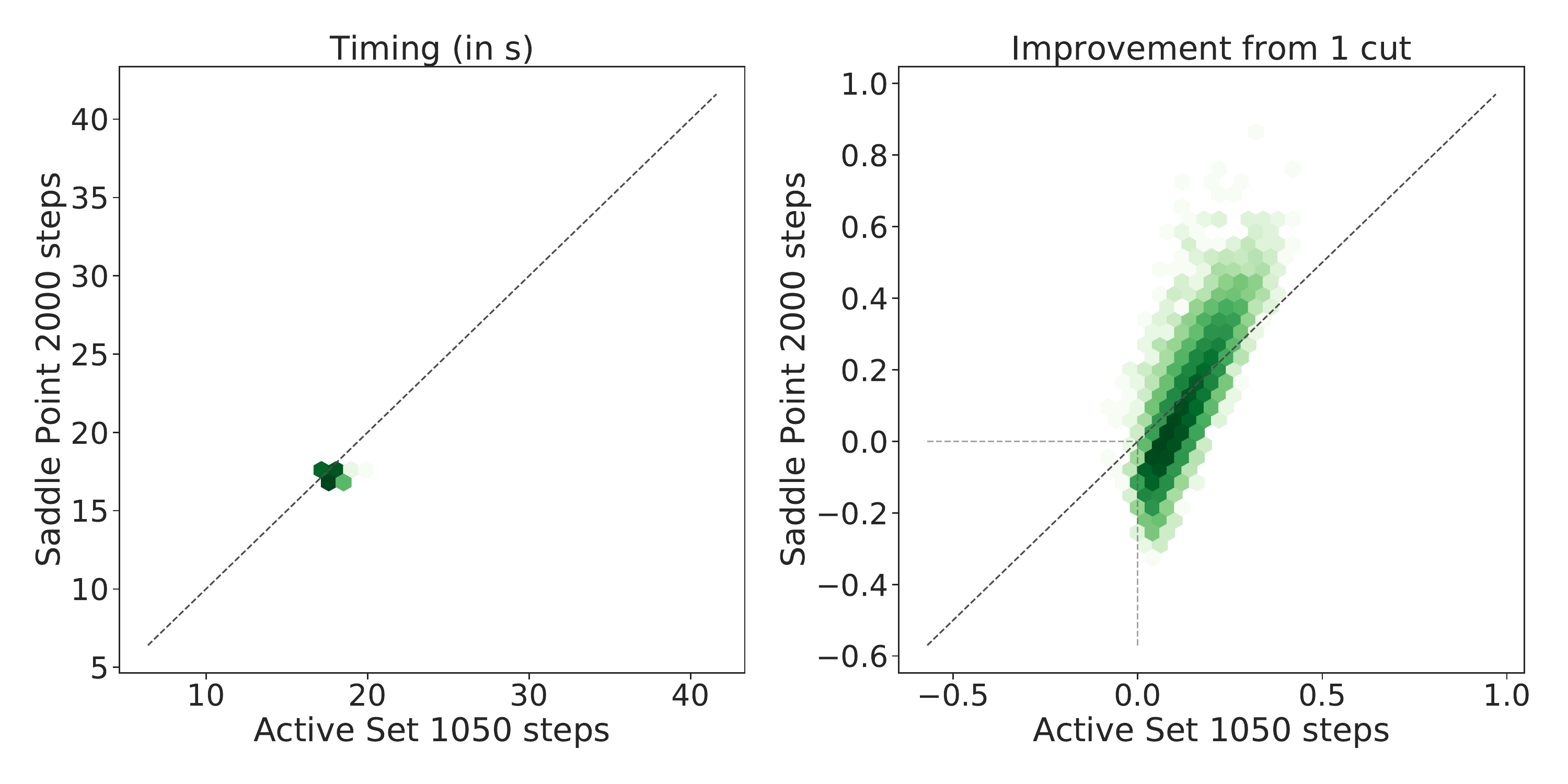}
	\end{minipage}
	\vspace{5pt}
	\begin{minipage}{.48\textwidth}
		\centering
		\includegraphics[width=\textwidth]{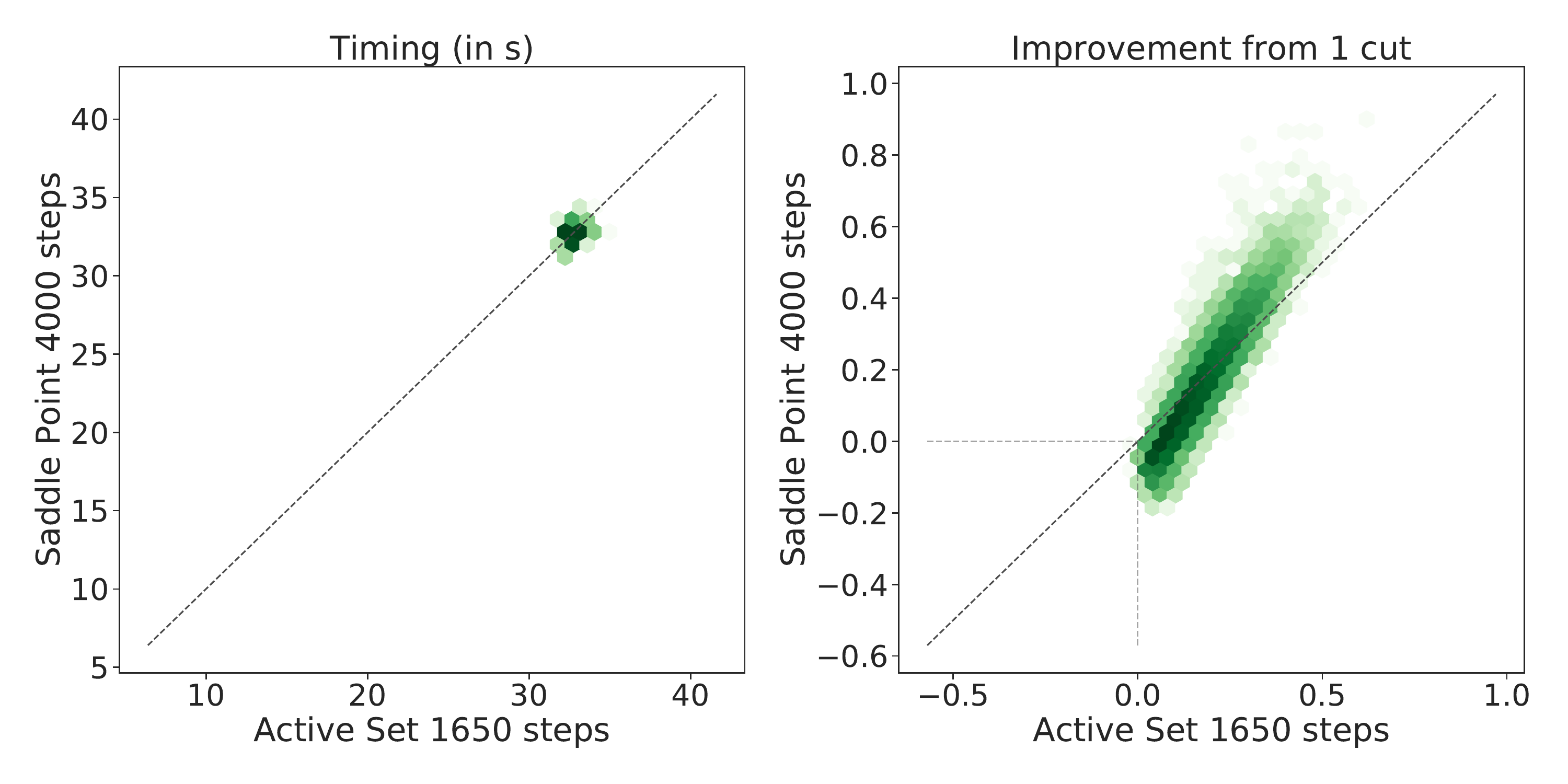}
	\end{minipage}
	\hspace{5pt}
	\begin{minipage}{.48\textwidth}
		\centering
		\includegraphics[width=\textwidth]{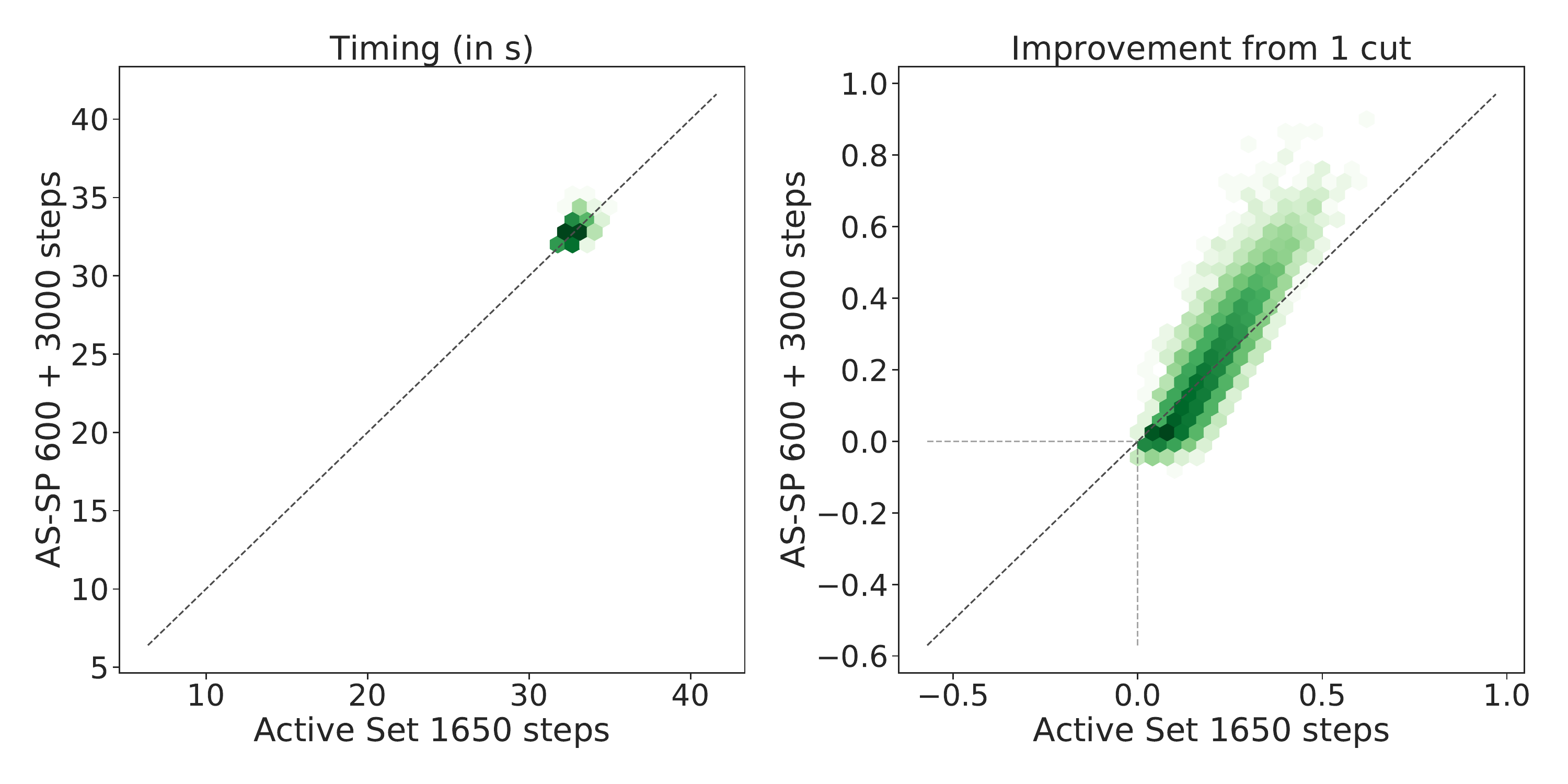}
	\end{minipage}
	\vspace{-15pt}
	\caption{Pointwise comparison between our proposed solvers on the data presented in Figure \ref{fig:sgd8}. Darker colour shades mean higher point density (on a logarithmic scale). The oblique dotted line corresponds to the equality. \label{fig:sgd-as-vs-spfw}}
\end{figure*}

Figure \ref{fig:sgd-as-vs-spfw} compares the performance of Active Set and Saddle Point for different runtimes, once again on the data from Figure~\ref{fig:sgd8}. While Active Set yields better average bounds when fewer iterations are employed, the performance gap shrinks with increasing computational budgets, and AS-SP (Active Set initialization to Saddle Point) yields tighter average bounds than Active Set in the same time. 
Differently from Active Set, the memory footprint of Saddle Point does not increase with the number of iterations (see $\S$\ref{sec:sp}). Therefore, we believe the Frank-Wolfe-based algorithm is particularly convenient in incomplete verification settings that require tight bounds.

\subsubsection{Memory Efficiency} \label{sec:exp-memory}

As stated in $\S$\ref{sec:sp}, one of the main advantages of Saddle Point is its memory efficiency. In fact, differently from Active Set, the attainable bounding tightness will not depend on the available memory. 
In order to illustrate this, we present results on a large fully connected network with two hidden layers of width $7000$. We trained the network adversarially~\citep{Madry2018} against perturbations of size $\epsilon_{train}=2/255$, which is the same radius that we employ at verification time. 
On the Nvidia Titan Xp GPU employed for our experiments, Active Set is only able to include a single constraint from $\mathcal{A}_{\mathcal{E}, k}$ per neuron without running out of memory. Except the maximum allowed number of cuts for Active Set, we run all algorithms with the same hyper-parameters as in $\S$\ref{sec:exp-inc-sgd8}. We conduct the experiment on the first $500$ images of the CIFAR-10 test set.
Figure~\ref{fig:adv_fc} shows that, while Active Set is competitive with Saddle Point when both are run for a few iterations, Saddle Point yields significantly better speed-accuracy trade-offs when both algorithms are run for longer. Indeed, the use of a single tightening constraint per neuron severely limits the tightness attainable by Active Set, which yields looser bounds than Gurobi cut on average. This is dissimilar from Figure~\ref{fig:sgd8}, where Active Set rapidly overcomes Gurobi cut in terms of bounding tightness.
On the other hand, Saddle Point returns bounds that are markedly tighter than those from the primal baselines in a fraction of their runtime, highlighting its benefits~in~memory-intensive~settings.

\begin{figure*}[t!]
	\renewcommand{\figurename}{Figure}
	\vspace{-10pt}
	\centering
	\includegraphics[width=1\textwidth]{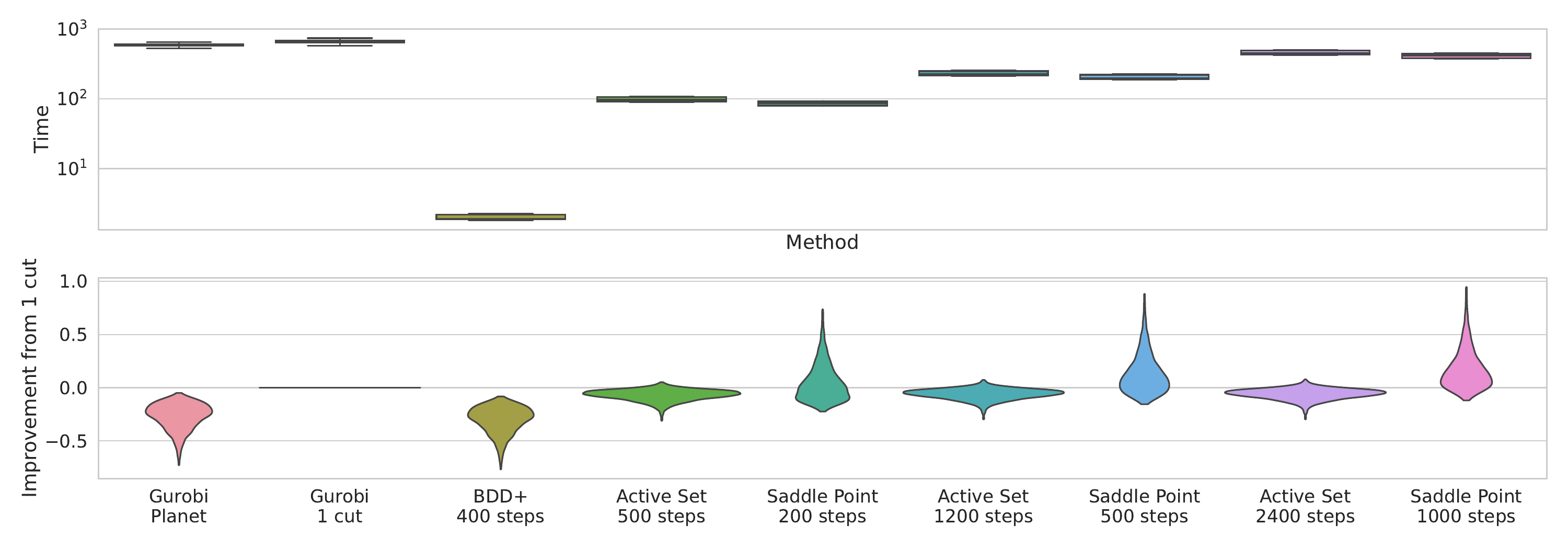}
	\vspace{-20pt}
	\caption{\label{fig:adv_fc} 
	Upper bounds to the adversarial vulnerability of a fully connected network with two hidden layers of width $7000$. 
	Box plots: distribution of runtime in seconds.
	Violin plots: difference with the bounds obtained by Gurobi with a cut from $\mathcal{A}_k$ per neuron; higher is better. 
	When run for enough iterations, Saddle Point achieves bounds tighter than both Gurobi 1 cut and Active Set, whose tightness is constrained by memory, in less time. 
	}
	\vspace{-10pt}
\end{figure*}

\subsection{Branch and Bound} \label{sec:comp_verif}

We now assess the effectiveness of our algorithms within branch and bound (see $\S$\ref{sec:stratified}). In particular, we will employ them within BaBSR~\citep{Bunel2020}. 
In BaBSR, branching is carried out by splitting an unfixed ReLU into its passing and blocking phases (see $\S$\ref{sec:stratified} for a description of branch and bound). 
In order to determine which ReLU to split on, BaBSR employs an inexpensive heuristic based on the bounding algorithm by \citet{Wong2018}. The goal of the heuristic is to assess which ReLU induces maximum change in the domain's lower bound when made unambiguous.

\subsubsection{Complete Verification} \label{sec:exp-comp-sub}

We evaluate the performance on complete verification by verifying the adversarial robustness of a network to perturbations in $\ell_{\infty}$ norm on a subset of the data set by~\citet{Lu2020Neural}. We replicate the experimental setting from~\citet{BunelDP20}.  \\

\citet{Lu2020Neural} provide, for a subset of the CIFAR-10 test set, a verification radius $\epsilon_{ver}$ defining the small region over which to look for adversarial examples (input points for which the output of the network is not the correct class)  and a (randomly sampled) non-correct class to verify against.  
The verification problem is formulated as the search for an adversarial example, carried out by minimizing the difference between the ground truth logit and the target logit. If the minimum is positive, we have not succeeded in finding a counter-example, and the network is robust.
The $\epsilon_{ver}$ radius was tuned to meet a certain ``problem difficulty" via binary search, employing a Gurobi-based bounding algorithm~\citep{Lu2020Neural}. 
In particular, \citet{Lu2020Neural} chose perturbation radii to rule out properties for which an adversarial example can be rapidly found, and properties for which Gurobi Planet would be able to prove robustness without any branching.
This characteristic makes the data set an appropriate testing ground for tighter relaxations like the one by~\citet{Anderson2020} ($\S$\ref{sec:relax-anderson}).
The networks are robust on all the properties we employed.
Three different network architectures of different sizes are used, all robustly trained for $\epsilon_{train}=2/255$ with the algorithm by \citet{Wong2018}. A ``Base'' network with $3172$ ReLU activations, and two networks with roughly twice as many activations: one ``Deep'', the other ``Wide''. Details can be found in Table \ref{tab:problem_size}. We restricted the original data set to $100$ properties per network so as to mirror the setup of the recent VNN-COMP competition~\citep{vnn-comp}.
The properties have an average perturbation radius of $\epsilon_{ver} = \nicefrac{10.1}{255}$, $\epsilon_{ver} = \nicefrac{6.9}{255}$,  $\epsilon_{ver} = \nicefrac{7.1}{255}$ for the Base, Wide, and Deep networks, respectively. \\

\begin{table}[t!]
	\centering
	\small
	\begin{minipage}{.9\textwidth}
		\begin{tabular}{|c|c|c|}
			\hline
			\textbf{Network Name}\TBstrut & \textbf{No. of Properties} \TBstrut & \textbf{Network Architecture} \TBstrut\\
			\hline
			\begin{tabular}{l}
				BASE \\ 
				Model
			\end{tabular} & \begin{tabular}{@{}c@{}}
				100
			\end{tabular} & \begin{tabular}{@{}c@{}} 
				\footnotesize
				Conv2d(3,8,4, stride=2, padding=1) \Tstrut \\
				\footnotesize Conv2d(8,16,4, stride=2, padding=1)\\
				\footnotesize linear layer of 100 hidden units \\
				\footnotesize linear layer of 10 hidden units\\
				(Total ReLU activation units: 3172) \Bstrut
			\end{tabular}\\
			\hline
			WIDE & 100 & \begin{tabular}{@{}c@{}} 
				\footnotesize
				Conv2d(3,16,4, stride=2, padding=1) \Tstrut\\
				\footnotesize Conv2d(16,32,4, stride=2, padding=1)\\
				\footnotesize linear layer of 100 hidden units \\
				\footnotesize linear layer of 10 hidden units\\
				(Total ReLU activation units: 6244) \Bstrut
			\end{tabular}\\
			\hline
			DEEP & 100 & \begin{tabular}{@{}c@{}} 
				\footnotesize
				Conv2d(3,8,4, stride=2, padding=1) \Tstrut \\
				\footnotesize Conv2d(8,8,3, stride=1, padding=1) \\
				\footnotesize Conv2d(8,8,3, stride=1, padding=1) \\
				\footnotesize Conv2d(8,8,4, stride=2, padding=1)\\
				\footnotesize linear layer of 100 hidden units \\
				\footnotesize linear layer of 10 hidden units\\
				(Total ReLU activation units: 6756) \Bstrut
			\end{tabular}\\
			\hline
		\end{tabular}
	\end{minipage}
	\caption{\label{tab:problem_size} For each complete verification experiment, the network architecture used and the number of verification properties tested, a subset of the data set by \citet{Lu2020Neural}. Each layer but the last is followed by a ReLU activation function.}
\end{table}

We compare the effect on final verification time of using the different bounding methods in $\S$\ref{sec:incomp_verif} within BaBSR. 
When stratifying two bounding algorithms (see $\S$\ref{sec:stratified}) we denote the resulting method by the names of both the looser and the tighter bounding method, separated by a plus sign (for instance, \textbf{Big-M~+~Active Set}). 
In addition, we compare against the following complete verification~algorithms:
\begin{itemize}
	\item \textbf{MIP $\mathcal{A}_k$} encodes problem \eqref{eq:noncvx_pb} as a Big-M MIP~\citep{Tjeng2019} and solves it in Gurobi by adding cutting planes from $\mathcal{A}_k$. This mirrors the original experiments from~\citet{Anderson2020}. 
	\item \textbf{ERAN}~\citep{eran}, a state-of-the-art complete verification toolbox. Results on the data set by~\citet{Lu2020Neural} are taken from the recent VNN-COMP competition\footnote{These were executed by \citet{eran} on a 2.6 GHz Intel Xeon CPU E5-2690 with 512 GB of main memory, utilising $14$ cores.}~\citep{vnn-comp}.
	\item \textbf{Fast-and-Complete}~\citep{xu2021fast}: a recent complete verifier based on BaBSR that pairs fast dual bounds with Gurobi Planet to obtain state-of-the-art performance.
\end{itemize}
\sisetup{detect-weight=true,detect-inline-weight=math,detect-mode=true}
\begin{table*}[t!]
	\centering
	\scriptsize
	\setlength{\tabcolsep}{4pt}
	\aboverulesep = 0.1mm  
	\belowrulesep = 0.2mm  
	\renewcommand{\bfseries}{\fontseries{b}\selectfont}
	\newrobustcmd{\B}{\bfseries}
	\newcommand{\boldentry}[2]{
		\multicolumn{1}{S[table-format=#1,
			mode=text,
			text-rm=\fontseries{b}\selectfont
			]}{#2}}
	\begin{adjustbox}{max width=\textwidth, center}
		\begin{tabular}{
				l
				S[table-format=4.3]
				S[table-format=4.3]
				S[table-format=4.3]
				S[table-format=4.3]
				S[table-format=4.3]
				S[table-format=4.3]
				S[table-format=4.3]
				S[table-format=4.3]
				S[table-format=4.3]
			}
			& \multicolumn{3}{ c }{Base} & \multicolumn{3}{ c }{Wide} & \multicolumn{3}{ c }{Deep} \\
			\toprule
			
			\multicolumn{1}{ c }{Method} &
			\multicolumn{1}{ c }{time(s)} &
			\multicolumn{1}{ c }{sub-problems} &
			\multicolumn{1}{ c }{$\%$Timeout} &
			\multicolumn{1}{ c }{time(s)} &
			\multicolumn{1}{ c }{sub-problems} &
			\multicolumn{1}{ c }{$\%$Timeout} &
			\multicolumn{1}{ c }{time(s)} &
			\multicolumn{1}{ c }{sub-problems} &
			\multicolumn{1}{ c }{$\%$Timeout} \\
			
			\cmidrule(lr){1-1} \cmidrule(lr){2-4} \cmidrule(lr){5-7} \cmidrule(lr){8-10}
			
			\multicolumn{1}{ c }{\tiny \textsc{BDD+ BaBSR}}
			&883.55 &     82699.40 &    22.00 &   568.25 &     43751.88 &    13.00 &   281.47 &     10763.48 &     5.00 \\
			
			\multicolumn{1}{ c }{\tiny\textsc{Big-M BaBSR}}
			&826.60 &     68582.00 &    19.00 &   533.79 &     35877.24 &    12.00 &   253.37 &      9346.78 &     4.00 \\
			
			\multicolumn{1}{ c }{\tiny\textsc{A. Set 100 it. BaBSR}}
			&422.32 &     \B 9471.90 &  \B   7.00 &  169.73 &      \B 1873.36 &   \B  3.00 &   227.26 &      2302.16 &   2.00 \\
			
			\multicolumn{1}{ c }{\tiny\textsc{Big-M + A. Set 100 it. BaBSR}}
			&  \B  415.20 &     10449.10 &   \B    7.00 &  \B 163.02 &      2402.28 &   \B   3.00 &    199.70 &      2709.60 &    2.00 \\
			
			\multicolumn{1}{ c }{\tiny\textsc{G. Planet + G. 1 cut BaBSR}}
			&   949.06 &      1572.10 &    15.00 &   762.42 &       514.02 &     6.00 &   799.71 &       391.70 &     2.00 \\
			
			\multicolumn{1}{ c }{\tiny\textsc{MIP $\mathcal{A}_k$}}
			&3227.50 &       226.24 &    82.00 &  2500.70 &       100.93 &    64.00 &  3339.37 &       434.57 &    91.00 \\
			
			\multicolumn{1}{ c }{\tiny\textsc{ERAN}}			
			& 805.89 &   {-}  &     5.00 &   632.12 &   {-}   &     9.00 &   545.72 &  {-}  &    0.00 \\

			\multicolumn{1}{ c }{\tiny\textsc{Fast-and-Complete}}			
			&   711.63 &      9801.22 &    16.00 &   350.57 &      8699.74 &     8.00 &  \B  56.52 &      \B 2238.52 &    \B 1.00 \\

			\bottomrule
			
		\end{tabular}
	\end{adjustbox}
	\caption{\small We compare average solving time, average number of solved sub-problems and the percentage of timed out properties on data from \citet{Lu2020Neural}. The best dual iterative method is highlighted in bold.}
	\label{table:models}
\end{table*}

\begin{figure*}[b!]
	\centering
	\begin{subfigure}{.32\textwidth}
		\centering
		\includegraphics[width=\textwidth]{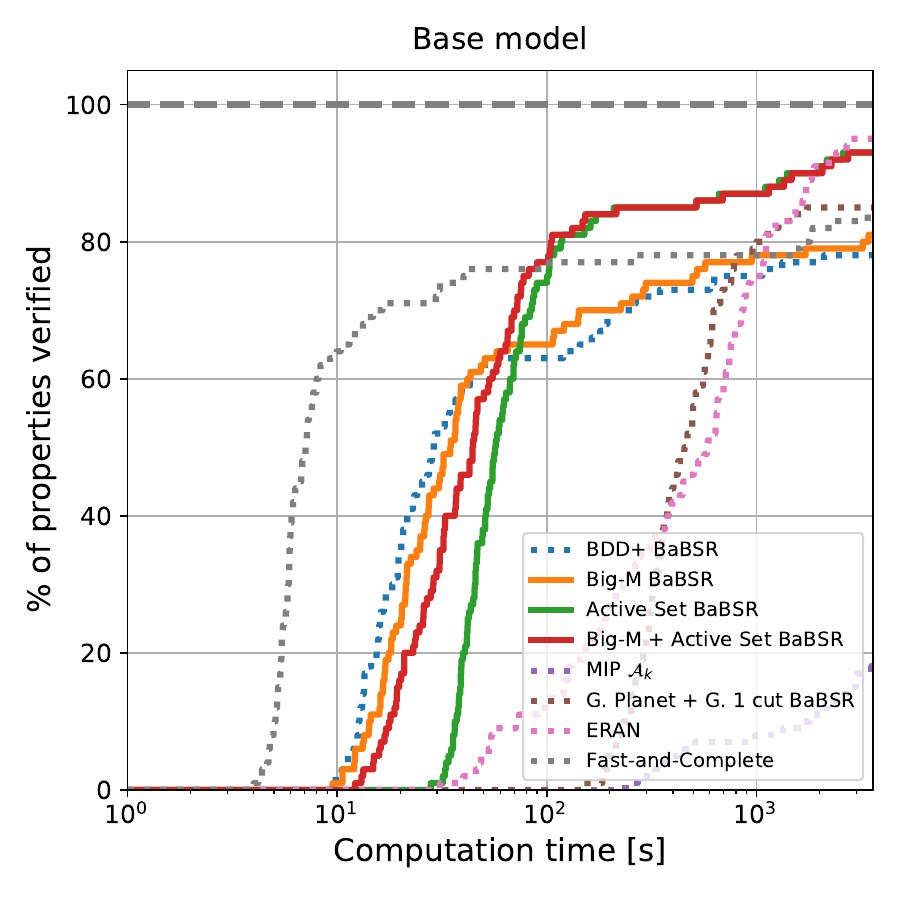}
	\end{subfigure}
	\begin{subfigure}{.32\textwidth}
		\centering
		\includegraphics[width=\textwidth]{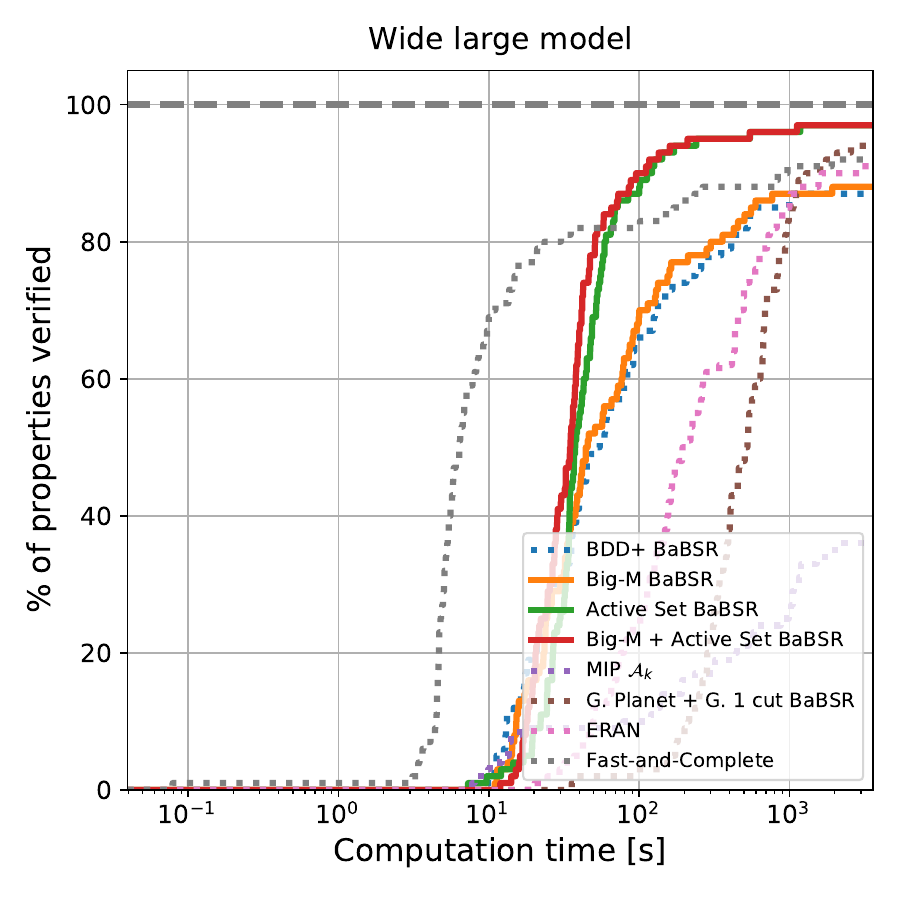}
	\end{subfigure}
	\begin{subfigure}{.32\textwidth}
		\centering
		\includegraphics[width=\textwidth]{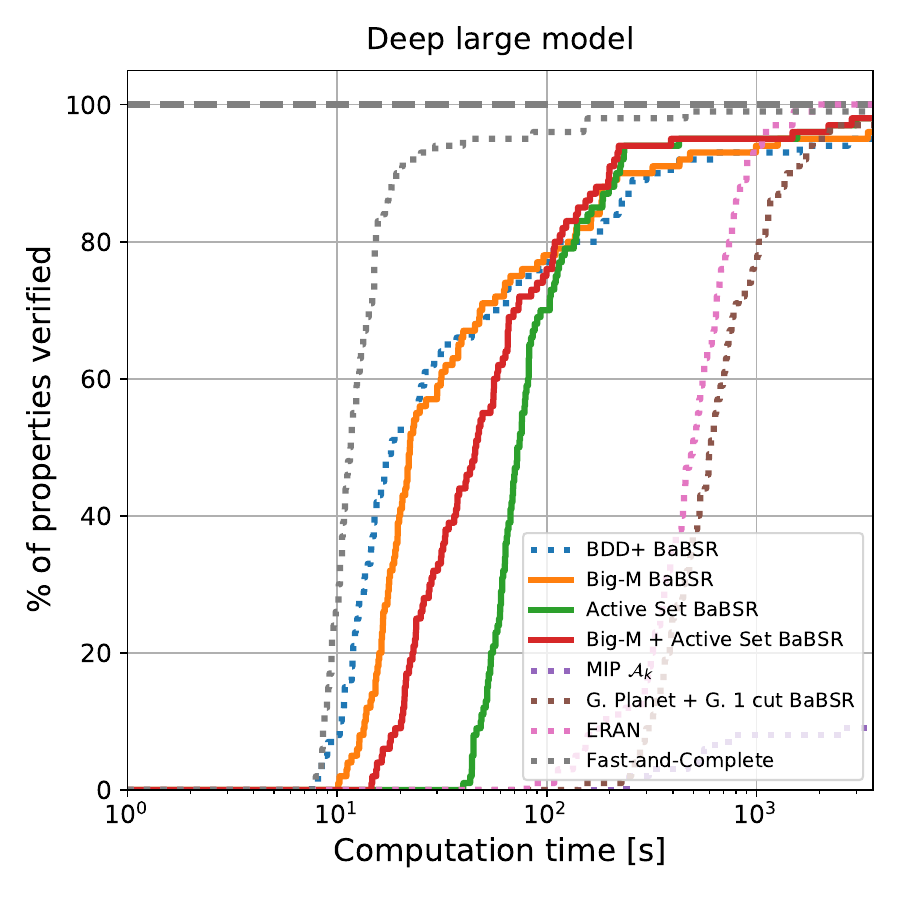}
	\end{subfigure}
	\vspace{-10pt}
	\caption{Cactus plots on properties from \citet{Lu2020Neural}, displaying the percentage of solved properties as a function of runtime. Baselines are represented by dotted lines.}
	\label{fig:verification-main}
\end{figure*} 

We use $100$ iterations for BDD+ (as done by \citet{BunelDP20}) and $180$ for Big-M, which was tuned to employ roughly the same time per bounding computation as BDD+.
We re-employed the same hyper-parameters for Big-M, Active Set and Saddle Point, except the number of iterations.
For dual iterative algorithms, we solve 300 subproblems at once for the base network and 200 for the deep and wide networks (see $\S$\ref{sec:parallel-masked}). Additionally, dual variables are initialized from their parent node's bounding computation.
As in \citet{BunelDP20}, the time-limit is kept at one hour. \\

Figure \ref{fig:sgd-planet-improvement} in the previous section shows that Active Set can yield a relatively large improvement over Gurobi Planet bounds, the convex barrier as defined by \citet{Salman2019}, in a few iterations. In light of this, we start our experimental evaluation by comparing the performance of $100$ iterations of Active Set (within BaBSR) with the relevant baselines.
Figure \ref{fig:verification-main} and Table \ref{table:models} show that Big-M performs competitively with BDD+. 
With respect to BDD+ and Big-M, which operate on the looser formulation \eqref{eq:primal-bigm}, Active Set verifies a larger share of properties and yields faster average verification. This demonstrates the benefit of tighter bounds ($\S$\ref{sec:incomp_verif}) in complete verification. On the other hand, the poor performance of MIP + $\mathcal{A}_k$ and of Gurobi Planet + Gurobi 1 cut, tied to scaling limitations of off-the-shelf solvers, shows that tighter bounds are effective only if they can be computed efficiently. Nevertheless, the difference in performance between the two Gurobi-based methods confirms that customized Branch and Bound solvers (BaBSR) are preferable to generic MIP solvers, as observed by \citet{Bunel2020} on the looser Planet relaxation.
Moreover, the stratified bounding system allows us to retain some of the speed of Big-M on easier properties, without sacrificing Active Set's gains on the harder ones.
While ERAN verifies $2\%$ more properties than Active Set on two networks, BaBSR (with any dual bounding algorithm) is faster on most of the properties. 
Fast-and-Complete performs particularly well on the easier properties and on the Deep model, where it is the fastest algorithm.
Nevertheless, it struggles on the harder properties, leading to larger average verification times and more timeouts than Active Set on the Base and Wide models. 
We believe that stratifying Active Set on the inexpensive dual bounds from Fast-and-Complete, which fall short of the convex barrier yet jointly tighten the intermediate bounds, would preserve the advantages of both methods.
BaBSR-based results could be further improved by employing the learned branching strategy presented by \citet{Lu2020Neural}: in this work, we focused on the bounding component of branch~and~bound. \\

\begin{figure*}[t!]
	\centering
	\begin{subfigure}{.32\textwidth}
		\centering
		\includegraphics[width=\textwidth]{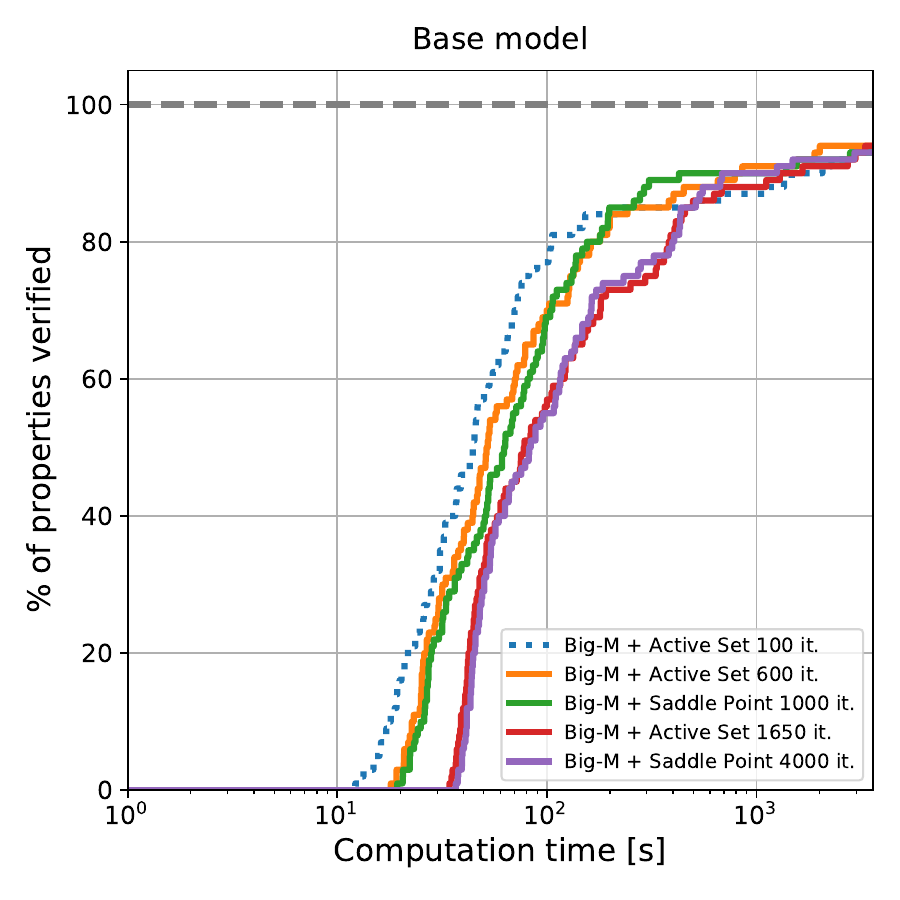}
	\end{subfigure}
	\begin{subfigure}{.32\textwidth}
		\centering
		\includegraphics[width=\textwidth]{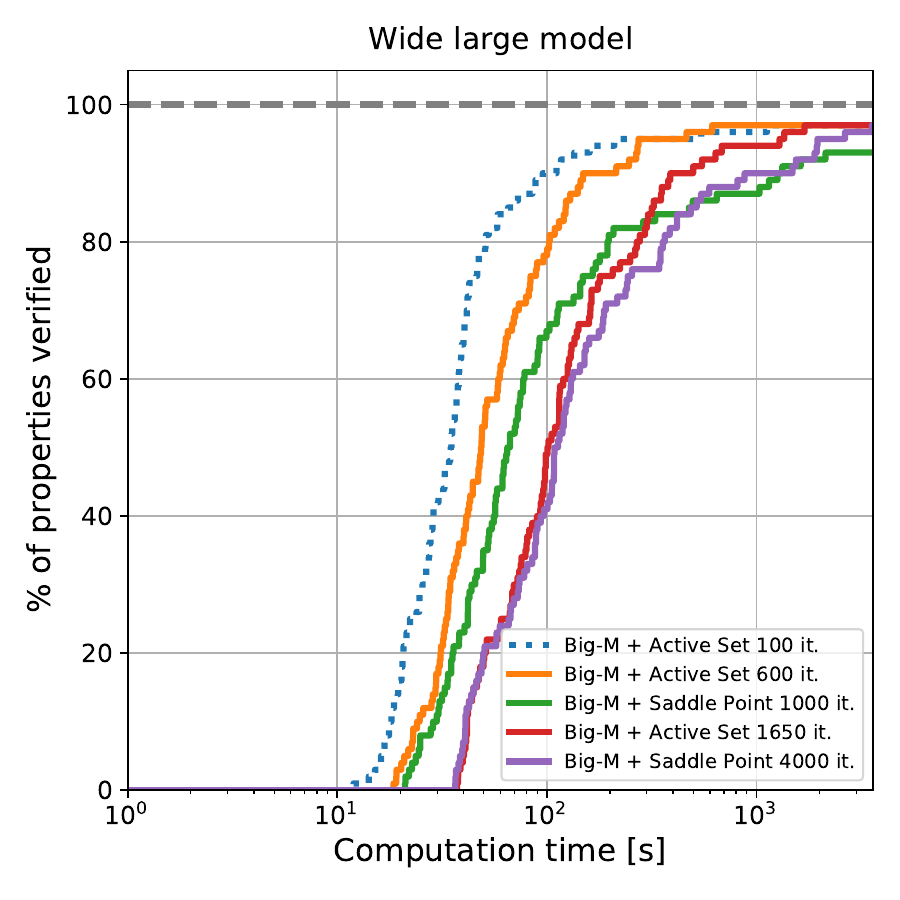}
	\end{subfigure}
	\begin{subfigure}{.32\textwidth}
		\centering
		\includegraphics[width=\textwidth]{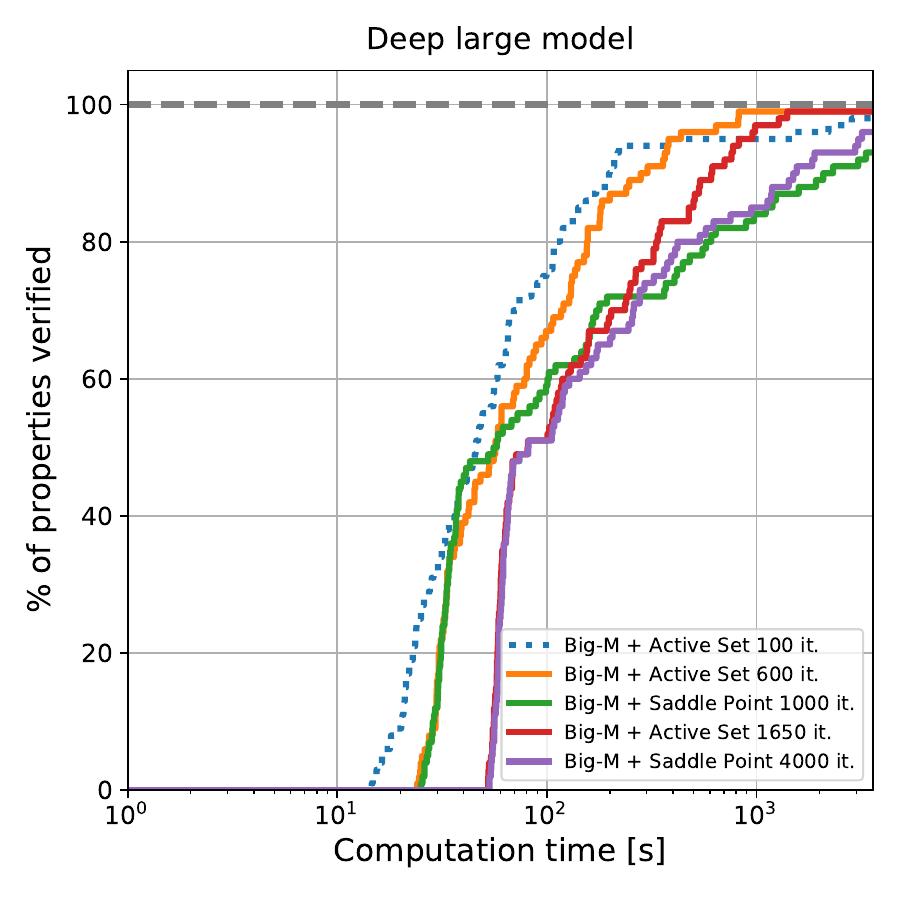}
	\end{subfigure}
	\vspace{-10pt}
	\caption{Cactus plots on properties from \citet{Lu2020Neural}, displaying the percentage of solved properties as a function of runtime. Comparison of best performing method from Figure \ref{fig:verification-main} with tighter bounding schemes.}
	\label{fig:verification-spfw}
\end{figure*} 
\sisetup{detect-weight=true,detect-inline-weight=math,detect-mode=true}
\begin{table*}[b!]
	\centering
	\scriptsize
	\setlength{\tabcolsep}{4pt}
	\aboverulesep = 0.1mm  
	\belowrulesep = 0.2mm  
	\renewcommand{\bfseries}{\fontseries{b}\selectfont}
	\newrobustcmd{\B}{\bfseries}
	\newcommand{\boldentry}[2]{
		\multicolumn{1}{S[table-format=#1,
			mode=text,
			text-rm=\fontseries{b}\selectfont
			]}{#2}}
	\begin{adjustbox}{max width=\textwidth, center}
		\begin{tabular}{
				l
				S[table-format=4.3]
				S[table-format=4.3]
				S[table-format=4.3]
				S[table-format=4.3]
				S[table-format=4.3]
				S[table-format=4.3]
				S[table-format=4.3]
				S[table-format=4.3]
				S[table-format=4.3]
			}
			& \multicolumn{3}{ c }{Base} & \multicolumn{3}{ c }{Wide} & \multicolumn{3}{ c }{Deep} \\
			\toprule
			
			\multicolumn{1}{ c }{Method} &
			\multicolumn{1}{ c }{time(s)} &
			\multicolumn{1}{ c }{sub-problems} &
			\multicolumn{1}{ c }{$\%$Timeout} &
			\multicolumn{1}{ c }{time(s)} &
			\multicolumn{1}{ c }{sub-problems} &
			\multicolumn{1}{ c }{$\%$Timeout} &
			\multicolumn{1}{ c }{time(s)} &
			\multicolumn{1}{ c }{sub-problems} &
			\multicolumn{1}{ c }{$\%$Timeout} \\
			
			\cmidrule(lr){1-1} \cmidrule(lr){2-4} \cmidrule(lr){5-7} \cmidrule(lr){8-10}
			
			\multicolumn{1}{ c }{\tiny \textsc{Big-M + Active Set 100 it.}}
			&   415.20 &     10449.10 &     7.00 & \B  163.02 &      2402.28 & \B     3.00 &   199.70 &      2709.60 &     2.00 \\
			
			\multicolumn{1}{ c }{\tiny\textsc{Big-M + Active Set 600 it.}}
			& \B   360.16 &     \B  4806.14 & \B    6.00 &   181.27 &   \B     1403.90 &   \B    3.00 &  \B  148.06 &   \B  1061.90 &   \B     1.00 \\
			
			\multicolumn{1}{ c }{\tiny\textsc{Big-M + Saddle Point 1000 it.}}
			&   382.15 &      5673.04 &     7.00 &   417.41 &      1900.90 &     7.00 &   540.68 &      2551.62 &     7.00 \\
			
			\multicolumn{1}{ c }{\tiny\textsc{Big-M + Active Set 1650 it.}}
			&   463.00 &      7484.54 &  \B   6.00 &   285.99 &    1634.98 &  \B   3.00 &   250.52 &     1119.00 &   \B  1.00 \\
			
			\multicolumn{1}{ c }{\tiny\textsc{Big-M + Saddle Point 4000 it.}}
			&   434.92 &      4859.68 &     7.00 &   402.86 &      1703.48 &  \B    3.00 &   482.93 &      1444.20 &     4.00 \\
			
			\bottomrule
			
		\end{tabular}
	\end{adjustbox}
	\caption{\small We compare average solving time, average number of solved sub-problems and the percentage of timed out properties on data from \citet{Lu2020Neural}. The best result is highlighted in bold. Comparison of best performing method from Figure \ref{fig:verification-main} with tighter bounding schemes within~BaBSR.} 
	\label{table:models-tighter}
\end{table*}

The results in Figure~\ref{fig:verification-main} demonstrate that a small yet inexpensive tightening of the Planet bounds yields large complete verification improvements. 
We now investigate the effect of employing even tighter, yet more costly, bounds from our solvers.
To this end, we compare $100$ iterations of Active Set with more expensive bounding schemes from our incomplete verification experiments ($\S$\ref{sec:incomp_verif}). All methods were stratified along Big-M to improve performance on easier properties (see $\S$\ref{sec:stratified}) and employed within BaBSR.
Figure \ref{fig:verification-spfw} and Table \ref{table:models-tighter} show results for two different bounding budgets: $600$ iterations of Active Set or $1000$ of Saddle Point, $1650$ iterations of Active Set or $4000$ of Saddle Point (see Figure \ref{fig:sgd-as-vs-spfw}).
Despite their ability to prune more subproblems, due to their large computational cost, tighter bounds do not necessarily correspond to shorter overall verification times. Running Active Set for $600$ iterations of Active Set leads to faster verification of the harder properties while slowing it down for the easier ones. On the base and deep model, the benefits lead to a smaller average runtime (Table \ref{table:models-tighter}). This does not happen on the wide network, for which not enough subproblems are pruned.
On the other hand, $1650$ Active Set iterations do not prune enough subproblems for their computational overhead, leading to slower formal verification. 
The behavior of Saddle Point mirrors what was seen for incomplete verification: while it does not perform as well as Active Set for small computational budgets, the gap shrinks when the algorithms are run for more iterations and it is very competitive with Active Set on the base network. Keeping in mind that the memory cost of each variable addition to the active set is larger in complete verification due to subproblem batching (see \ref{sec:parallel-masked}), Saddle Point constitutes a valid complete verification alternative in settings where memory is critical. 

\sisetup{detect-weight=true,detect-inline-weight=math,detect-mode=true}
\begin{table*}[b!]
	\centering
	\scriptsize
	\setlength{\tabcolsep}{4pt}
	\aboverulesep = 0.1mm  
	\belowrulesep = 0.2mm  
	\renewcommand{\bfseries}{\fontseries{b}\selectfont}
	\newrobustcmd{\B}{\bfseries}
	\newcommand{\boldentry}[2]{
		\multicolumn{1}{S[table-format=#1,
			mode=text,
			text-rm=\fontseries{b}\selectfont
			]}{#2}}
	\begin{adjustbox}{max width=\textwidth, center}
		\begin{tabular}{
				l
				c
				c
				c
				c
				c
				c
				c
			}
			\toprule
			\\[-8pt]
			& \textsc{kPoly} & \textsc{FastC2V} & \textsc{OptC2V} & \textsc{BDD+ BaBSR} & \textsc{Active Set BaBSR} & \textsc{Saddle Point BaBSR} \\
			\midrule
			\\[-7pt]
			Verified Properties & 399 & 390 & 398 & 401 & 434 & 408 \\
			
			Average Runtime [s] & 86 & 15.3 & 104.8 & 2.5 & 7.0 & 43.16 \\
			
			\bottomrule
		\end{tabular}
	\end{adjustbox}
	\caption{\small Number of verified properties and average runtime on the adversarially trained~\citep{Madry2018} \texttt{ConvSmall} network from the ERAN~\citep{eran} data set, on the first $1000$ images of the CIFAR-10 test set. Results for FastC2V and OptC2V are taken from \citet{Tjandraatmadja}, results for kPoly are taken from \citet{Singh2019b}.} 
	\label{table:eran-convs-incomplete}
\end{table*}

\subsubsection{Branch and Bound for Incomplete Verification} \label{sec:bab-incomplete}

When run for a fixed number of iterations, branch and bound frameworks can be effectively employed for incomplete verification. 
As shown in $\S$\ref{sec:exp-comp-sub}, one of the main advantages of our algorithms is their integration and performance within branch and bound.
We provide further evidence of this by comparing them with previous algorithms that overcame the convex barrier, but which were not designed for employment within branch and bound: \textbf{kPoly} by~\citet{Singh2019b}, \textbf{Fast2CV} and \textbf{Opt2CV} by~\citet{Tjandraatmadja}.
In particular, we compute the number of verified images, against perturbations of radius $\epsilon_{ver}=2/255$, on the first $1000$ examples of the CIFAR-10 test set for the adversarially-trained~\citep{Madry2018} \texttt{ConvSmall} network from the ERAN~\citep{eran} data set, excluding misclassified images. 
We test different speed-accuracy trade-offs within branch and bound: BDD+ is run for $400$ iterations per branch and bound sub-problem, and at most $5$ branch and bound batches. Active Set and Saddle Point are run for $200$ and $4000$ iterations, respectively, with at most $10$ batches within BaBSR. 
Hyper-parameters are kept as in $\S$\ref{sec:incomp_verif}. This experiment is run on a single Nvidia Titan V GPU.
Table~\ref{table:eran-convs-incomplete} shows that, regardless of the employed speed-accuracy trade-off, the use of branch and bound is beneficial with respect to kPoly, Fast2CV, and Opt2CV. Therefore, the use of relaxations tighter than Planet does not necessarily improve incomplete verification performance.
Increasing the computational budget of branch and bound results in a larger number of verified properties, as demonstrated by the performance of Active Set and Saddle Point.

\section{Discussion}

The vast majority of neural network bounding algorithms focuses on (solving or loosening) a popular triangle-shaped relaxation, referred to as the ``convex barrier" for verification.
Relaxations that are tighter than this convex barrier have been recently introduced, but the complexity of the standard solvers for such relaxations hinders their applicability.
We have presented two different sparse dual solvers for one such relaxation, and empirically demonstrated that they can yield significant formal verification speed-ups.
Our results show that tightness, when paired with scalability, is key to the efficiency of neural network verification and instrumental in the definition of a more appropriate ``convex barrier".
We believe that new customized solvers for similarly tight relaxations are a crucial avenue for future research in the area, possibly beyond piecewise-linear networks.
Finally, as it is inevitable that tighter bounds will come at a larger computational cost, future verification systems will be required to recognise a priori whether tight bounds are needed for a given property.
A possible solution to this problem could rely on learning algorithms.

\acks{ADP was supported by the EPSRC Centre for Doctoral Training in Autonomous Intelligent Machines and Systems, grant EP/L015987/1, and an IBM PhD fellowship. HSB was supported using a Tencent studentship through the University of Oxford.}

\newpage

\renewcommand{\theHsection}{A\arabic{section}}
\appendix

\section{Limitations of Previous Dual Approaches} \label{sec:no-previous}

In this section, we show that previous dual derivations~\citep{BunelDP20,Dvijotham2018} violate Fact \ref{fact}. Therefore, they are not efficiently applicable to problem \eqref{eq:primal-anderson}, motivating our own derivation in Section \ref{sec:solver}.

We start from the approach by~\citet{Dvijotham2018}, which relies on relaxing equality constraints \eqref{eq:verif-linconstr}, \eqref{eq:verif-relu} from the original non-convex problem \eqref{eq:noncvx_pb}. 
\citet{Dvijotham2018} prove that this relaxation corresponds to solving convex problem \eqref{eq:primal-bigm}, which is equivalent to the Planet relaxation~\citep{Ehlers2017}, to which the original proof refers. 
As we would like to solve tighter problem \eqref{eq:primal-anderson}, the derivation is not directly applicable. 
Relying on intuition from convex analysis applied to duality gaps~\citep{lemarechal2001lagrangian}, we conjecture that relaxing the composition \eqref{eq:verif-relu}   $\circ$ \eqref{eq:verif-linconstr} might tighten the primal problem equivalent to the relaxation, obtaining the following dual:
\begin{equation}
\begin{split}
\max_{\mub} \min_{\xb} &\quad W_n\xb_{n-1} + \mbf{b}_n 
+  \sum_{k=1}^{n-1} \mub_k^T \left( \xb_k - \max \left\{ W_k \xb_{k-1} + \mbf{b}_k, 0\right\} \right) \\
\text{s.t. }\quad& \lb_k \leq \xb_k \leq \ub_k \qquad\qquad  k \in \left\llbracket1,n-1\right\rrbracket,\\
& \xb_0 \in \mathcal{C}.
\end{split}
\label{eq:dj_adapted}
\end{equation}
Unfortunately dual \eqref{eq:dj_adapted} requires an LP (the inner minimisation over $\xb$, which in this case does not decompose over layers) to be solved exactly to obtain a supergradient and any time a valid bound is needed. This is markedly different from the original dual by ~\citet{Dvijotham2018}, which had an efficient closed-form for the inner problems.

The derivation by~\citet{BunelDP20}, instead, operates by substituting \eqref{eq:verif-relu} with its convex hull and solving its Lagrangian Decomposition dual. The Decomposition dual for the convex hull of \eqref{eq:verif-relu}  $\circ$ \eqref{eq:verif-linconstr} (i.e., $\mathcal{A}_k$) takes the following form:
\begin{equation}
\begin{split}
	\max_{\lambdabhat}\min_{\xb, \zb}\ & W_n \xb_{A, n-1} + \bb_n
	+ \sum_{k=1}^{n-1} \lambdabhat_k^T \left( \xb_{B, k} - \xb_{A, k} \right)\\[5pt]
	\text{s.t. }\quad & \xb_0 \in \mathcal{C}, \\
	&(\xb_{B, k}, W_n \xb_{A, n-1} + \bb_n, \zb_{k}) \in \mathcal{A}_{\text{dec}, k}  \qquad k \in \left\llbracket1, n-1\right\rrbracket,
\end{split}
\label{eq:uai_adapted}
\end{equation}
where $\mathcal{A}_{\text{dec}, k}$ corresponds to $\mathcal{A}_{k}$ with the following substitutions: $\xb_{k} \rightarrow \xb_{B, k}$, and $\xbhat_{k} \rightarrow W_n \xb_{A, n-1} + \bb_n$.
It can be easily seen that the inner problems  (the inner minimisation over $\xb_{A, k}, \xb_{B, k}$, for each layer $k>0$) are an exponentially sized LP.
Again, this differs from the original dual on the Planet relaxation~\citep{BunelDP20}, which had an efficient closed-form for the inner problems.

\vspace{10pt}
\section{Dual Initialisation} \label{sec:dual-init}

\begin{minipage}{.99\textwidth}
	\begin{algorithm}[H]
		\caption{Big-M solver}\label{alg:bigm}
		\begin{algorithmic}[1]
			\algrenewcommand\algorithmicindent{1.0em}
			\Function{bigm\_compute\_bounds}{$\{W_k, \mbf{b}_k, \lb_k, \ub_k,\hat{\lb}_k, \hat{\ub}_k\}_{k=1..n}$}
			\State Initialize duals $\balpha^0, \bbeta^0_{\mathcal{M}}$ using interval propagation bounds~\citep{Gowal2018}
			\For{$t \in \llbracket1, T-1\rrbracket$}\label{alg:inner_loopstart}
			\State $\xb^{*, t}, \zb^{*, t} \in$ $\argmin_{\xb, \zb}\mathcal{L}_{\mathcal{M}}(\xb, \zb, \balpha^t, \bbeta^t_{\mathcal{M}}) $ using \eqref{eq:bigm-primalk-min}-\eqref{eq:bigm-primal0-min}\label{alg:inner_min}
			\State $\balpha^{t+1}, \bbeta^{t+1}_{\mathcal{M}}\leftarrow [\balpha^{t}, \bbeta^{t}] + H[\nabla_{\balpha}d_{\mathcal{M}}(\balpha, \bbeta), \nabla_{\bbeta_{\mathcal{M}}}d_{\mathcal{M}}(\balpha, \bbeta)]$ \Comment{supergradient step, using \eqref{eq:bigm-supergradient}}
			\State $\balpha^{t+1}, \bbeta^{t+1}_{\mathcal{M}}\leftarrow \max(\balpha^{t+1}, 0), \max(\bbeta^{t+1}_{\mathcal{M}}, 0) \qquad$ \Comment{dual projection}
			\EndFor\label{alg:inner_loopend}
			\State\Return $\min_{\xb, \zb}\mathcal{L}_{\mathcal{M}}(\xb, \zb, \balpha^T, \bbeta^T_{\mathcal{M}}) $
			\EndFunction
		\end{algorithmic}
	\end{algorithm}
\end{minipage}
\vspace{10pt}

As shown in Section \ref{sec:solver}, the Active Set solver reduces to a dual solver for the Big-M relaxation \eqref{eq:primal-bigm} if the active set $\mathcal{B}$ is kept empty throughout execution. We employ this Big-M solver as dual initialization for both Active Set ($\S$\ref{sec:active}) and Saddle Point ($\S$\ref{sec:sp}).
We demonstrate experimentally in $\S$\ref{sec:exp} that, when used as a stand-alone solver, our Big-M solver is competitive with previous dual algorithms for problem \eqref{eq:primal-bigm}.

The goal of this section is to explicitly describe the Big-M solver, which is summarised in algorithm \ref{alg:bigm}.
We point out that, in the notation of restricted variable sets from Section \ref{sec:active}, $\bbeta_{\mathcal{M}} := \bbeta_{\emptyset}$.
We now describe the equivalence between the Big-M and Planet relaxations,
before presenting the solver in Section \ref{sec:bigm-solver} and the dual it operates on in Section \ref{sec:bigm-dual}.

\subsection{Equivalence to Planet} \label{sec:planet-eq}
As previously shown \citep{Bunel2018}, the Big-M relaxation ($\mathcal{M}_k$, when considering the $k$-th layer only) in problem \eqref{eq:primal-bigm} is equivalent to the Planet relaxation by \citet{Ehlers2017}. 
Then, due to strong duality, our Big-M solver (Section \ref{sec:bigm-dual}) and the solvers by~\citet{BunelDP20,Dvijotham2018} will all converge to the bounds from the solution of problem \eqref{eq:primal-bigm}. In fact, the Decomposition-based method \citep{BunelDP20} directly operates on the Planet relaxation, while \citet{Dvijotham2018} prove that their dual is equivalent to doing so. 

\begin{figure}[b!]
	\centering
	\includegraphics[width=0.45\textwidth]{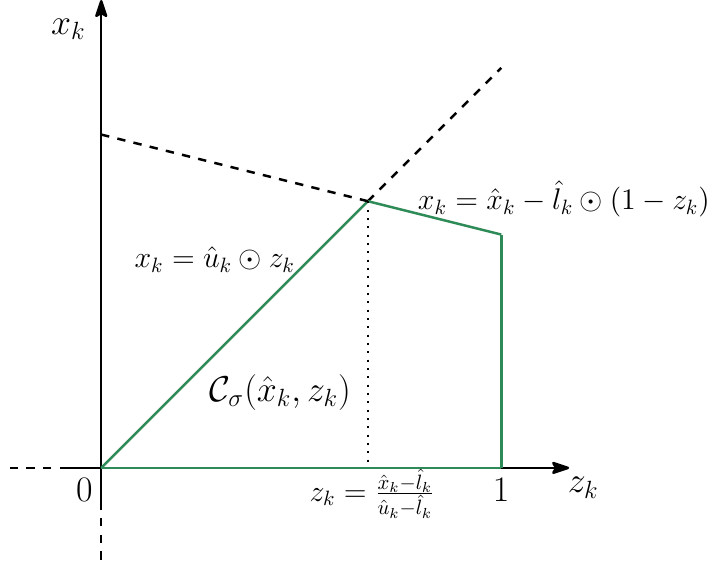}
	\caption{$\mathcal{M}_k$ plotted on the $(\zb_k, \xb_k)$ plane, under the assumption that $\hat{\lb}_k \leq  0 \text{ and } \hat{\ub}_k \geq 0$. }
	\label{fig:bigm-projection}
\end{figure}

On the $k$-th layer, the Planet relaxation takes the following form: 
\begin{equation}
\begin{split}
\label{eq:planet}
\mathcal{P}_k : =
\begin{cases}
\text{if } \hat{\lb}_k \leq  0 \text{ and } \hat{\ub}_k \geq 0:& \\
&\xb_{k} \geq 0, \quad \xb_k \geq \xbhat_k,  \\
&\xb_{k} \leq \mbf{\hat{u}_k} \odot (\xbhat_k - \mbf{\hat{l}}_k) \odot \left( 1/(\mbf{\hat{u}_k}-\mbf{\hat{l}_k})\right).\\
\text{if } \hat{\ub}_k \leq 0:& \\
&\xb_k = 0. \\
\text{if } \hat{\lb}_k \geq 0:& \\
&\xb_k = \xbhat_k.  
\end{cases}
\end{split}
\end{equation}

It can be seen that $\mathcal{P}_k = \text{Proj}_{\xb, \xbhat} \left(\mathcal{M}_k\right)$, where $\text{Proj}_{\xb, \xbhat} $ denotes projection on the $\xb, \xbhat$ hyperplane.
In fact, as $\zb_{k}$ does not appear in the objective of the primal formulation \eqref{eq:primal-bigm}, but only in the constraints, this means assigning it the value that allows the largest possible feasible region.
This is trivial for passing or blocking ReLUs. For the ambiguous case, instead, Figure \ref{fig:bigm-projection} (on a single ReLU) shows that $z_{k} = \frac{\hat{x}_k - \hat{l}_k}{\hat{u}_k-\hat{l}_k}$ is the correct assignment.

\subsection{Big-M Dual} \label{sec:bigm-dual}

As evident from problem \eqref{eq:primal-anderson}, $\mathcal{A}_k \subseteq \mathcal{M}_k$. 
If we relax all constraints in $\mathcal{M}_k$ (except, again, the box constraints), we are going to obtain a dual with a strict subset of the variables in problem \eqref{eq:dual-anderson}.
The Big-M dual is a specific instance of the Active Set dual \eqref{eq:dual-active-set} where $\mathcal{B}=\emptyset$, and it takes the following form:
\begin{equation} 
\begin{aligned}
\max_{(\balpha, \bbeta)\geq 0} &d_{\mathcal{M}}(\balpha, \bbeta_{\mathcal{M}}) \qquad  \text{where:} \qquad d_{\mathcal{M}}(\balpha, \bbeta_{\mathcal{M}}) :=\min_{\xb, \zb} \enskip \mathcal{L}_{\mathcal{M}}(\xb, \zb, \balpha, \bbeta_{\mathcal{M}}),  \\
\mathcal{L}_{\mathcal{M}}(\xb, \zb, \balpha, \bbeta_{\mathcal{M}})  = &
\left[\begin{array}{l}
- \sum_{k=0}^{n-1} \left(\balpha_{k} - W_{k+1}^T\balpha_{k+1} - (\bbeta_{k, 0} + \bbeta_{k, 1} -  W_{k+1}^T \bbeta_{k+1, 1} )\right)^T \xb_k\\
+ \sum_{k=1}^{n-1}\bb_k^T \balpha_{k} - \sum_{k=1}^{n-1}\left(\bbeta_{k, 0} \odot \hat{\ub}_k + \bbeta_{k, 1} \odot \hat{\lb}_k\right)^T \zb_k \\
+ \sum_{k=1}^{n-1} (\hat{\lb}_k - \bb_{k})^T \bbeta_{k, 1}\end{array}\right.\\[5pt]
\text{s.t. } \qquad & \xb_0 \in \mathcal{C}, \qquad  (\xb_{k}, \zb_{k}) \in [\lb_k, \ub_k] \times [\mbf{0},\ \mbf{1}] \qquad k \in \llbracket1, n-1\rrbracket.
\end{aligned}
\label{eq:dual-bigm}
\end{equation}

\subsection{Big-M Solver} \label{sec:bigm-solver}

We initialise dual variables from interval propagation bounds~\citep{Gowal2018}: this can be easily done by setting all dual variables except $\balpha_{n}$ to $0$.
Then, we can maximize $d_{\mathcal{M}}(\balpha, \bbeta)$ via projected supergradient ascent, exactly as described in Section \ref{sec:active} on a generic active set $\mathcal{B}$. All the computations in the solver follow from keeping $\mathcal{B}=\emptyset$ in $\S$\ref{sec:active}. We explicitly report them here for the reader's convenience.

Let us define the following shorthand for the primal coefficients:
\begin{equation*}
	\begin{gathered}
		\boldsymbol{f}_{\mathcal{M}, k}(\balpha, \bbeta_{\mathcal{M}}) = \left(\balpha_{k} - W_{k+1}^T\balpha_{k+1} - (\bbeta_{k, 0} + \bbeta_{k, 1} -  W_{k+1}^T \bbeta_{k, 1} )\right) \\
		\boldsymbol{g}_{\mathcal{M}, k}(\bbeta_{\mathcal{M}}) = \bbeta_{k, 0} \odot \hat{\ub}_k + \bbeta_{k, 1} \odot \hat{\lb}_k.
	\end{gathered}
\end{equation*}
The minimisation of the Lagrangian $\mathcal{L}_{\mathcal{M}}(\xb, \zb, \balpha, \bbeta)$ over the primals for $k \in \llbracket1, n-1\rrbracket$ is as follows:
\begin{equation}
\begin{gathered}
\label{eq:bigm-primalk-min}
\xb_{k}^* = \mathds{1}_{\boldsymbol{f}_{\mathcal{M}, k}(\balpha, \bbeta_{\mathcal{M}})  \geq 0} \odot \hat{\ub}_k + \mathds{1}_{\boldsymbol{f}_{\mathcal{M}, k}(\balpha, \bbeta_{\mathcal{M}})  < 0} \odot \hat{\lb}_k \qquad
\zb_{k}^* = \mathds{1}_{\boldsymbol{g}_{\mathcal{M}, k}(\bbeta_{\mathcal{M}})   \geq 0} \odot \mbf{1} \\[8pt]
\end{gathered}
\end{equation}

For $k = 0$, instead (assuming, as $\S$\ref{sec:active} that this can be done efficiently):
\begin{equation}
\xb^*_0 \in \argmin_{\xb_0} \quad   \boldsymbol{f}_{\mathcal{M}, k}(\balpha, \bbeta_{\mathcal{M}})^T\xb_{0} \qquad \quad \text{s.t. } \quad  \xb_0 \in \mathcal{C}.
\label{eq:bigm-primal0-min}
\end{equation}

The supergradient over the Big-M dual variables $\balpha, \bbeta_{k, 0}, \bbeta_{k, 1}$ is computed exactly as in $\S$\ref{sec:active} and is again a subset of the supergradient of the full dual problem \eqref{eq:dual-anderson}.
We report it for completeness.
For each $k \in \llbracket0, n-1\rrbracket$:
\begin{equation}
\begin{gathered}
\label{eq:bigm-supergradient}
\nabla_{\balpha_k}d(\balpha, \bbeta) = W_k \xb^*_{k-1} + \bb_k - \xb^*_k, \quad
\nabla_{\bbeta_{k, 0}}d(\balpha, \bbeta)  = \xb_k - \zb_k \odot \hat{\ub}_{k},\\
\nabla_{\bbeta_{k, 1}}d(\balpha, \bbeta) = \xb_k - \left(W_k\xb_{k-1} + \bb_{k}\right) + \left(1 -\zb_k \right)\odot \hat{\lb}_{k}.
\end{gathered}
\end{equation}

\section{Implementation Details for Saddle Point} 

In this section, we present details of the Saddle Point solver that were omitted from the main paper. 
We start with the choice of price caps $\mub_{k}$ on the Lagrangian multipliers ($\S$\ref{sec:sp-fw-caps}), and conclude with a description of our primal initialisation procedure~($\S$\ref{sec:sp-fw-primalinit}).

\subsection{Constraint Price Caps} \label{sec:sp-fw-caps}

Price caps $\mub$ containing the optimal dual solution to problem \eqref{eq:dual-anderson} can be found via binary search (see $\S$\ref{sec:sp-fw-theory}). However, such a strategy might require running Saddle Point to convergence on several instances of problem \eqref{eq:sp-fw-dual}. Therefore, in practice, we employ a heuristic to set constraint caps to a reasonable approximation.

Given duals $(\balpha^0, \bbeta^0)$ from the dual initialization procedure (algorithm \ref{alg:active-set} or algorithm \ref{alg:bigm}, depending on the computational budget), for each $k \in \left\llbracket1, n-1\right\rrbracket$ we set $\mub_k$ as follows:
\begin{equation}
	\begin{aligned}
		\mub_{\alpha, k} &= \begin{cases}
		\balpha_k^0 & \text{if } \balpha_k^0 > 0 \\
		c_{\alpha, k} & \text{otherwise} 
		\end{cases} \\
		\mub_{\beta, k} &= \begin{cases}
		\begin{array}{l} \sum_{I_{k}  \in 2^{W_k}} \bbeta_k^0 \end{array} & \text{if } \begin{array}{l} \sum_{I_{k}  \in 2^{W_k}} \bbeta_k^0 \end{array} > 0 \\[5pt]
		c_{\beta, k} & \text{otherwise}, 
		\end{cases}
	\end{aligned}
	\label{eq:price-caps}
\end{equation}
where $c_{\alpha, k}$ and $c_{\beta, k}$ are small positive constants.
In other words, we cap the (sums over) dual variables to their values at initialization if these are non-zero, and we allow dual variables to turn positive otherwise. 
While a larger feasible region (for instance, setting $\mub_{\alpha, k} = \max_i \balpha_k [i] \odot \mbf{1}$, $\mub_{\beta, k} = \max_i \sum_{I_{k}  \in 2^{W_k}} \bbeta_k^0[i] \odot \mbf{1}$) might yield tighter bounds at convergence, we found assignment \eqref{eq:price-caps} to be more effective on the iteration budgets employed in $\S$\ref{sec:incomp_verif},$\S$\ref{sec:comp_verif}.

\subsection{Primal Initialization} \label{sec:sp-fw-primalinit}

If we invert the minimization and maximization, our saddle-point problem \eqref{eq:sp-fw-dual} can be seen as a non-smooth minimization problem (the minimax theorem holds):
\begin{equation}
\begin{aligned}
\label{eq:primal-subg-problem}
\min_{\xb, \zb} \max_{\balpha, \bbeta} & \left[\begin{array}{l}
\sum_{k=1}^{n-1} \balpha_{k}^T \left( W_k \xb_{k-1} + \bb_k - \xb_k \right) + \sum_{k=1}^{n-1} \bbeta_{k, 0}^T \left(\xb_k - \zb_k \odot \hat{\ub}_{k}\right)  \\
+ \sum_{k=1}^{n-1} \sum_{I_{k}  \in\mathcal{E}_k} \bbeta_{k, I_{k}}^T\left(\begin{array}{l}
\left(W_k  \odot I_{k} \odot \lbrv_{k-1}\right) \diamond (1 - \zb_k) \\
-  \bb_k \odot \zb_k - \left(W_k \odot  \left(1 - I_{k}\right) \odot \ubrv_{k-1}\right) \diamond \zb_k  \\
- \left( W_k \odot I_{k} \right) \xb_{k-1} + \xb_k  \end{array}\right) \\
+\sum_{k=1}^{n-1} \bbeta_{k, 1}^T \left(\xb_k - \left(W_k\xb_{k-1} + \bb_{k}\right) + \left(1 -\zb_k \right)\odot \hat{\lb}_{k}\right) + W_n \xb_{n-1} + \bb_n \end{array}\right. \qquad \\
\text{s.t. } \enskip &\xb_0 \in \mathcal{C}, \\
&(\xb_k, \zb_k) \in [\lb_k, \ub_k] \times [\mbf{0},\ \mbf{1}]  && \hspace{-120pt} k \in \left\llbracket1, n-1\right\rrbracket,  \\
& \hspace{-7pt} \begin{array}{l}
\balpha_{k} \in [\mbf{0}, \mub_{k, \alpha}] \\
\hspace{-2pt}\left( \bbeta_{k} \geq \mbf{0}, \ \sum_{I_{k}\in \mathcal{E}_k \cup \{0,1\}} \bbeta_{k, I_{k}} \leq  \mub_{k, \beta}\right)  \end{array} \quad && \hspace{-120pt} k \in \left\llbracket1, n-1\right\rrbracket.
\end{aligned}
\end{equation}
The ``primal view"\footnote{due to the restricted dual domain, problem \eqref{eq:primal-subg-problem} does not correspond to the original primal problem in~\eqref{eq:primal-anderson}.} of problem \eqref{eq:primal-subg-problem} admits an efficient subgradient method, which we employ as primal initialization for the Saddle Point solver. 

\subsubsection{Technical Details}
The inner maximizers $(\balpha^*, \bbeta^*)$, which are required to obtain a valid subgradient, will be given in closed-form by the dual conditional gradients from equations \eqref{eq:alphak-condg}, \eqref{eq:betak-condg} (the objective is bilinear).
Then, using the dual functions defined in \eqref{eq:dual-anderson-functions}, the subgradient over the linearly-many primal variables can be computed as $\boldsymbol{f}_{k}(\balpha^*, \bbeta^*)$ for $\xb_k$ and $\boldsymbol{g}_{k}(\bbeta^*)$ for $\zb_k$. After each subgradient step, the primals are projected to the feasible space: this is trivial for $\xb_k$ and $\zb_k$, which are box constrained, and can be efficiently performed for $\mathcal{C}$ in the common cases of $\ell_{\infty}$ or $\ell_2$ balls.

\subsubsection{Simplified Initialization}
Due to the additional cost associated to masked forward-backward passes (see appendix \ref{sec:masked-passes}), the primal initialization procedure can be simplified by restricting the dual variables to the ones associated to the Big-M relaxation (appendix \ref{sec:dual-init}). This can be done by substituting $\mathcal{E}_k \leftarrow \emptyset$ and $\bbeta_k \leftarrow \bbeta_{\emptyset, k}$ (see notation in $\S$\ref{sec:as-solver}) in problem \eqref{eq:primal-subg-problem}. A subgradient method for the resulting problem can then be easily adapted from the description above.

\vspace{10pt}
\section{Dual Derivations} \label{sec:dual-derivation}
We now derive problem \eqref{eq:dual-anderson}, the dual of the full relaxation by \citet{Anderson2020} described in equation \eqref{eq:primal-anderson}.
The Active Set (equation \eqref{eq:dual-active-set}) and Big-M duals (equation \eqref{eq:dual-bigm}) can be obtained by removing  $\bbeta_{k, I_{k}} \forall \ I_{k} \in \mathcal{E}_k \setminus \mathcal{B}_k$ and $\bbeta_{k, I_{k}} \forall \ I_{k} \in \mathcal{E}_k$, respectively.
We employ the following Lagrangian multipliers:

\begin{equation*}
\begin{split}
\xb_{k} \geq \xbhat_{k} &\Rightarrow \balpha_k, \\
\xb_{k} \leq \hat{\ub}_k \odot \zb_{k} &\Rightarrow \bbeta_{k, 0}, \\
\xb_{k} \leq \xbhat_{k} - \hat{\lb}_k \odot  (1 - \zb_{k}) & \Rightarrow \bbeta_{k, 1}, \\
\xb_{k} \leq \left( \begin{array}{l} \left(W_k  \odot I_{k}\right)  \xb_{k-1}  + \zb_k \odot \bb_k \\
- \left(W_k  \odot I_{k} \odot \lbrv_{k-1}\right) \diamond (1 - \zb_k) \\
+ \left(W_k \odot  \left(1 - I_{k}\right) \odot \ubrv_{k-1}\right) \diamond \zb_k \end{array} \right) &\Rightarrow \bbeta_{k, I_k},
\end{split}
\end{equation*}

and obtain, as a Lagrangian (using $\xbhat_{k} = W_k \xb_{k-1} + \bb_k$):
\begin{equation*}
	\mathcal{L}(\xb, \zb, \balpha, \bbeta) = \left[\hspace{-2pt}\begin{array}{l}
	 \sum_{k=1}^{n-1} \balpha_{k}^T \left( W_k \xb_{k-1} + \bb_k - \xb_k \right) + \sum_{k=1}^{n-1} \bbeta_{k, 0}^T \left(\xb_k - \zb_k \odot \hat{\ub}_{k}\right)  \\
	+ \sum_{k=1}^{n-1} \sum_{I_{k}  \in\mathcal{E}_k} \bbeta_{k, I_{k}}^T\left( \hspace{-4pt} \begin{array}{l}
	 \left(W_k  \odot I_{k} \odot \lbrv_{k-1}\right) \diamond (1 - \zb_k) - \left( W_k \odot I_{k} \right) \xb_{k-1} \\
	- \bb_k \odot  \zb_k - \left(W_k \odot  \left(1 - I_{k}\right) \odot \ubrv_{k-1}\right) \diamond  \zb_k  + \xb_k  \end{array}\hspace{-4pt}\right) \\
	+\sum_{k=1}^{n-1} \bbeta_{k, 1}^T \left(\xb_k - \left(W_k\xb_{k-1} + \bb_{k}\right) + \left(1 -\zb_k \right)\odot \hat{\lb}_{k}\right) + W_n \xb_{n-1} + \bb_n \end{array}\right.
\end{equation*}

Let us use $\sum_{I_{k}}$ as shorthand for $\sum_{I_{k} \in \mathcal{E}_k \cup \{0, 1\}}$.
If we collect the terms with respect to the primal variables and employ dummy variables $\balpha_{0}= 0$, $\bbeta_0 = 0$, $\balpha_{n} = I$, $\bbeta_{n} = 0$, we obtain:

\begin{equation*}
\mathcal{L}(\xb, \zb, \balpha, \bbeta) = \hspace{-2pt}\left[\hspace{-5pt}\begin{array}{l}
- \sum_{k=0}^{n-1}\left(\begin{array}{l}
\balpha_{k} - W_{k+1}^T\balpha_{k+1} - \sum_{I_{k}}\bbeta_{k, I_{k}} \\+ \smash{\sum_{I_{k+1}} (W_{k+1} \odot I_{k+1})^T \bbeta_{k+1, I_{k+1}}}
\end{array} \right)^T \xb_k \\
- \sum_{k=1}^{n-1}\def\arraystretch{1.5} \left(\begin{array}{l} \sum_{I_{k} \in \mathcal{E}_k} \bbeta_{k, I_{k}} \odot \bb_k + \bbeta_{k, 1} \odot \hat{\lb}_k + \bbeta_{k, 0} \odot \hat{\ub}_k \\
+ \sum_{I_{k}  \in\mathcal{E}_k} \left(W_k \odot I_{k} \odot \lbrv_{k-1}\right) \diamond \bbeta_{k, I_{k}}\\ 
+  \sum_{I_{k}  \in\mathcal{E}_k} \left(W_k \odot (1 - I_{k}) \odot \ubrv_{k-1}\right) \diamond \bbeta_{k, I_{k}}  \end{array}\right)^T \zb_k \\
+ \sum_{k=1}^{n-1}\bb_k^T \balpha_{k} + \sum_{k=1}^{n-1} \left( \sum_{I_{k}  \in\mathcal{E}_k} (W_k \odot I_{k} \odot \lbrv_{k-1}) \oblong \bbeta_{k, I_{k}} + \bbeta_{k, 1}^T (\hat{\lb}_k - \bb_{k})\right)\end{array}\right.
\end{equation*}
which corresponds to the form shown in problem \eqref{eq:dual-anderson}.

\vspace{10pt}
\section{Intermediate Bounds} \label{sec:intermediate}

A crucial quantity in both ReLU relaxations ($\mathcal{M}_k$ and $\mathcal{A}_k$) are intermediate pre-activation bounds $\hat{\lb}_k, \hat{\ub}_k$.
In practice, they are computed by solving a relaxation $\mathcal{C}_k$ (which might be $\mathcal{M}_k$, $\mathcal{A}_k$, or something looser) of \eqref{eq:noncvx_pb} over subsets of the network \citep{BunelDP20}. For $\hat{\lb}_i$, this means solving the following problem (separately, for each entry $\hat{\lb}_i[j]$):

\begin{equation}
\begin{aligned}
\min_{\xb, \xbhat, \zb}  \enskip &\xbhat_{i}[j]  \\
\text{s.t. }\enskip &\xb_0
\in \mathcal{C}\span\\
& \xbhat_{k+1} = W_{k+1} \xb_{k} + \bb_{k+1},  && k \in \left\llbracket0, i-1 \right\rrbracket,\\
& (\xb_{k}, \xbhat_{k}, \zb_{k}) \in \mathcal{C}_k  && k \in \left\llbracket1, i-1\right\rrbracket.
\end{aligned}
\label{eq:primal-intermediates}
\end{equation}

As \eqref{eq:primal-intermediates} needs to be solved twice for each neuron (lower and upper bounds, changing the sign of the last layer's weights) rather than once as in \eqref{eq:primal-anderson}, depending on the computational budget, $\mathcal{C}_k$ might be looser than the relaxation employed for the last layer bounds (in our case, $\mathcal{A}_k$). In all our experiments, we compute intermediate bounds as the tightest bounds between the method by \citet{Wong2018} and Interval Propagation~\citep{Gowal2018}.

Once pre-activation bounds are available, post-activation bounds can be simply computed as $\lb_k = \max(\hat{\lb}_k, 0), \ub_k = \max(\hat{\ub}_k, 0)$.

\vspace{10pt}
\section{Pre-activation Bounds in $\mathcal{A}_k$} \label{sec:intermediate-anderson}
We now highlight the importance of an explicit treatment of pre-activation bounds in the context of the relaxation by \citet{Anderson2020}. 
In $\S$\ref{sec:preact-example} we will show through an example that, without a separate pre-activation bounds treatment, $\mathcal{A}_k$ could be looser than the less computationally expensive $\mathcal{M}_k$ relaxation. We then ($\S$\ref{sec:preact-anderson-derivation}) justify our specific pre-activation bounds treatment by extending the original proof by \citet{Anderson2020}.

The original formulation by \citet{Anderson2020} is the following:
\begin{equation}
\left. \begin{array} {l}
\xb_k \geq W_k \xb_{k-1} + \bb_{k} \\
\xb_k \leq \left( \begin{array}{l} \left(W_k  \odot I_{k}\right)  \xb_{k-1}  + \zb_k \odot \bb_k \\
- \left(W_k  \odot I_{k} \odot \lbrv_{k-1}\right) \diamond (1 - \zb_k) \\
+ \left(W_k \odot  \left(1 - I_{k}\right) \odot \ubrv_{k-1}\right) \diamond \zb_k \end{array} \right) 
\ \forall I_{k} \in 2^{W_k}  \\
 (\xb_{k}, \xbhat_{k}, \zb_{k}) \in [\lb_k, \ub_k]  \times [\hat{\lb}_k, \hat{\ub}_k] \times [\mbf{0},\ \mbf{1}]  \end{array} \right\} = \mathcal{A}'_k.
\label{eq:orig-anderson-set}
\end{equation}
While pre-activation bounds regularly appear as lower and upper bounds to $\xbhat_k$, they do not appear in any other constraint.
Indeed, the difference with respect to $\mathcal{A}_k$ as defined in equation \eqref{eq:primal-anderson} exclusively lies in the treatment of pre-activation bounds within the exponential family. 
Set $\mathcal{A}_k$ explicitly employs generic $\hat{\lb}_k, \hat{\ub}_k$ in the constraint set through $\xb_{k} \leq \hat{\ub}_k \odot \zb_{k}$, and $\xb_{k} \leq \xbhat_{k} - \hat{\lb}_k \odot  (1 - \zb_{k})$ from $\mathcal{M}_k$.
On the other hand,
$\mathcal{A}'_k$ implicitly sets $\hat{\lb}_k, \hat{\ub}_k$ to the value dictated by interval propagation bounds \citep{Gowal2018} via the constraints in $I_k = 0$ and $I_k = 1$ from the exponential family.
In fact, setting $I_k = 0$ and $I_k = 1$, we obtain the following two constraints:
\begin{equation}
\begin{aligned}
&\xb_{k} \leq \xbhat_{k} - M^-_{k} \odot  (1 - \zb_{k}) \\
&\xb_{k} \leq M^+_{k} \odot \zb_{k} \\
 \text{where:} \qquad &M^-_{k} := \min_{\xb_{k-1} \in [\lb_{k-1}, \ub_{k-1}]} W_k^T \xb_{k-1} + \bb_k = W_k \odot \lbrv_{k-1} + \bb_k  \\
&M^+_{k} := \max_{\xb_{k-1} \in [\lb_{k-1}, \ub_{k-1}]} W_k^T \xb_{k-1} + \bb_k =W_k \odot \ubrv_{k-1} + \bb_k 
\end{aligned}	
\label{eq:anderson-loose-bigm}
\end{equation}	
which correspond to the upper bounding ReLU constraints in $\mathcal{M}_k$ if we set $\hat{\lb}_k \rightarrow M^-_{k}, \hat{\ub}_k \rightarrow M^+_{k} $. While $\hat{\lb}_k, \hat{\ub}_k$ are (potentially) computed solving an optimisation problem over the entire network (problem \ref{eq:primal-intermediates}), the optimisation for $M^-_{k}, M^+_{k}$ involves only the layer before the current. Therefore, the constraints in \eqref{eq:anderson-loose-bigm} might be much looser than those in $\mathcal{M}_k$.

In practice, the effect of $\hat{\lb}_k[i], \hat{\ub}_k[i]$ on the resulting set is so significant that $\mathcal{M}_k$ might yield better bounds than $\mathcal{A}'_k$, even on very small networks. We now provide a simple example.

\subsection{Motivating Example} \label{sec:preact-example}

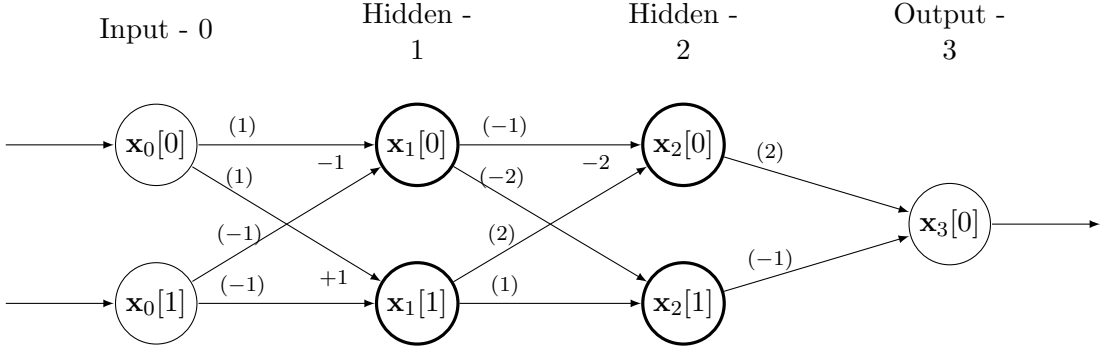
\begin{figure}[t]
	\centering
	\begin{tikzpicture}[
	clear/.style={ 
		draw=none,
		fill=none
	},
	net/.style={
		matrix of nodes,
		nodes={ draw, circle, inner sep=2pt },
		nodes in empty cells,
		column sep=1.5cm,
		row sep=-2pt,
		column 2/.style={very thick},
		column 3/.style={very thick},
	},
	>=latex
	]
	\matrix[net] (mat)
	{
		|[clear]| \parbox{1.7cm}{\centering Input - $0$} 
		& |[clear]| \parbox{1.7cm}{\centering Hidden - $1$} 
		& |[clear]| \parbox{1.7cm}{\centering Hidden - $2$} 
		& |[clear]| \parbox{1.7cm}{\centering Output - $3$} \\
		
		$\xb_{0}[0]$  & $\xb_{1}[0]$  & $\xb_{2}[0]$ & |[clear]|  \\
		|[clear]|   & |[clear]|  & |[clear]|  & $\xb_{3}[0]$ \\
		$\xb_{0}[1]$ & $\xb_{1}[1]$ & $\xb_{2}[1]$ & |[clear]|\\
	};
	\foreach \ai in {2,4}
	\draw[<-] (mat-\ai-1) -- +(-2cm,0);
	\draw[->] (mat-2-1) -- (mat-2-2) node [above=-1pt, near start, font=\fontsize{8}{0}] {$(1)$} node [below, near end, font=\fontsize{8}{0}] {$-1$};
	\draw[->] (mat-4-1) -- (mat-2-2) node [above=-1pt, near start, font=\fontsize{8}{0}] {$(-1)$};
	\draw[->] (mat-2-1) -- (mat-4-2) node [above=-1pt, near start, font=\fontsize{8}{0}] {$(1)$} node [below=3pt, near end, font=\fontsize{8}{0}] {$+1$};
	\draw[->] (mat-4-1) -- (mat-4-2) node [above=-1pt, near start, font=\fontsize{8}{0}] {$(-1)$};
	\draw[->] (mat-2-2) -- (mat-2-3) node [above=-1pt, near start, font=\fontsize{8}{0}] {$(-1)$} node [below, near end, font=\fontsize{8}{0}] {$-2$};
	\draw[->] (mat-4-2) -- (mat-2-3) node [above=-1pt, near start, font=\fontsize{8}{0}] {$(2)$};
	\draw[->] (mat-2-2) -- (mat-4-3) node [above=-1pt, near start, font=\fontsize{8}{0}] {$(-2)$};
	\draw[->] (mat-4-2) -- (mat-4-3) node [above=-1pt, near start, font=\fontsize{8}{0}] {$(1)$};
	\draw[->] (mat-2-3) -- (mat-3-4) node [above=-1pt, near start, font=\fontsize{8}{0}] {$(2)$};
	\draw[->] (mat-4-3) -- (mat-3-4) node [above=-1pt, near start, font=\fontsize{8}{0}] {$(-1)$};
	
	\draw[->] (mat-3-4) -- +(2cm,0);
	\end{tikzpicture}
	\caption{Example network architecture in which $\mathcal{M}_k \subset \mathcal{A}'_k$, with pre-activation bounds computed with $\mathcal{C}_k=\mathcal{M}_k$. For the bold nodes (the two hidden layers) a ReLU activation follows the linear function. The numbers between parentheses indicate multiplicative weights, the others additive biases (if any).}
	\label{fig:example}
\end{figure}

Figure \ref{fig:example} illustrates the network architecture. The size of the network is the minimal required to reproduce the phenomenon. $\mathcal{M}_k$ and $\mathcal{A}_k$ coincide for single-neuron layers \citep{Anderson2020}, and $\hat{\lb}_k =M^-_{k}, \hat{\ub}_k = M^+_{k}$ on the first hidden layer (hence, a second layer is needed).

Let us write the example network as a (not yet relaxed, as in problem \eqref{eq:noncvx_pb}) optimization problem for the lower bound on the output node $\xb_{3}$.

\begin{subequations}
	\begin{alignat}{2}
	\lb_3 = \text{arg min }_{\xb, \xbhat}\quad& \begin{bmatrix}2 & -1\end{bmatrix} \xb_{2} \span\\[5px]
	\text{s.t. }\quad& \xb_0 \in [-1, 1]^2\span\\[5px]
	& \xbhat_{1} =\begin{bmatrix}1 & -1 \\ 1 & -1 \end{bmatrix}  \xb_{0} + \begin{bmatrix}-1 \\ 1\end{bmatrix} \qquad \xb_1 = \max(0, \xbhat_{1})  \\[5px]
	& \xbhat_{2} =\begin{bmatrix}-1 & 2 \\ -2 & 1 \end{bmatrix}  \xb_{1} + \begin{bmatrix}-2 \\ 0\end{bmatrix} \qquad \xb_2 = \max(0,  \xbhat_{2}) \\[5px]
	& \xb_{3} = \begin{bmatrix}2 & -1\end{bmatrix}  \xb_{2}
	\end{alignat}
\end{subequations}

Let us compute pre-activation bounds with $\mathcal{C}_k=\mathcal{M}_k$ (see problem \eqref{eq:primal-intermediates}).
For this network, the final output lower bound is tighter if the employed relaxation is $\mathcal{M}_k$ rather than $\mathcal{A}_k$ (hence, in this case, $\mathcal{M}_k \subset \mathcal{A}'_k$). Specifically: $\hat{\lb}_{3, \mathcal{A}'_k} = -1.2857$, $\hat{\lb}_{3, \mathcal{M}_k} = -1.2273$. In fact:

\begin{itemize}
	\item In order to compute $\lb_1$ and $\ub_1$, the post-activation bounds of the first-layer, it suffices to solve a box-constrained linear program for $\hat{\lb}_1$ and $\hat{\ub}_1$, which at this layer coincide with interval propagation bounds, and to clip them to be non-negative. This yields $\lb_1 = \begin{bmatrix}0 & 0\end{bmatrix}^T$, $\ub_1 = \begin{bmatrix}1 & 3\end{bmatrix}^T$.
	\item  Computing $M^+_{2}[1] = \max_{\xb_{1} \in [\lb_{1}, \ub_{1}]} \begin{bmatrix}-2 & 1 \end{bmatrix}  \xb_{1} = 3$ we are assuming that $\xb_{1}[0] = \lb_{1}[0]$ and $\xb_{1}[1] = \ub_{1}[1]$. These two assignments are in practice conflicting, as they imply different values for $\xb_{0}$.  Specifically, $\xb_{1}[1] = \ub_{1}[1]$ requires $\xb_{0} =\begin{bmatrix}\ub_{0}[0] & \lb_{0}[1] \end{bmatrix} = \begin{bmatrix}1 & -1 \end{bmatrix}$, but this would also imply $\xb_{1}[0] = \ub_{1}[0]$, yielding $\xbhat_{2}[1]=1 \neq 3$.
	
	Therefore, explicitly solving a LP relaxation of the network for the value of $\hat{\ub}_2[1]$ will tighten the bound. Using $\mathcal{M}_k$, the LP for this intermediate pre-activation bound is:
	\begin{subequations}
		\begin{alignat}{2}
		\hat{\ub}_2[1] = \text{arg min }_{\xb, \xbhat, \zb}\quad& \begin{bmatrix}-2 & 1 \end{bmatrix} \xb_{1} \span\\[5px]
		\text{s.t. }\quad& \xb_0 \in [-1, 1]^2 \text{, } \zb_1 \in [0, 1]^2\text{, }  \xb_{1} \in \mathbb{R}_{\geq 0}^2 \span\\[5px]
		& \xbhat_{1} =\begin{bmatrix}1 & -1 \\ 1 & -1 \end{bmatrix}  \xb_{0} + \begin{bmatrix}-1 \\ 1\end{bmatrix} \\
		& \xb_{1} \geq \xbhat_{1} \\
		& \xb_{1} \leq \hat{\ub}_1 \odot \zb_{1} = \begin{bmatrix}1 \\ 3\end{bmatrix} \odot \zb_{1} \\
		& \xb_{1} \leq \xbhat_{1} - \hat{\lb}_1 \odot (1 - \zb_{1}) = \xbhat_{1} -  \begin{bmatrix}-3 \\ -1\end{bmatrix}  \odot (1 - \zb_{1})
		\end{alignat}
	\end{subequations}
	Yielding $\hat{\ub}_2[1] = 2.25 < 3 = M^+_{2}[1] $. An analogous reasoning holds for $M^-_{2}[1]$ and $\hat{\lb}_2[1]$.
	\item In $\mathcal{M}_k$, we therefore added the following two constraints: 
	\begin{equation}
	\begin{split}
	& \xb_{2}[1] \leq \xbhat_{2}[1] - \hat{\lb}_2[1] (1-\zb_{2}[1]) \\
	& \xb_{2}[1] \leq \hat{\ub}_2[1]  \zb_{2}[1]
	\end{split}
	\label{eq:bigmconstraints}
	\end{equation}	
	that in $\mathcal{A}'_k$ correspond to the weaker:
	\begin{equation}
	\begin{split}
	& \xb_{2}[1] \leq \xbhat_{2}[1] - M^-_{2}[1] (1-\zb_{2}[1]) \\
	& \xb_{2}[1] \leq M^+_{2}[1] \zb_{2}[1]
	\end{split}
	\label{eq:andersonbigmconstraints}
	\end{equation}
	
	As the last layer weight corresponding to $\xb_2[1]$ is negative ($W_3[0, 1] = -1$), these constraints are going to influence the computation of $\hat{\lb}_{3}$.
	
	\item In fact, the constraints in \eqref{eq:bigmconstraints} are both active when optimizing for $\hat{\lb}_{3, \mathcal{M}_k} $, whereas their counterparts for $\hat{\lb}_{3, \mathcal{A}'_k} $ in \eqref{eq:andersonbigmconstraints} are not.
	The only active upper constraint at neuron $\xb_{2}[1]$ for the Anderson relaxation is $\xb_{2}[1] \leq \xb_{1}[1]$, corresponding to the constraint from $\mathcal{A}'_2$ with $I_2[1, \cdot]=
	[0\ 1]$. Evidently, its effect is not sufficient to counter-balance the effect of the tighter constraints \eqref{eq:bigmconstraints} for $I_2[1, \cdot]=
	[1\ 1]$ and $I_2[1, \cdot]=
	[0\ 0]$, yielding a weaker lower bound for the network output.
	
\end{itemize}

\subsection{Derivation of $\mathcal{A}_k$} \label{sec:preact-anderson-derivation}

Having motivated an explicit pre-activation bounds treatment for the relaxation by \citet{Anderson2020}, we now extend the original proof for $\mathcal{A}'_k$ (equation \eqref{eq:orig-anderson-set}) to obtain our formulation $\mathcal{A}_k$ (as defined in equation \eqref{eq:primal-anderson}). For simplicity, we will operate on a single neuron $\xb_{k}[i]$.

A self-contained way to derive $\mathcal{A}'_k$ is by applying Fourier-Motzkin elimination on a standard MIP formulation referred to as the \emph{multiple choice} formulation \citep{Anderson2019}, which is defined as follows:
\begin{equation}
	\left. \begin{array}{l}
	 (\xb_{k-1}, \xb_{k}[i]) =  (\xb_{k-1}^0, \xb_{k}^0[i]) + (\xb_{k-1}^1, \xb_{k}^1[i])\\
	 \xb_{k}^0[i] = 0 \geq \boldsymbol{w}_{i, k}^T \xb_{k-1}^0 + \bb_{k}[i]  (1 - \zb_{k}[i]) \\
	 \xb_{k}^1[i] =  \boldsymbol{w}_{i, k}^T \xb_{k-1}^1 + \bb_{k} [i] \zb_{k}[i] \geq 0 \\
	  \lb_{k-1} (1 - \zb_{k}[i]) \leq \xb_{k-1}^0  \leq \ub_{k-1}  (1 - \zb_{k}[i]) \\
	  \lb_{k-1}\zb_{k}[i] \leq \xb_{k-1}^1  \leq \ub_{k-1} \zb_{k}[i] \\
	 \zb_k[i] \in [0,\ 1]
	\end{array} \right\} = \mathcal{S}'_{k, i}
\label{eq:mc}
\end{equation}	
Where $\boldsymbol{w}_{i, k}$ denotes the $i$-th row of $W_k$, and $\xb_{k-1}^1$ and $\xb_{k-1}^0$ are copies of the previous layer variables. Applying \eqref{eq:mc} to the entire neural network results in a quadratic number of variables (relative to the number of neurons). 
The formulation can be obtained from well-known techniques from the MIP literature \citep{Jeroslow} (it is the union of the two polyhedra for a passing and a blocking ReLU, operating in the space of $\xb_{k-1}$). \citet{Anderson2019} show that $\mathcal{A}'_{k} = \text{Proj}_{\xb_{k-1}, \xb_{k}, \zb_{k}} (\mathcal{S}'_{k})$.

If pre-activation bounds $\hat{\lb}_k, \hat{\ub}_k$ (computed as described in Section \ref{sec:intermediate}) are available, we can naturally add them to  \eqref{eq:mc} as follows:

\begin{equation}	
\begin{split}
	\left. \begin{array}{l}
	(\xb_{k-1}, \xb_{k}[i], \zb_{k}[i]) \in \mathcal{S}'_{k,i} \\
	\hat{\lb}_k[i](1 - \zb_{k}[i]) \leq \boldsymbol{w}_{i, k}^T \xb_{k-1}^0 + \bb_{k}[i](1 - \zb_{k}[i])  \leq \hat{\ub}_k[i](1 - \zb_{k}[i]) \\
	\hat{\lb}_k[i] \odot \zb_{k}[i] \leq  \boldsymbol{w}_{i, k}^T \xb_{k-1}^1 + \bb_{k}[i]\zb_{k}[i] \leq \hat{\ub}_k[i]\zb_{k}[i] 
	\end{array} \right\} = \mathcal{S}_{k,i} 
\end{split}
\label{eq:mc-preact}
\end{equation}	

We now prove that this formulation yields $\mathcal{A}_{k}$ when projecting out the copies of the activations.
\\
\begin{proposition}
	Sets $\mathcal{S}_{k}$ from equation \eqref{eq:mc-preact} and $\mathcal{A}_{k}$ from problem \eqref{eq:primal-anderson} are equivalent, in the sense that $\mathcal{A}_{k} = \text{Proj}_{\xb_{k-1}, \xb_{k}, \zb_{k}} (\mathcal{S}_{k})$.
\end{proposition}
\begin{proof}
	In order to prove the equivalence, we will rely on Fourier-Motzkin elimination as in the original Anderson relaxation proof \citep{Anderson2019}.
	Going along the lines of the original proof, we start from \eqref{eq:mc} and eliminate $\xb_{k-1}^1$, $\xb_{k}^0[i]$ and $\xb_{k}^1[i]$ exploiting the equalities. We then re-write all the inequalities as upper or lower bounds on $\xb_{k-1}^0[0]$ in order to eliminate this variable. As \citet{Anderson2019}, we assume $\boldsymbol{w}_{i, k}[0] > 0$. The proof generalizes by using $\lbrv$ and $\ubrv$ for $\boldsymbol{w}_{i, k}[0] < 0$, whereas if the coefficient is $0$ the variable is easily eliminated. 
	We get the following system:
	\begin{subequations}
		\begin{alignat}{2}
		& \xelim = \frac{1}{\welim} \left( \wcurr^T \xb_{k-1} - \wxrem + \boldsymbol{b}_k[i] \zb_{k}[i] - \xb_{k}[i] \right) \label{subeq:equality} \\
		& \xelim \leq - \frac{1}{\welim} \left( \wxrem + \boldsymbol{b}_k[i] (1 - \zb_{k}[i]) \right) \label{subeq:one} \\
		& \xelim \leq \frac{1}{\welim} \left( \wcurr^T \xb_{k-1} - \wxrem + \boldsymbol{b}_k[i] \zb_{k}[i]  \right) \\
		& \lb_{k-1} [0] (1 - \zb_{k}[i]) \leq \xelim \leq \ub_{k-1} [0] (1 - \zb_{k}[i]) \label{subeq:three-four} \\
		& \xelim \leq \xb_{k-1}[0] -  \lb_{k-1} [0] \zb_{k}[i] \\
		& \xelim \geq \xb_{k-1}[0] -  \ub_{k-1} [0] \zb_{k}[i] \label{subeq:six}  \\
		& \xelim \leq \frac{1}{\welim} \left( \wcurr^T \xb_{k-1} - \wxrem + (\boldsymbol{b}_k[i] - \hat{\lb}_k[i]) \zb_{k}[i] \right) \label{subeq:alpha} \\
		& \xelim \geq \frac{1}{\welim} \left( \wcurr^T \xb_{k-1} - \wxrem + (\boldsymbol{b}_k[i] - \hat{\ub}_k[i]) \zb_{k}[i] \right) \label{subeq:beta}  \\
		& \xelim \geq \frac{1}{\welim} \left( (\hat{\lb}_k[i] - \boldsymbol{b}_k[i]) (1 -\zb_{k}[i]) - \wxrem  \right) \label{subeq:gamma}  \\
		& \xelim \leq \frac{1}{\welim} \left( (\hat{\ub}_k[i] - \boldsymbol{b}_k[i]) (1 -\zb_{k}[i]) - \wxrem  \right) \label{subeq:delta}
		\end{alignat}
	\end{subequations}
	
	where only inequalities \eqref{subeq:alpha} to \eqref{subeq:delta} are not present in the original proof. We therefore focus on the part of the Fourier-Motzkin elimination that deals with them, and invite the reader to refer to \citet{Anderson2019} for the others. The combination of these new inequalities yields trivial constraints. For instance: 
	\begin{equation}
	\eqref{subeq:gamma} + \eqref{subeq:alpha} \implies \hat{\lb}_k[i] \leq \wcurr^T \xb_{k-1} + \boldsymbol{b}_k[i] = \hat{\xb}_k[i]
	\end{equation}
	which holds by the definition of pre-activation bounds.
	
	Let us recall that $\xb_k[i] \geq 0$ and $\xb_k[i] \geq \hat{\xb}_k[i]$, the latter constraint resulting from $\eqref{subeq:equality} + \eqref{subeq:one}$. Then, it can be easily verified that the only combinations of interest (i.e., those that do not result in constraints that are obvious by definition or are implied by other constraints) are those containing the equality \eqref{subeq:equality}. In particular, combining inequalities \eqref{subeq:alpha} to \eqref{subeq:delta} with inequalities \eqref{subeq:three-four} to \eqref{subeq:six} generates constraints that are (after algebraic manipulations) superfluous with respect to those in \eqref{eq:fin-system}. We are now ready to show the system resulting from the elimination:
	
	\begin{subequations}
		\begin{alignat}{2}
		& \xb_{k}[i] \geq 0 \label{subeq:first-old} \\
		& \xb_{k}[i] \geq \xbhat_{k}[i] \label{subeq:geq-xhat} \\
		& \xb_{k}[i] \leq \welim \xb_{k-1}[0] - \welim \lb_{k-1}[0] (1 -\zb_{k}[i]) + \wxrem + \boldsymbol{b}_k[i] \zb_{k}[i]\\
		& \xb_{k}[i] \leq  \welim \ub_{k-1}[0] \zb_{k}[i] + \wxrem+ \boldsymbol{b}_k[i] \zb_{k}[i] \\
		& \xb_{k}[i] \geq \welim \xb_{k-1}[0] - \welim \ub_{k-1}[0] (1 -\zb_{k}[i]) + \wxrem + \boldsymbol{b}_k[i] \zb_{k}[i] \\
		& \xb_{k}[i] \geq  \welim \lb_{k-1}[0] \zb_{k}[i] + \wxrem+ \boldsymbol{b}_k[i] \zb_{k}[i] \\
		& \lb_{k-1}[0]  \leq \xb_{k}[i] \leq \ub_{k-1}[0] \label{subeq:last-old} \\
		& \xb_{k}[i] \geq \hat{\lb}_k[i] \zb_{k}[i] \label{subeq:superfluous-one} \\
		& \xb_{k}[i] \leq \hat{\ub}_k[i] \zb_{k}[i] \label{subeq:superfluous-three} \\
		& \xb_{k}[i] \leq \xbhat_{k}[i]  - \hat{\lb}_k[i] (1 -\zb_{k}[i]) \label{subeq:superfluous-four} \\
		& \xb_{k}[i] \geq \xbhat_{k}[i]  - \hat{\ub}_k[i] (1 -\zb_{k}[i]) \label{subeq:superfluous-two}
		\end{alignat}
		\label{eq:fin-system}
	\end{subequations}
	
	Constraints from \eqref{subeq:first-old} to \eqref{subeq:last-old} are those resulting from the original derivation of $\mathcal{A}'_k$ (see \citep{Anderson2019}). The others result from the inclusion of pre-activation bounds in \eqref{eq:mc-preact}. Of these, \eqref{subeq:superfluous-one} is implied by \eqref{subeq:first-old} if $\hat{\lb}_k[i] \leq 0$ and by the definition of pre-activation bounds (together with \eqref{subeq:geq-xhat}) if $\hat{\lb}_k[i] > 0$.  
	Analogously, \eqref{subeq:superfluous-two} is implied by \eqref{subeq:geq-xhat} if $\hat{\ub}_k[i] \geq 0$ and by \eqref{subeq:first-old} otherwise.
	
	By noting that no auxiliary variable is left in \eqref{subeq:superfluous-three} and in \eqref{subeq:superfluous-four}, we can conclude that these will not be affected by the remaining part of the elimination procedure. Therefore, the rest of the proof (the elimination of $\xb_{k-1}^0[1]$, $\xb_{k-1}^0[2]$, $\dots$) proceeds as in \citep{Anderson2019}, leading to $\mathcal{A}_{k, i}$. 
	Repeating the proof for each neuron $i$ at layer $k$, we get $\mathcal{A}_{k} = \text{Proj}_{\xb_{k-1}, \xb_{k}, \zb_{k}} (\mathcal{S}_{k})$.
	
\end{proof}

\section{Masked Forward and Backward Passes} \label{sec:masked-passes}

Crucial to the practical efficiency of our solvers is to represent the various operations as standard forward/backward passes over a neural network. This way, we can leverage the engineering efforts behind popular deep learning frameworks such as PyTorch~\citep{Paszke2017}. 
While this can be trivially done for the Big-M solver (appendix \ref{sec:dual-init}), both the Active Set ($\S$\ref{sec:active}) and Saddle Point ($\S$\ref{sec:sp}) solvers require a specialized operator that we call ``masked" forward/backward pass.
We now provide the details to our implementation.

\vspace{7pt}
As a reminder from $\S$\ref{sec:implementation-details}, masked forward and backward passes respectively take the following forms (writing convolutional operators via their fully connected equivalents): 
\begin{equation*}
	\left(W_k \odot I_{k}\right) \mbf{a}_{k}, \quad \left(W_k \odot I_{k}\right)^T \mbf{a}_{k+1},
\end{equation*}
where $\mbf{a}_{k} \in \mathbb{R}^{n_k}, \mbf{a}_{k+1} \in \mathbb{R}^{n_{k+1}}$. They are needed when dealing with the exponential family of constraints from the relaxation by~\citet{Anderson2020}.

\vspace{7pt}
Masked operators are straightforward to implement for fully connected layers (via element-wise products). We instead need to be more careful when handling convolutional layers. Standard convolution relies on re-applying the same weights (kernel) to many different parts of the image.
A masked pass, instead, implies that the convolutional kernel is dynamically changing while it is being slided through the image.
A naive solution is to convert convolutions into equivalent linear operators, but this has a high cost in terms of performance, as it involves much redundancy. Our implementation relies on an alternative view of convolutional layers, which we outline next.

\subsection{Convolution as Matrix-matrix Multiplication}
A convolutional operator can be represented via a matrix-matrix multiplication if the input is \emph{unfolded} and the filter is appropriately reshaped. The multiplication output can then be reshaped to the correct convolutional output shape. Given a filter $w \in \mathbb{R}^{c \times k_1 \times k_2}$, an input $\xb \in \mathbb{R}^{i_1 \times i_2 \times i_3}$ and the convolutional output $\texttt{conv}_{w}(\xb) = \yb \in \mathbb{R}^{c \times o_2 \times o_3}$, we need the following definitions:
\begin{equation}
\begin{split}
[\cdot]_{\mathcal{I}, \mathcal{O}} &:  \mathcal{I} \to \mathcal{O}\\
\{\cdot\}_{j} &:  \mathbb{R}^{d_1 \times \dots \times d_n} \to \mathbb{R}^{d_1 \times \dots \times d_{j-1} \times d_{j+1} \times \dots \times  d_n}  \\
\texttt{unfold}_w(\cdot) &: \mathbb{R}^{i_1 \times i_2 \times i_3} \to \mathbb{R}^{k_1 k_2 \times o_2 o_3} \\
\texttt{fold}_w(\cdot) &: \mathbb{R}^{k_1 k_2 \times o_2 o_3} \to \mathbb{R}^{i_1 \times i_2 \times i_3} 
\end{split}
\end{equation}

where the brackets simply reshape the vector from shape $\mathcal{I}$ to $\mathcal{O}$, while the braces sum over the $j$-th dimension.
 $\texttt{unfold}$ decomposes the input image into the (possibly overlapping) $o_2 o_3$ blocks the sliding kernel operates on, taking padding and striding into account. $\texttt{fold}$ brings the output of $\texttt{unfold}$ to the original input space.

Let us define the following reshaped versions of the filter and the convolutional output:
\begin{equation*}
	\begin{gathered}
		W_R = [w]_{\mathbb{R}^{c \times k_1 \times k_2}, \mathbb{R}^{c \times k_1  k_2}} \\
		\yb_R = [\yb]_{\mathbb{R}^{c \times o_2 \times o_3}, \mathbb{R}^{c \times o_2 o_3}}
	\end{gathered}
\end{equation*}

The standard forward/backward convolution (neglecting the convolutional bias, which can be added at the end of the forward pass) can then be written as:
\begin{equation}
\begin{split}
\texttt{conv}_{w}(\xb) &= [W_R \ \texttt{unfold}_{w}(\xb)]_{\mathbb{R}^{c\times o_2 o_3}, \mathbb{R}^{c \times o_2 \times o_3}}  \\
\texttt{back\_conv}_{w}(\yb) &= \texttt{fold}_w (W_R  ^T \  \yb_R ).
\end{split}
\label{eq:unfold-conv}
\end{equation}

\subsection{Masked Convolution as Matrix-matrix Multiplication}
We need to mask the convolution with a different scalar for each input-output pair. 
Therefore, we employ a mask $I \in \mathbb{R}^{c \times k_1  k_2 \times o_2 o_3}$, whose additional dimension with respect to $W_R$ is associated to the output space of the convolution.
Assuming vectors are broadcast to the correct output shape\footnote{if we want to perform an element-wise product $\mbf{a} \odot \mbf{b}$ between $\mbf{a} \in \mathbb{R}^{d_1 \times d_2 \times d_3}$ and $\mbf{b} \in \mathbb{R}^{d_1 \times d_3}$, the operation is implicitly performed as $\mbf{a} \odot \mbf{b}'$, where $\mbf{b}' \in \mathbb{R}^{d_1 \times d_2 \times d_3}$ is an extended version of $\mbf{b}$ obtained by copying along the missing dimension.}
, we can write the masked forward and backward passes by adapting equation \eqref{eq:unfold-conv} as follows:
\begin{equation}
\begin{gathered}
\texttt{conv}_{w, I}(\xb) = \left[\left\{W_R \ \odot I \odot \texttt{unfold}_{w}(\xb)\right\}_2 \right]_{\mathbb{R}^{c\times o_2 o_3}, \mathbb{R}^{c \times o_2 \times o_3}}  \\
\texttt{back\_conv}_{w, I}(\yb) = \texttt{fold}_w (\left\{W_R \ \odot I \odot \yb_R \right\}_1 ).
\end{gathered}
\end{equation}

Owing to the avoided redundancy with respect to the equivalent linear operation (e.g., copying of the kernel matrix, zero-padding in the linear weight matrix), this implementation of the masked forward/backward pass reduces both the memory footprint and the number of floating point operations (FLOPs) associated to the passes computations by a factor $(i_1 i_2 i_3) /(k_1 k_2)$. In practice, this ratio might be significant: on the incomplete verification networks ($\S$\ref{sec:incomp_verif}) it ranges from $16$ to $64$ depending on the layer.  \vspace{30pt}

\begin{figure*}[b!]
	\centering
	\begin{subfigure}{\textwidth}
		\includegraphics[width=1\textwidth]{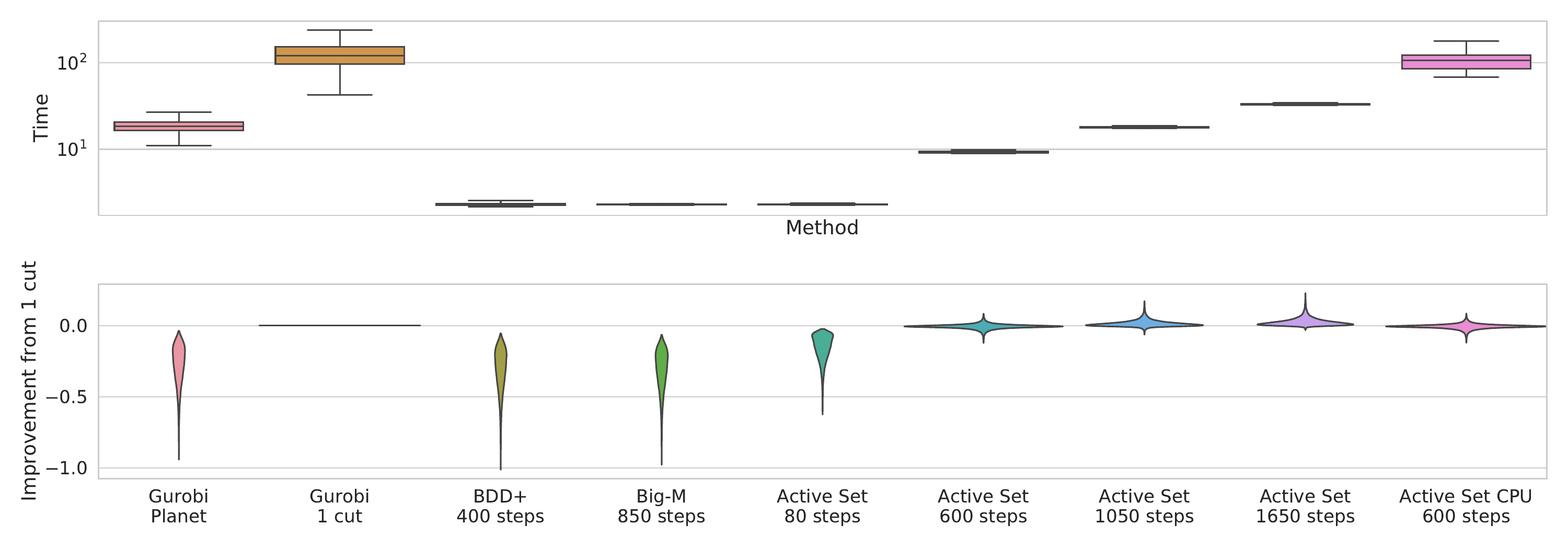}
		\vspace{-20pt}
		\subcaption{Speed-accuracy trade-offs of Active Set for different iteration ranges and computing devices.}
		\vspace{5pt}
	\end{subfigure}
	\begin{subfigure}{\textwidth}
		\includegraphics[width=1\textwidth]{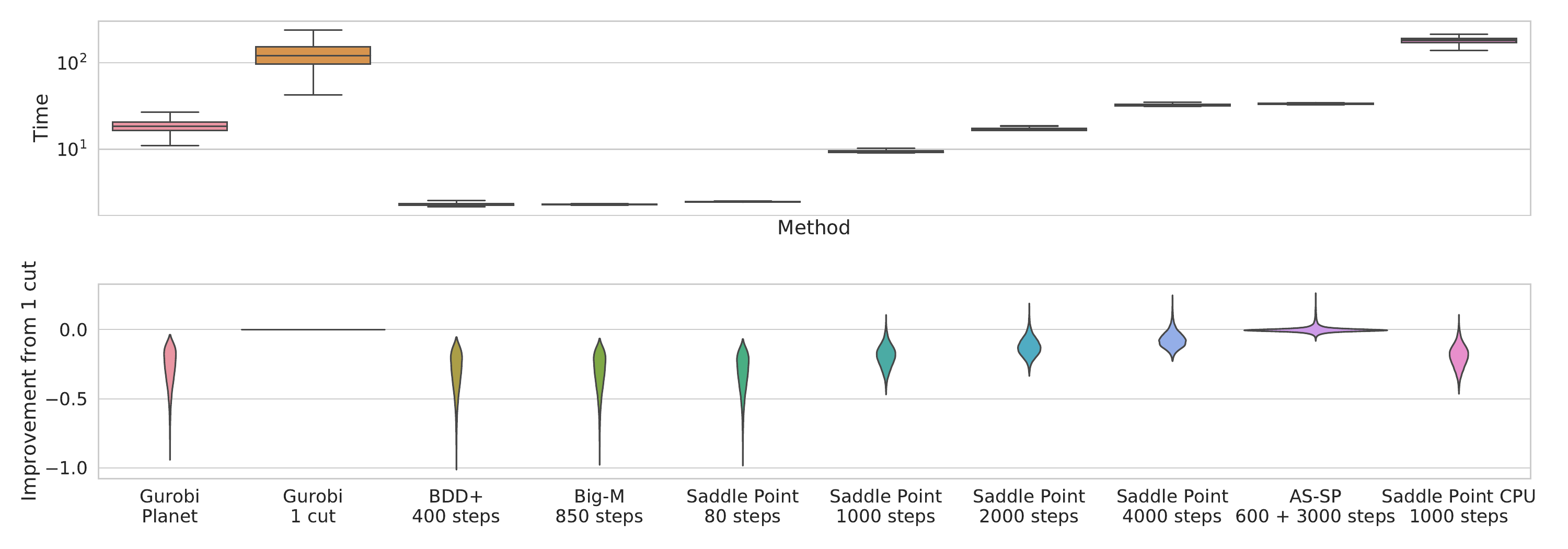}
		\vspace{-20pt}
		\subcaption{Speed-accuracy trade-offs of Saddle Point for different iteration ranges and computing devices.}
	\end{subfigure}
	\caption{\label{fig:madry8} 
		Upper bounds to the adversarial vulnerability for the network adversarially trained with the method by ~\citet{Madry2018}, from \citet{BunelDP20}.
		Box plots: distribution of runtime in seconds.
		Violin plots: difference with the bounds obtained by Gurobi with a cut from $\mathcal{A}_k$ per neuron; higher is better, the width at a given value represents the
		proportion of problems for which this is the result. }
\end{figure*}

\section{Experimental Appendix} \label{sec:exp-appendix}

We conclude the appendix by presenting supplementary experiments with respect to the presentation in the main paper.

\subsection{Adversarially-Trained CIFAR-10 Incomplete Verification} \label{sec:incomplete-madry}

\begin{figure*}[t!]
	\vspace{-8pt}
	\begin{subfigure}{.99\textwidth}
		\begin{minipage}{.48\textwidth}
			\centering
			\includegraphics[width=\textwidth]{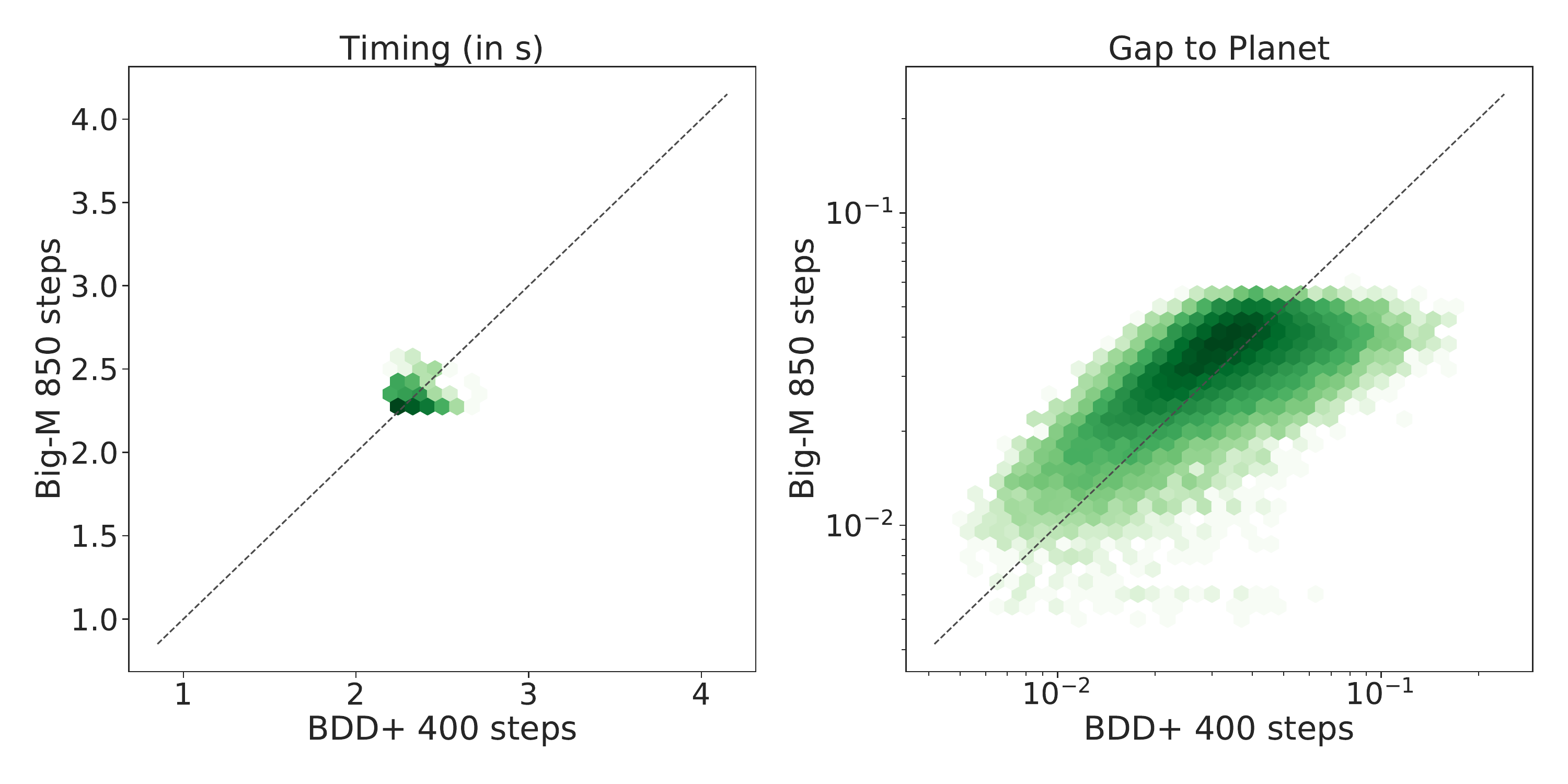}
			\vspace{-20pt}
		\end{minipage}
		\hspace{5pt}
		\begin{minipage}{.4\textwidth}
			\subcaption{\small Comparison of runtime (left) and gap to Gurobi Planet bounds (right). For the latter, lower is better.}
		\end{minipage}
		\vspace{8pt}
	\end{subfigure}
	\label{fig:madry-planet-gap}
	\begin{subfigure}{.99\textwidth}
		\begin{minipage}{.48\textwidth}
			\centering
			\includegraphics[width=\textwidth]{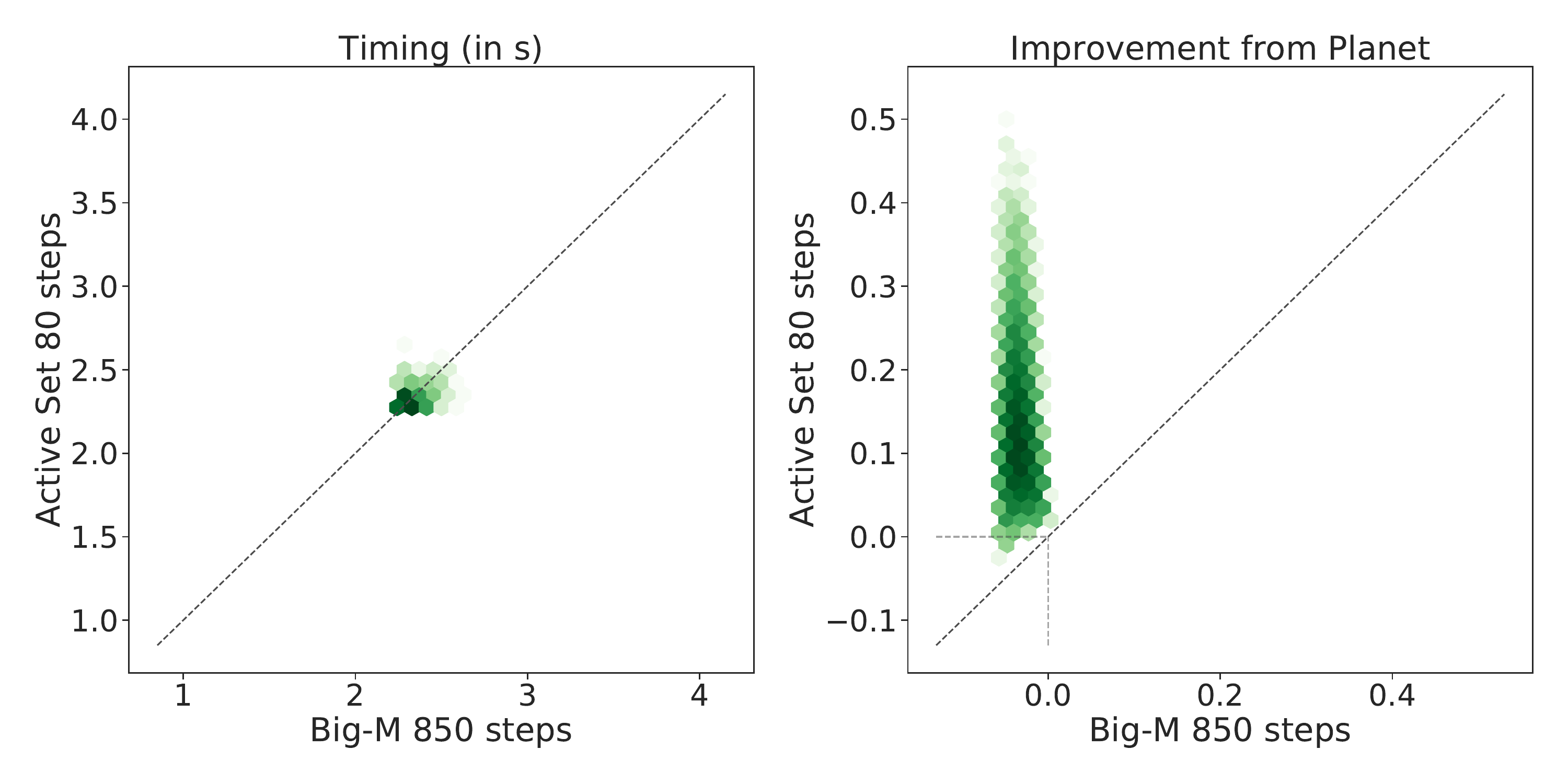}
		\end{minipage}
		\hspace{5pt}
		\begin{minipage}{.48\textwidth}
			\centering
			\includegraphics[width=\textwidth]{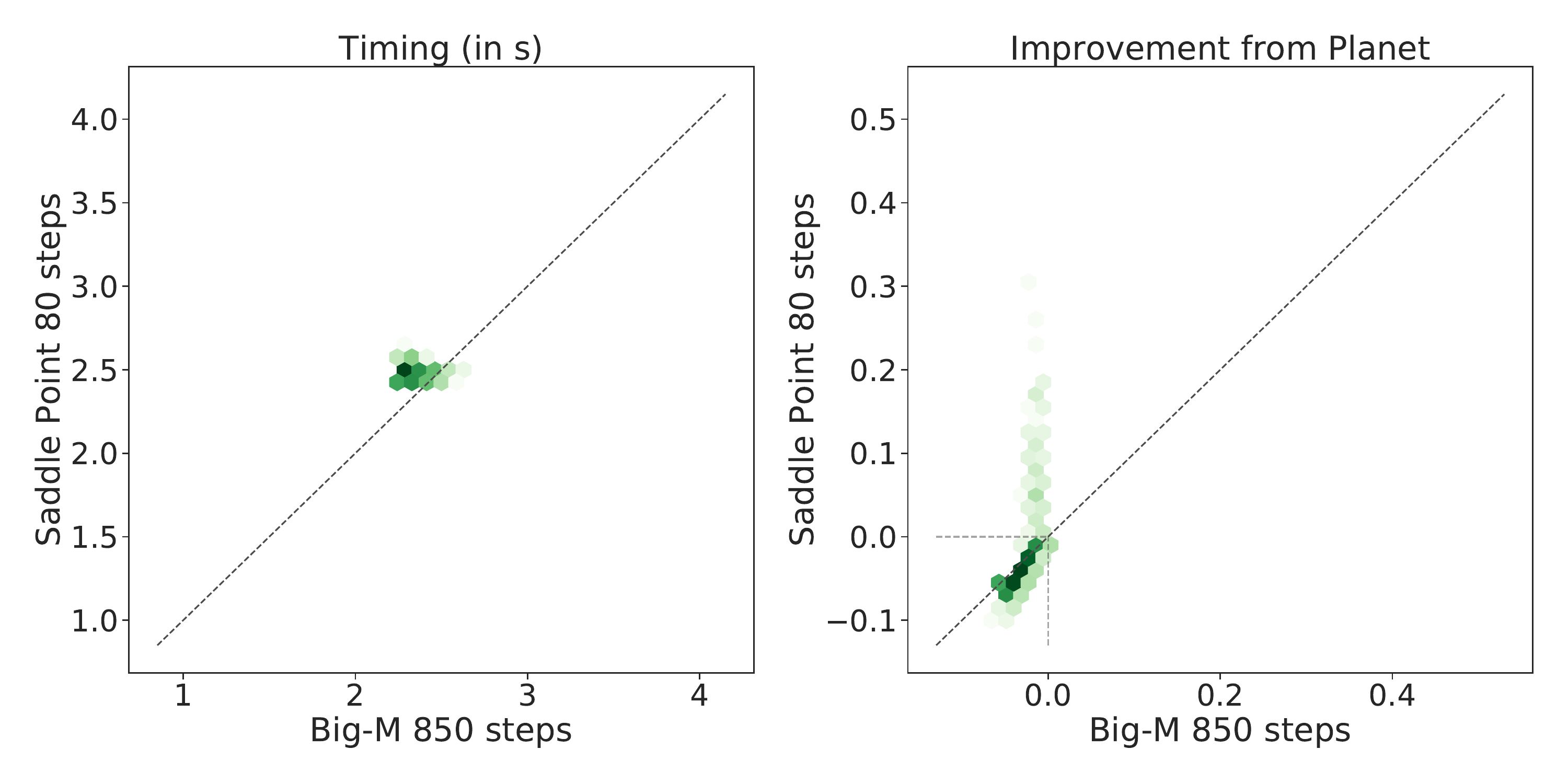}
		\end{minipage}
		\vspace{-8pt}
		\subcaption{Comparison of runtime (Timing) and difference with the Gurobi Planet bounds (Improvement from Planet). For the latter, higher is better.}
		\label{fig:madry-planet-improvement}
	\end{subfigure}
	\vspace{-5pt}
	\caption{Pointwise comparison for a subset of the methods on the data presented in figure~\ref{fig:sgd8}. Darker colour shades mean higher point density (on a logarithmic scale). The oblique dotted line corresponds to the equality.}
	\label{fig:madry8-pointwise}
\end{figure*}
\begin{figure*}[t!]
	\centering
	\begin{minipage}{.48\textwidth}
		\centering
		\includegraphics[width=\textwidth]{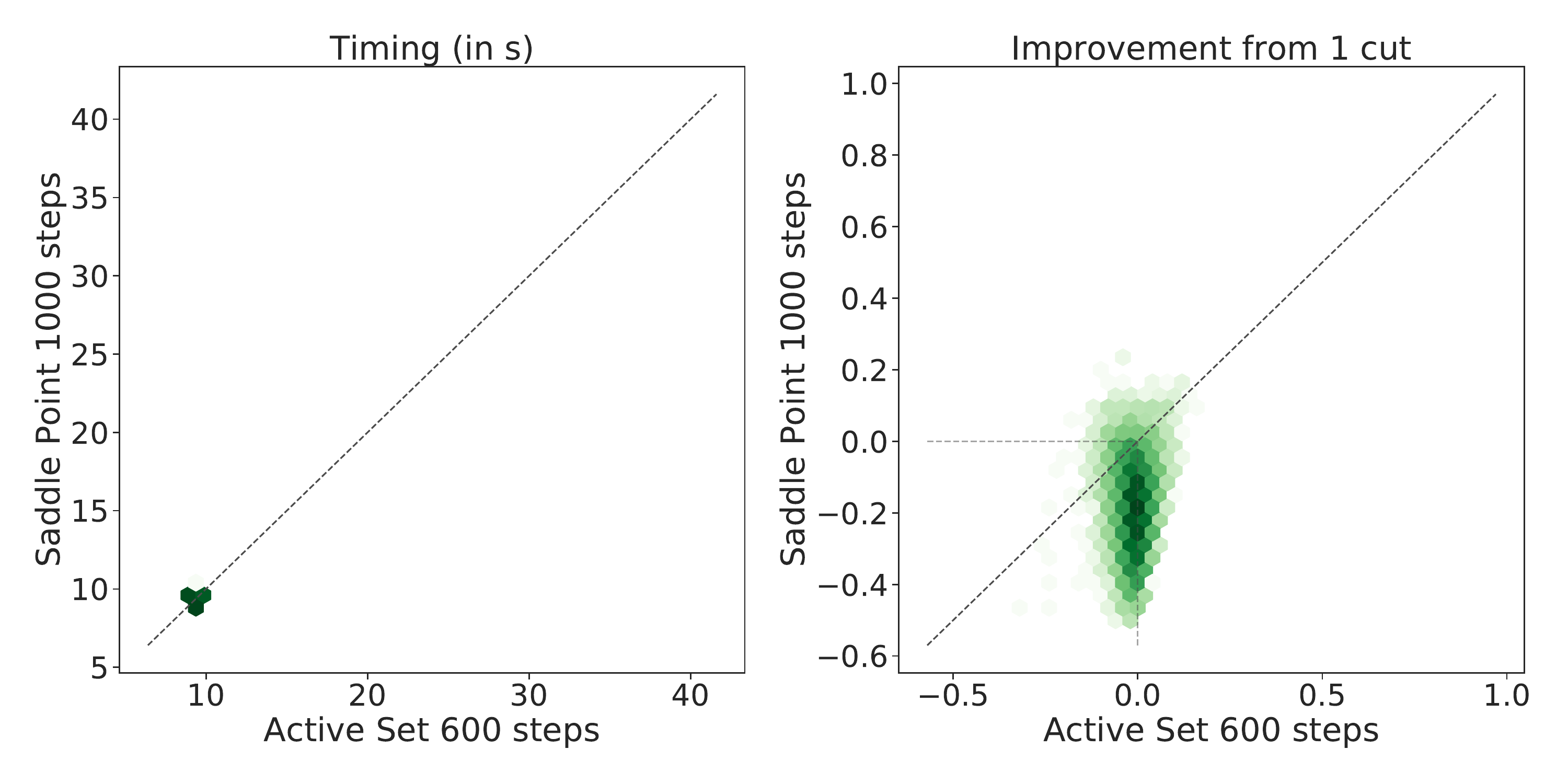}
	\end{minipage}
	\hspace{5pt}
	\begin{minipage}{.48\textwidth}
		\centering
		\includegraphics[width=\textwidth]{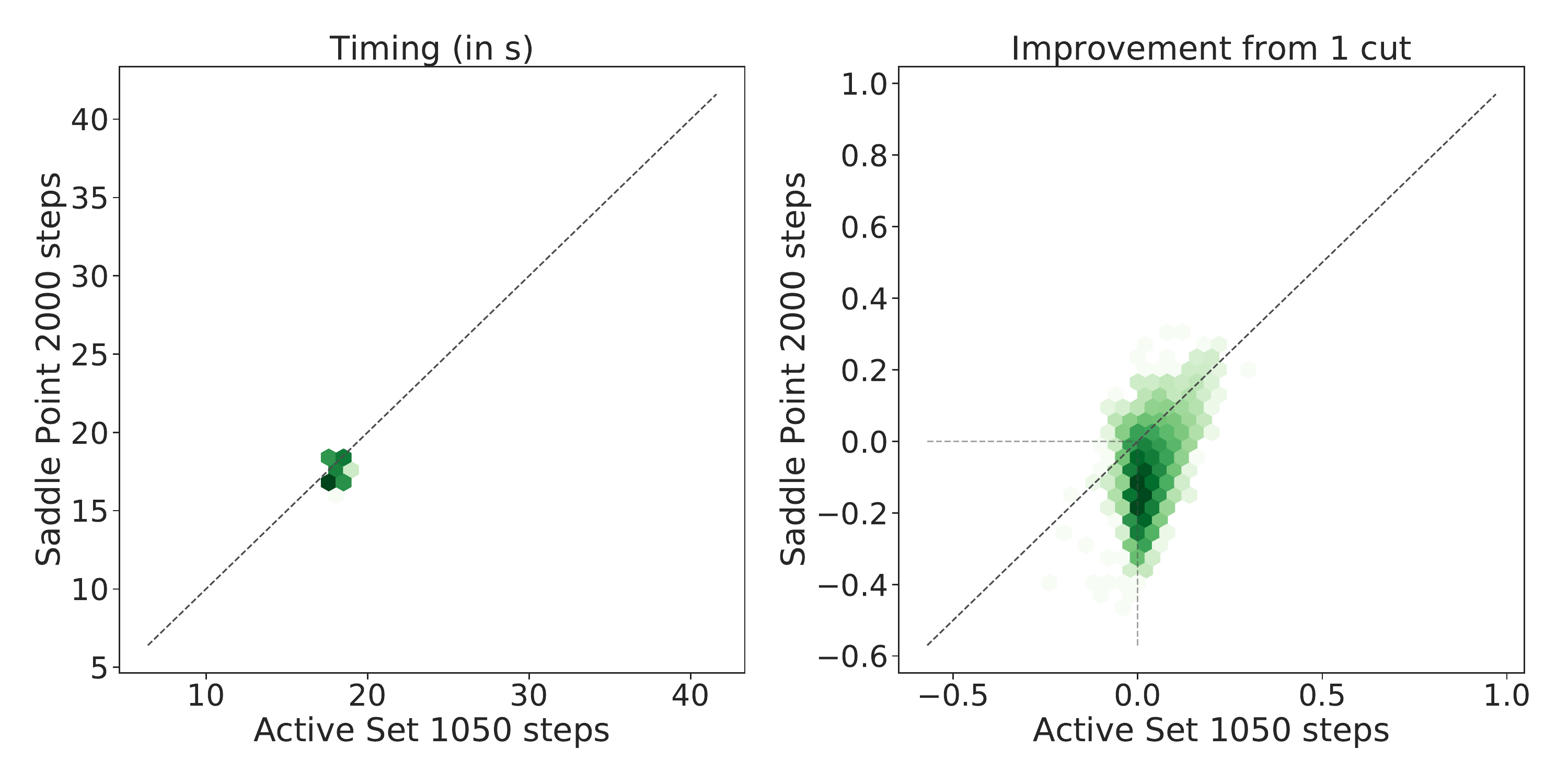}
	\end{minipage}
	\vspace{5pt}
	\begin{minipage}{.48\textwidth}
		\centering
		\includegraphics[width=\textwidth]{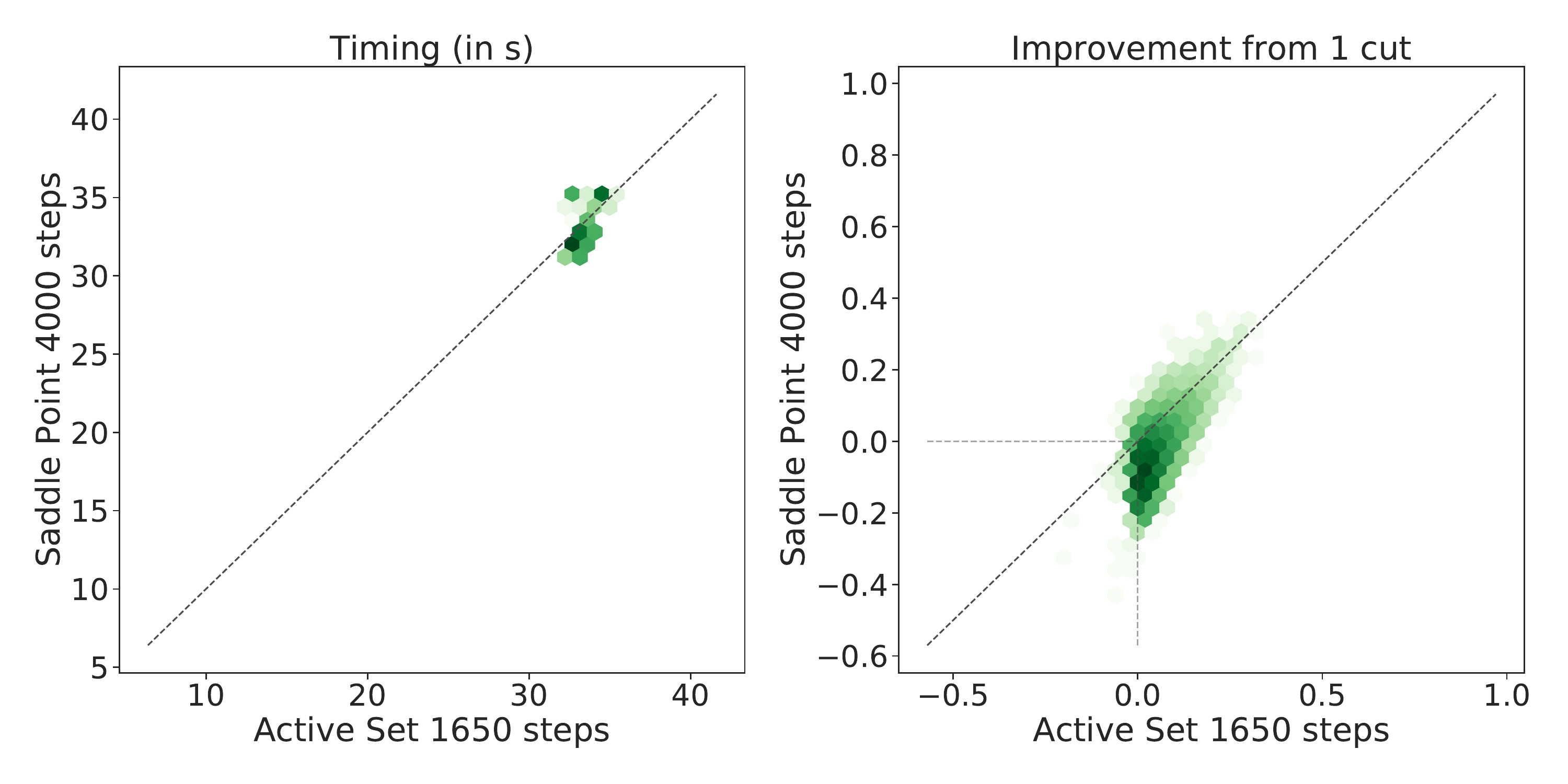}
	\end{minipage}
	\hspace{5pt}
	\begin{minipage}{.48\textwidth}
		\centering
		\includegraphics[width=\textwidth]{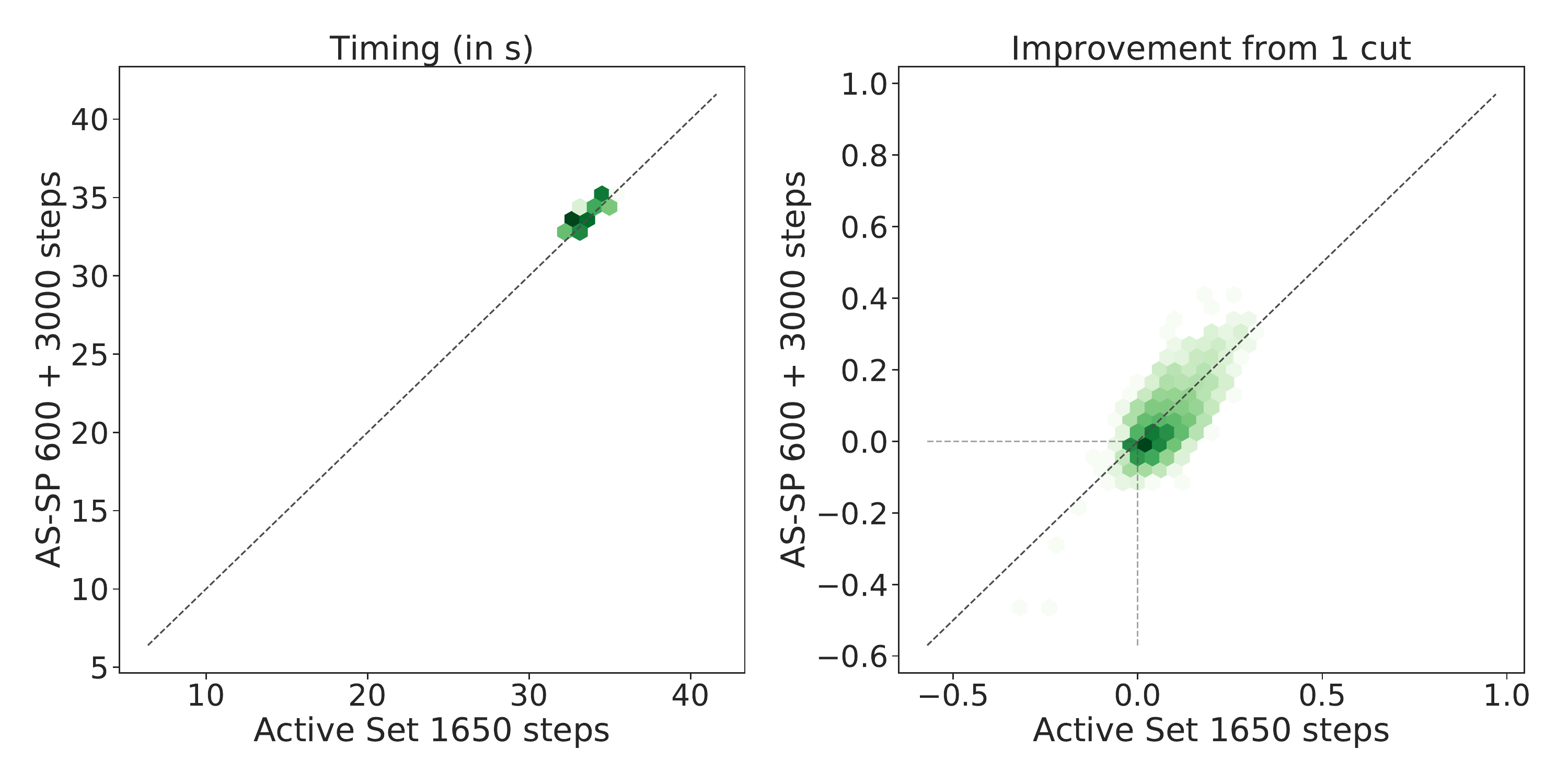}
	\end{minipage}
	\vspace{-15pt}
	\caption{Pointwise comparison between our proposed solvers on the data presented in figure \ref{fig:madry8}. Darker colour shades mean higher point density (on a logarithmic scale). The oblique dotted line corresponds to the equality. \label{fig:madry-as-vs-spfw}}
\end{figure*}

In addition to the SGD-trained network in $\S$\ref{sec:incomp_verif}, we now present results relative to the same architecture, trained with the adversarial training method by ~\citet{Madry2018} for robustness to perturbations of $\epsilon_{train} = 2/255$. Each adversarial sample for the training was obtained using 50 steps of projected gradient descent. 
For this network, we upper bound the vulnerability to perturbations with $\epsilon_{ver} = 2.7/255$.
Hyper-parameters are kept to the values tuned on the SGD-trained network from Section \ref{sec:incomp_verif}.

Figures \ref{fig:madry8}, \ref{fig:madry8-pointwise}, \ref{fig:madry-as-vs-spfw} confirm most of the observations carried out for the SGD-trained network in $\S$\ref{sec:incomp_verif}, with fewer variability around the bounds returned by Gurobi cut. 
Big-M is competitive with BDD+, and switching to Active Set after $500$ iterations results in much better bounds in the same time.
Increasing the computational budget for Active Set still results in better bounds than Gurobi cut in a fraction of its running time, even though the performance gap is on average smaller than on the SGD-trained network. As in Section  \ref{sec:incomp_verif} the gap between Saddle Point and Active Set, though larger here on average, decreases with the computational budget, and is further reduced when initializing with a few Active Set~iterations.
\begin{figure*}[t!]
	\vspace{-10pt}
	\centering
	\includegraphics[width=1\textwidth]{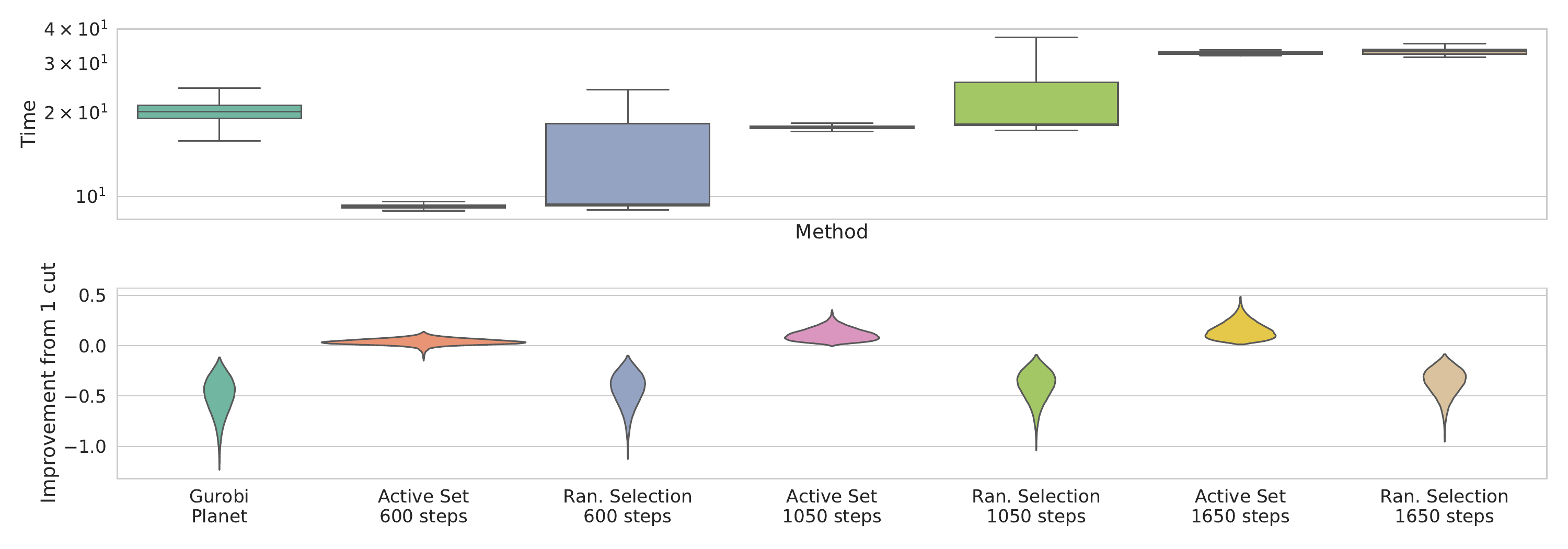}
	\vspace{-20pt}
	\caption{\label{fig:sgd8-selectionc} 
		Upper plot: distribution of runtime in seconds.
		Lower plot: difference with the bounds obtained by Gurobi with a cut from $\mathcal{A}_k$ per neuron; higher is better. Results for the SGD-trained network from \citet{BunelDP20}.  Sensitivity of Active Set to selection criterion (see $\S$\ref{sec:active-set-descr}).}
\end{figure*}
\begin{figure*}[b!]
	\centering
	\includegraphics[width=1\textwidth]{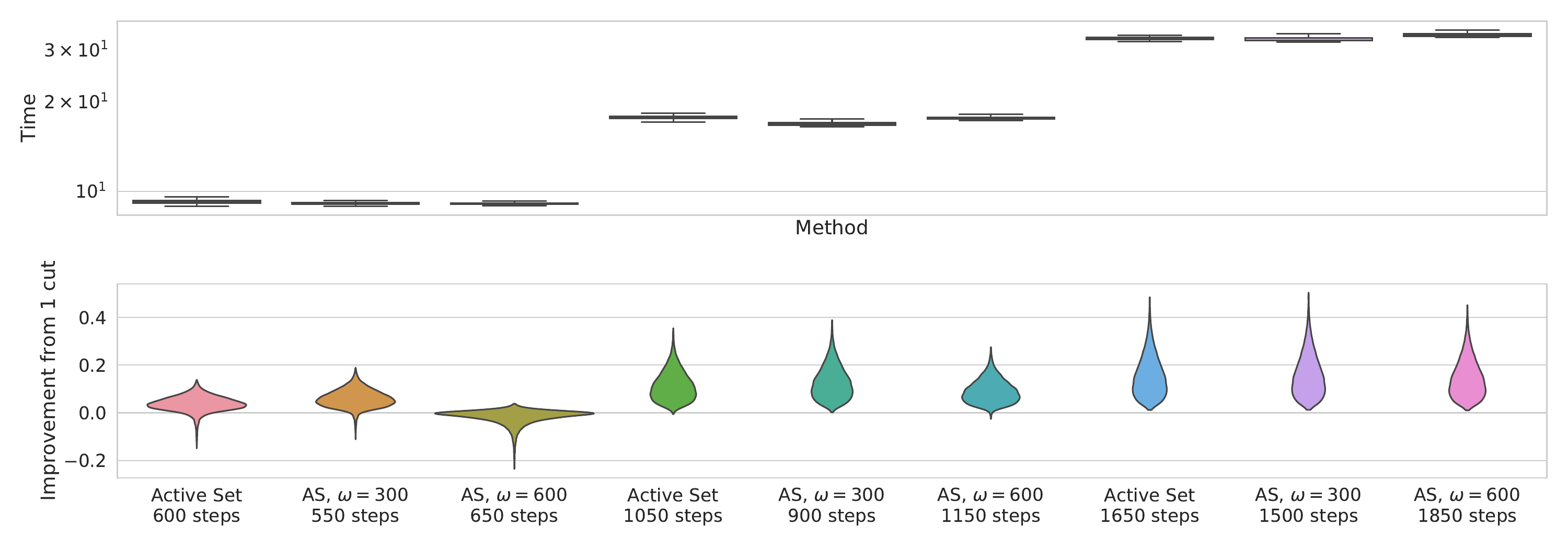}
	\vspace{-20pt}
	\caption{\label{fig:sgd8-selectioncf} 
		Upper plot: distribution of runtime in seconds.
		Lower plot: difference with the bounds obtained by Gurobi with a cut from $\mathcal{A}_k$ per neuron; higher is better. Results for the SGD-trained network from \citet{BunelDP20}. Sensitivity of Active Set to variable addition frequency $\omega$, with the selection criterion presented in $\S$\ref{sec:active-set-descr}.}
\end{figure*}

\subsection{Sensitivity of Active Set to Selection Criterion and Frequency}

In Section \ref{sec:active-set-descr}, we describe how to iteratively modify $\mathcal{B}$, the active set of dual variables on which our Active Set solver operates.  
In short, Active Set adds the variables corresponding to the output of oracle \eqref{eq:oracle} invoked at the primal minimiser of $\mathcal{L}_{\mathcal{B}}(\xb, \zb, \balpha, \bbeta_{\mathcal{B}})$, at a fixed frequency $\omega$.
We now investigate the empirical sensitivity of Active Set to both the selection criterion and the frequency of addition.

\begin{figure*}[t!]
	\centering
	\begin{subfigure}{\textwidth}
		\includegraphics[width=1\textwidth]{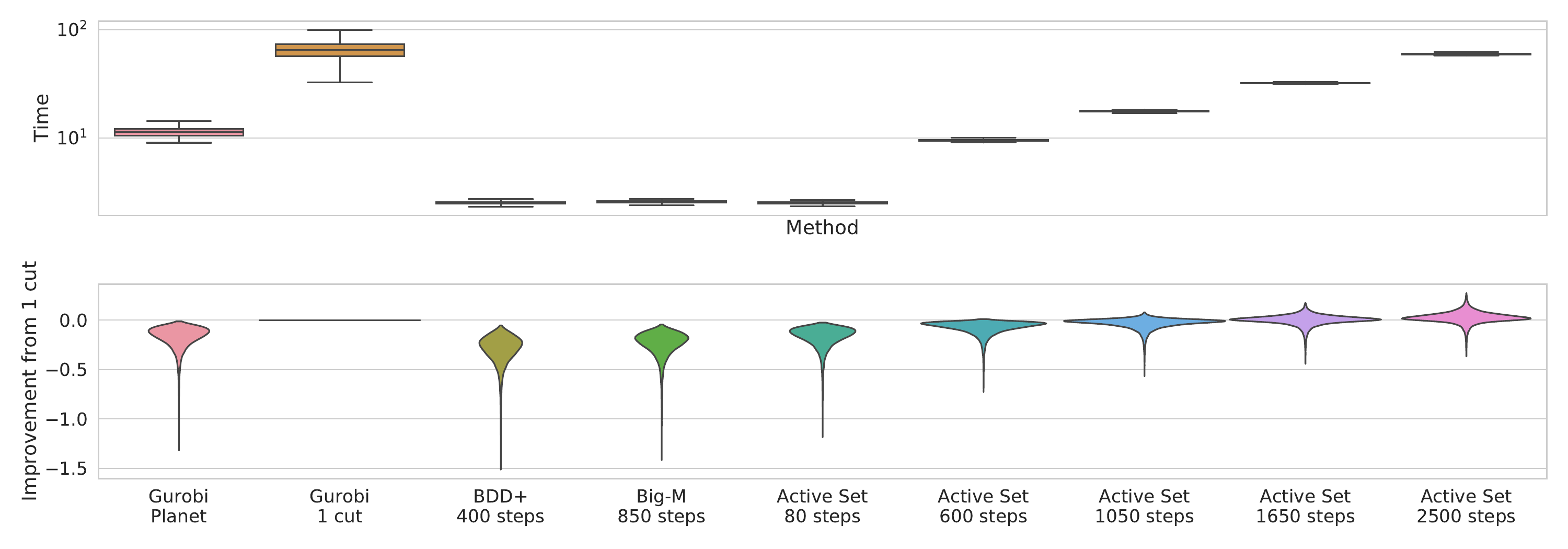}
		\vspace{-20pt}
		\subcaption{Speed-accuracy trade-offs of Active Set for different iteration ranges and computing devices.}
		\vspace{5pt}
	\end{subfigure}
	\begin{subfigure}{\textwidth}
		\includegraphics[width=1\textwidth]{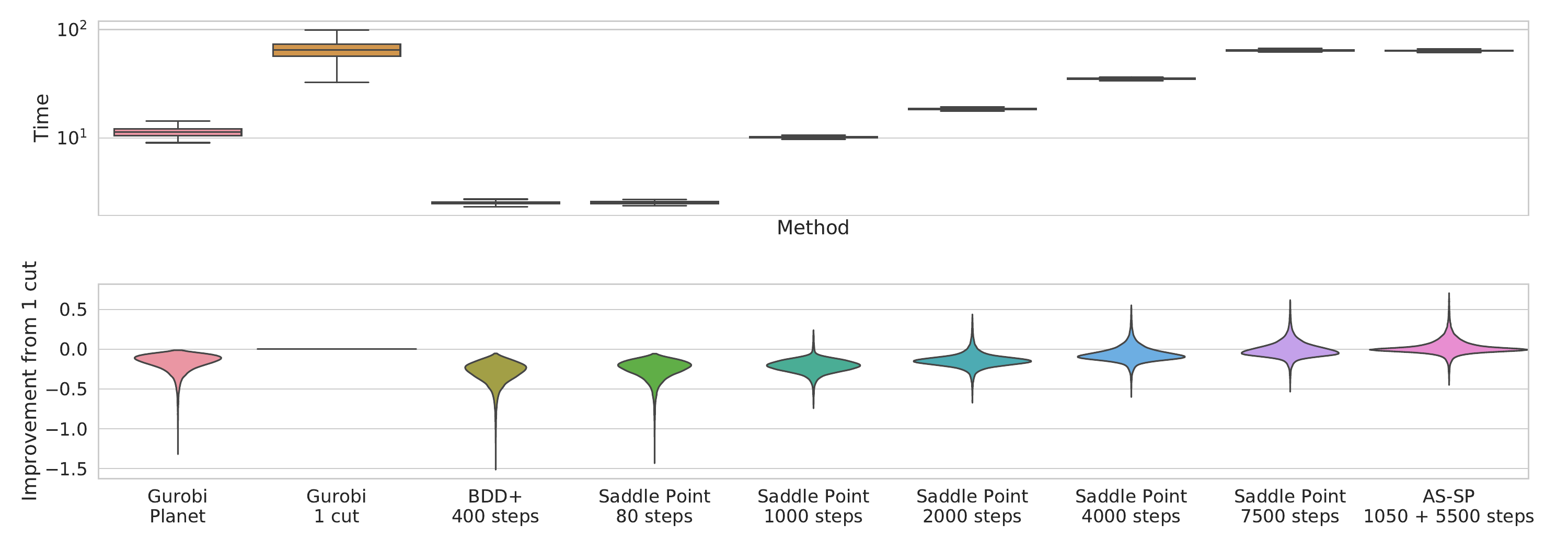}
		\vspace{-20pt}
		\subcaption{Speed-accuracy trade-offs of Saddle Point for different iteration ranges and computing devices.}
	\end{subfigure}
	\caption{\label{fig:mnistkw} 
		Upper bounds to the adversarial vulnerability for the MNIST network trained with the verified training algorithm by~\citet{Wong2018}, from \citet{Lu2020Neural}.
		Box plots: distribution of runtime in seconds.
		Violin plots: difference with the bounds obtained by Gurobi with a cut from $\mathcal{A}_k$ per neuron; higher is better. }
\end{figure*}

We test against \textbf{Ran. Selection}, a version of Active Set for which the variables to add are selected at random by uniformly sampling from the binary $I_k$ masks. 
As expected, Figure \ref{fig:sgd8-selectionc} shows that a good selection criterion is key to the efficiency of Active Set. In fact, random variable selection only marginally improves upon the Planet relaxation bounds, whereas the improvement becomes significant with our selection criterion from $\S$\ref{sec:active-set-descr}.

In addition, we investigate the sensitivity of Active Set (AS) to variable addition frequency $\omega$. In order to do so, we cap the maximum number of cuts at $7$ for all runs, and vary $\omega$ while keeping the time budget fixed (we test on three different time budgets). Figure \ref{fig:sgd8-selectioncf} compares the results for $\omega=450$ (Active Set), which were presented in $\S$\ref{sec:incomp_verif}, with the bounds obtained by setting $\omega=300$ and $\omega=600$ (respectively \textbf{AS $\boldsymbol{\omega=300}$} and \textbf{AS $\boldsymbol{\omega=600}$}). Our solver proves to be relatively robust to $\omega$ across all the three budgets, with the difference in obtained bounds decreasing with the number of iterations. Moreover, early cut addition tends to yield better bounds in the same time, suggesting that our selection criterion is effective before subproblem convergence.

\begin{figure*}[b!]
	\centering
	\begin{minipage}{.48\textwidth}
		\centering
		\includegraphics[width=\textwidth]{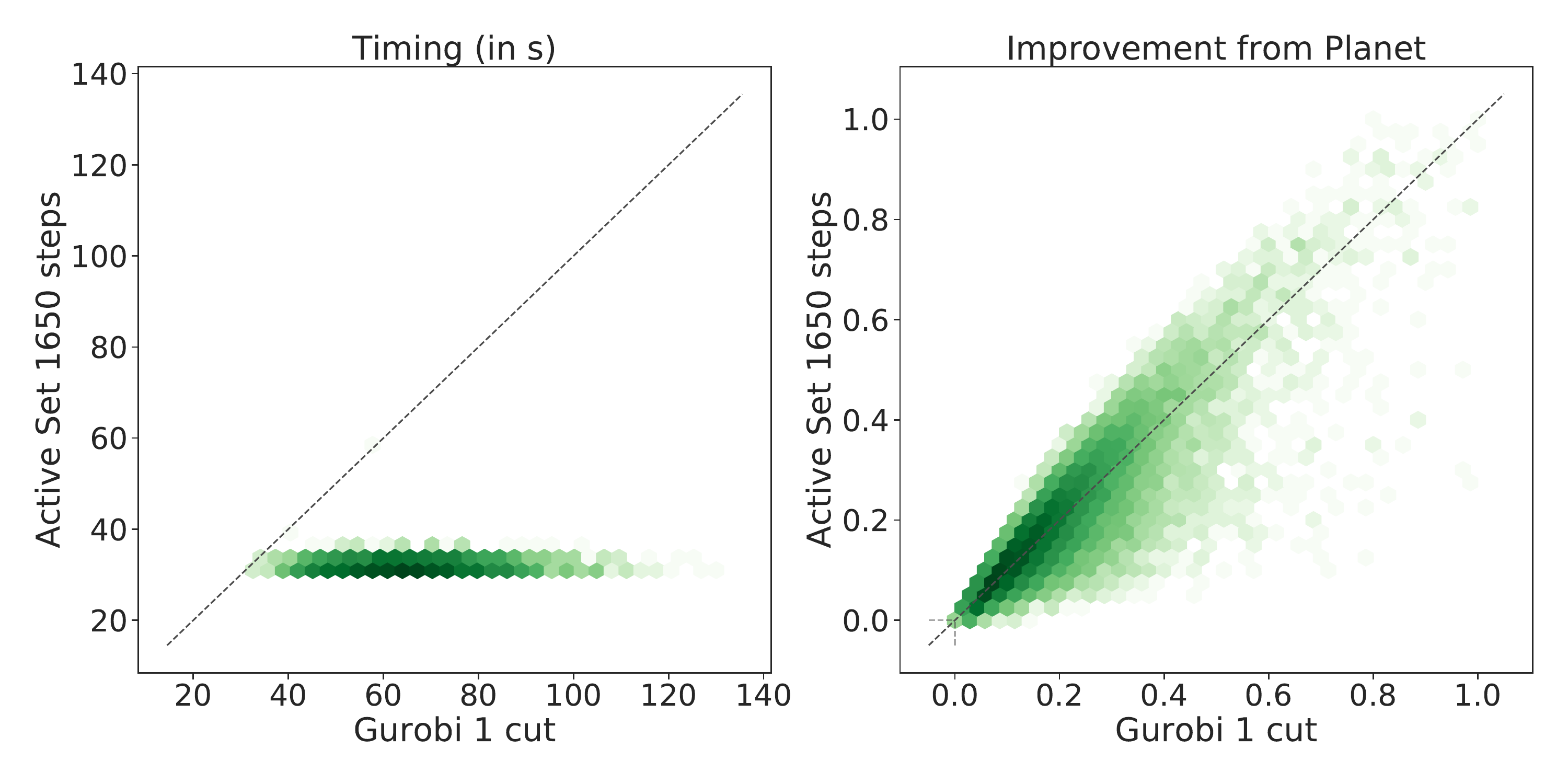}
	\end{minipage}
	\hspace{5pt}
	\begin{minipage}{.48\textwidth}
		\centering
		\includegraphics[width=\textwidth]{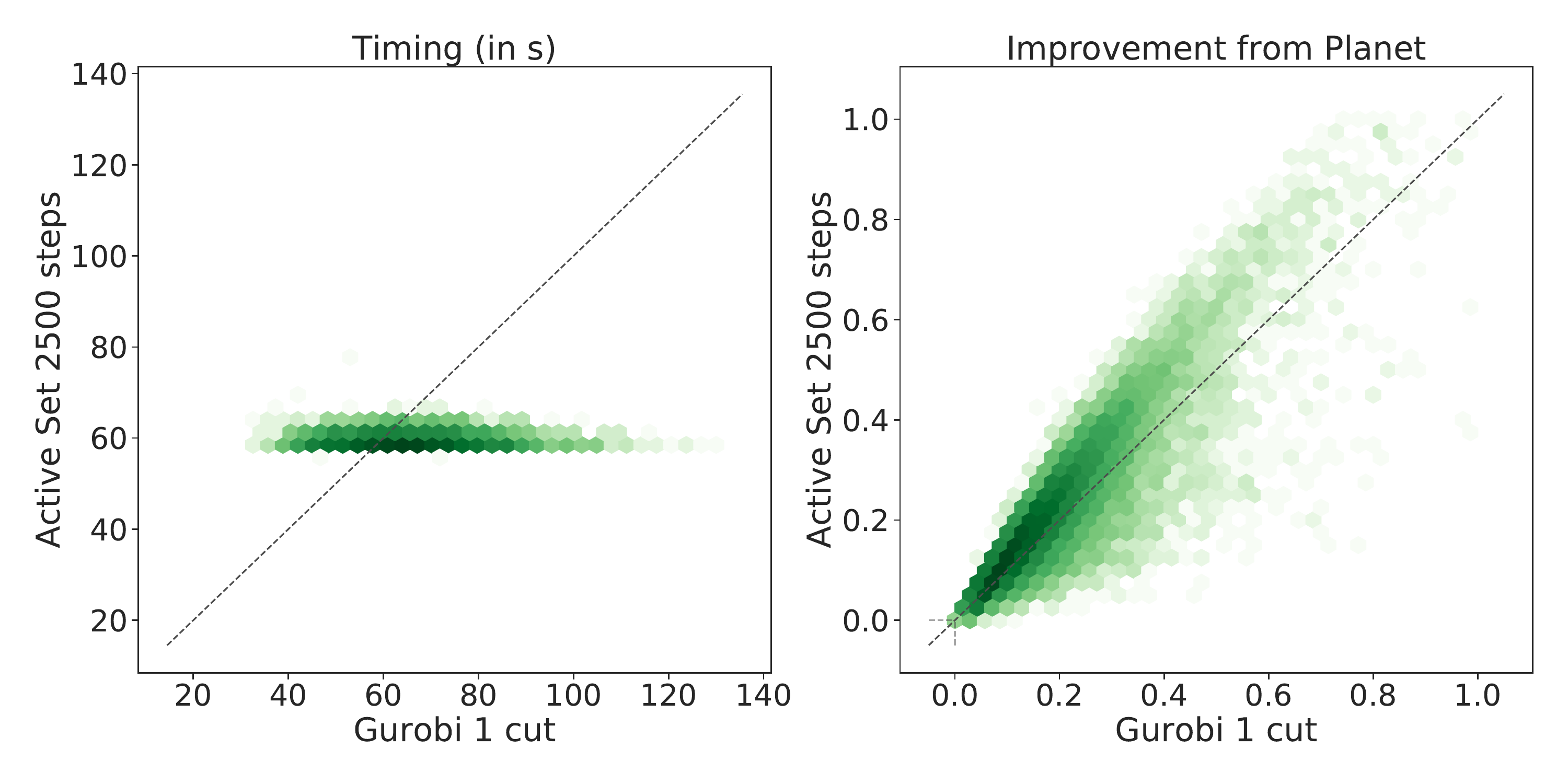}
	\end{minipage}
	\vspace{5pt}
	\begin{minipage}{.48\textwidth}
		\centering
		\includegraphics[width=\textwidth]{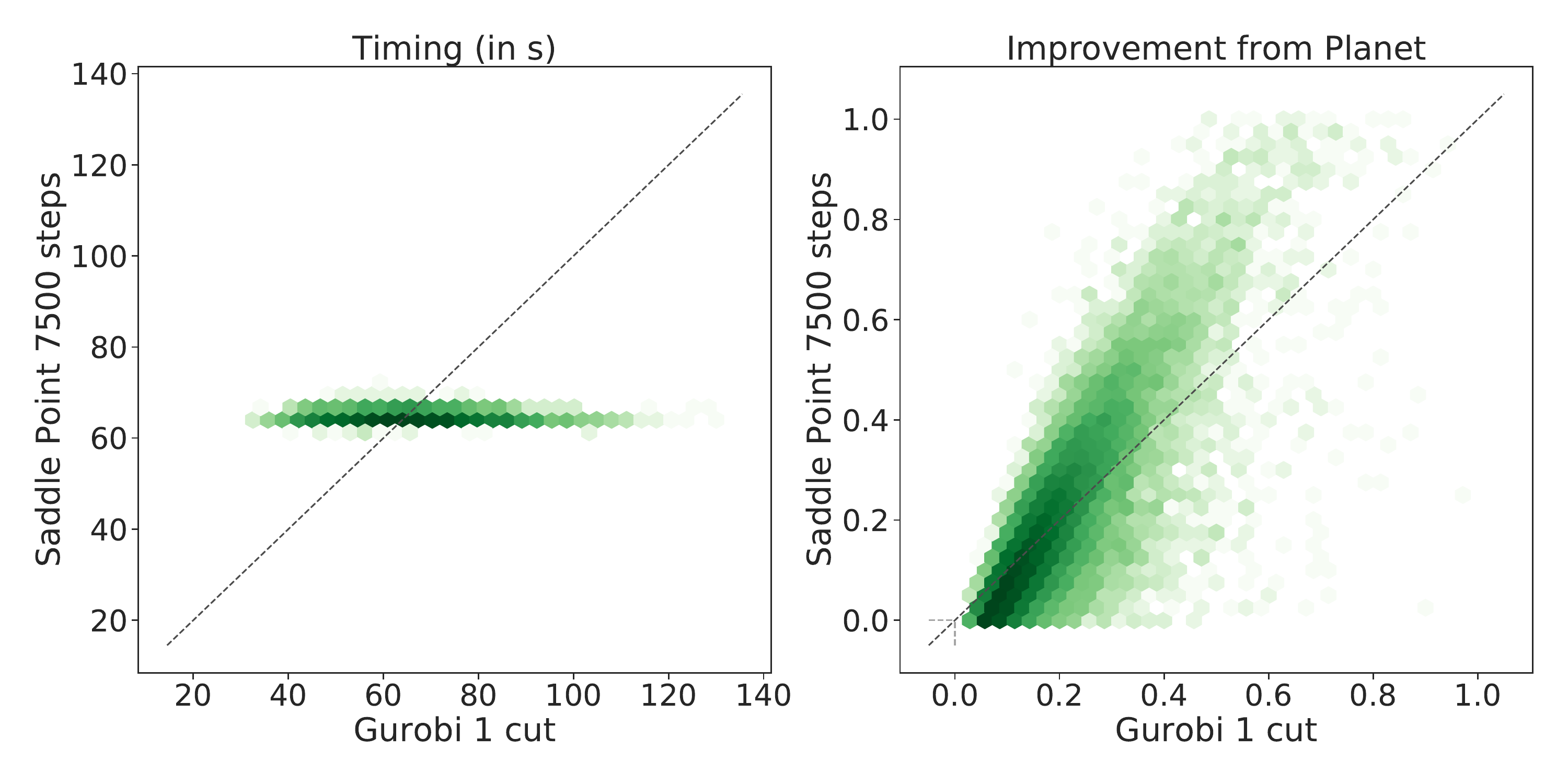}
	\end{minipage}
	\hspace{5pt}
	\begin{minipage}{.48\textwidth}
		\centering
		\includegraphics[width=\textwidth]{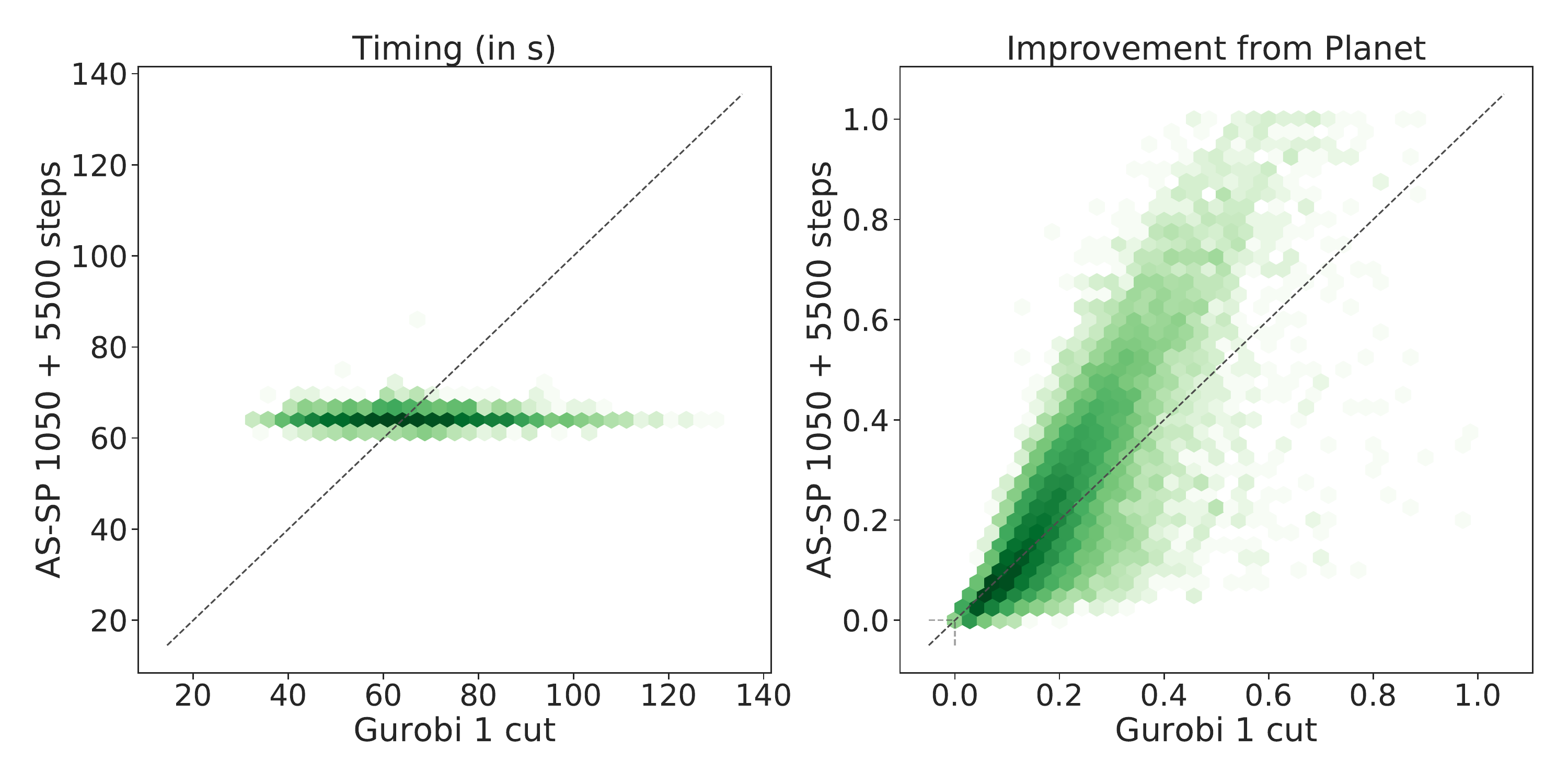}
	\end{minipage}
	\vspace{-15pt}
	\caption{Pointwise comparison for a subset of the methods on the data presented in Figure \ref{fig:mnistkw}. Comparison of runtime (left) and improvement from the Gurobi Planet bounds. For the latter, higher is better. Darker colour shades mean higher point density (on a logarithmic scale). The oblique dotted line corresponds to the equality. }
	\label{fig:mnist-pointwise}
	\vspace{-5pt}
\end{figure*}

\subsection{MNIST Incomplete Verification} \label{sec:incomplete-mnist}

We conclude the experimental appendix by presenting incomplete verification results (the experimental setup mirrors the one employed in Section~\ref{sec:incomp_verif} and appendix \ref{sec:incomplete-madry}) on the MNIST dataset~\citep{LeCun1998}. 

We report results on the ``wide" MNIST network from \citet{Lu2020Neural}, whose architecture is identical to the ``wide" network in Table \ref{tab:problem_size} except for the first layer, which has only one input channel to reflect the MNIST specification (the total number of ReLU activation units is $4804$).
As those employed for the complete verification experiments ($\S$\ref{sec:comp_verif}), and differently from the incomplete verification experiments in Section~\ref{sec:incomp_verif} and appendix \ref{sec:incomplete-madry}, the network was trained with the verified training method by \citet{Wong2018}. We compute the vulnerability to $\epsilon_{ver}=0.15$ on the first $820$ images of the MNIST test set.
All hyper-parameters are kept to the values employed for the CIFAR-10 networks, except the Big-M step size, which was linearly decreased from $10^{-1}$ to $10^{-3}$, and the weight of the proximal terms for BDD+, which was linearly increased from $1$ to $50$.

As seen on the CIFAR-10 networks, Figures \ref{fig:mnistkw}, \ref{fig:mnist-pointwise} show that our solvers for problem \eqref{eq:primal-anderson} (Active Set and Saddle Point) yield comparable or better bounds than Gurobi 1 cut in less average runtime. However, more iterations are required to reach the same relative bound improvement over Gurobi 1 cut (for Active Set, $2500$ as opposed to $600$ in Figures \ref{fig:sgd8}, \ref{fig:madry8}). 
Finally, the smaller average gap between the bounds of Gurobi Planet and Gurobi 1 cut (especially with respect to Figure \ref{fig:sgd8}) suggests that the relaxation by \citet{Anderson2020} is less effective on this MNIST benchmark.

\clearpage 
\bibliography{standardstrings,oval,explp}

\begin{thebibliography}{44}
\providecommand{\natexlab}[1]{#1}
\providecommand{\url}[1]{\texttt{#1}}
\expandafter\ifx\csname urlstyle\endcsname\relax
  \providecommand{\doi}[1]{doi: #1}\else
  \providecommand{\doi}{doi: \begingroup \urlstyle{rm}\Url}\fi

\bibitem[Anderson et~al.(2019)Anderson, Huchette, Tjandraatmadja, and
  Vielma]{Anderson2019}
Ross Anderson, Joey Huchette, Christian Tjandraatmadja, and Juan~Pablo Vielma.
\newblock Strong mixed-integer programming formulations for trained neural
  networks.
\newblock In \emph{Conference on Integer Programming and Combinatorial
  Optimization}, 2019.

\bibitem[Anderson et~al.(2020)Anderson, Huchette, Ma, Tjandraatmadja, and
  Vielma]{Anderson2020}
Ross Anderson, Joey Huchette, Will Ma, Christian Tjandraatmadja, and Juan~Pablo
  Vielma.
\newblock Strong mixed-integer programming formulations for trained neural
  networks.
\newblock \emph{Mathematical Programming}, 2020.

\bibitem[Balunovic and Vechev(2020)]{Balunovic2020}
Mislav Balunovic and Martin Vechev.
\newblock Adversarial training and provable defenses: Bridging the gap.
\newblock In \emph{International Conference on Learning Representations}, 2020.

\bibitem[Botoeva et~al.(2020)Botoeva, Kouvaros, Kronqvist, Lomuscio, and
  Misener]{Botoeva2020}
Elena Botoeva, Panagiotis Kouvaros, Jan Kronqvist, Alessio Lomuscio, and Ruth
  Misener.
\newblock Efficient verification of {R}e{LU}-based neural networks via
  dependency analysis.
\newblock In \emph{AAAI Conference on Artificial Intelligence}, 2020.

\bibitem[Bunel et~al.(2018)Bunel, Turkaslan, Torr, Kohli, and Kumar]{Bunel2018}
Rudy Bunel, Ilker Turkaslan, Philip~HS Torr, Pushmeet Kohli, and M~Pawan Kumar.
\newblock A unified view of piecewise linear neural network verification.
\newblock In \emph{Neural Information Processing Systems}, 2018.

\bibitem[Bunel et~al.(2020{\natexlab{a}})Bunel, De~Palma, Desmaison, Dvijotham,
  Kohli, Torr, and Kumar]{BunelDP20}
Rudy Bunel, Alessandro De~Palma, Alban Desmaison, Krishnamurthy Dvijotham,
  Pushmeet Kohli, Philip~HS Torr, and M~Pawan Kumar.
\newblock Lagrangian decomposition for neural network verification.
\newblock In \emph{Conference on Uncertainty in Artificial Intelligence},
  2020{\natexlab{a}}.

\bibitem[Bunel et~al.(2020{\natexlab{b}})Bunel, Lu, Turkaslan, Kohli, Torr, and
  Kumar]{Bunel2020}
Rudy Bunel, Jingyue Lu, Ilker Turkaslan, P~Kohli, P~Torr, and M~Pawan Kumar.
\newblock Branch and bound for piecewise linear neural network verification.
\newblock \emph{Journal of Machine Learning Research}, 21\penalty0 (2020),
  2020{\natexlab{b}}.

\bibitem[Cand{\`{e}}s et~al.(2008)Cand{\`{e}}s, Wakin, and Boyd]{Cands2008}
Emmanuel~J. Cand{\`{e}}s, Michael~B. Wakin, and Stephen~P. Boyd.
\newblock Enhancing sparsity by reweighted $\ell$ 1 minimization.
\newblock \emph{Journal of Fourier Analysis and Applications}, 2008.

\bibitem[Carlini and Wagner(2017)]{Carlini2017}
Nicholas Carlini and David Wagner.
\newblock Towards evaluating the robustness of neural networks.
\newblock In \emph{2017 IEEE Symposium on Security and Privacy}, 2017.

\bibitem[De~Palma et~al.(2021)De~Palma, Behl, Bunel, Torr, and
  Kumar]{DePalma2021}
Alessandro De~Palma, Harkirat~Singh Behl, Rudy Bunel, Philip H.~S. Torr, and
  M.~Pawan Kumar.
\newblock Scaling the convex barrier with active sets.
\newblock In \emph{International Conference on Learning Representations}, 2021.

\bibitem[Dvijotham et~al.(2018)Dvijotham, Stanforth, Gowal, Mann, and
  Kohli]{Dvijotham2018}
Krishnamurthy Dvijotham, Robert Stanforth, Sven Gowal, Timothy Mann, and
  Pushmeet Kohli.
\newblock A dual approach to scalable verification of deep networks.
\newblock In \emph{Uncertainty in Artificial Intelligence}, 2018.

\bibitem[Dvijotham et~al.(2020)Dvijotham, Stanforth, Gowal, Qin, De, and
  Kohli]{Dvijotham2019}
Krishnamurthy Dvijotham, Robert Stanforth, Sven Gowal, Chongli Qin, Soham De,
  and Pushmeet Kohli.
\newblock Efficient neural network verification with exactness
  characterization.
\newblock In \emph{Uncertainty in Artificial Intelligence}, 2020.

\bibitem[Ehlers(2017)]{Ehlers2017}
Ruediger Ehlers.
\newblock Formal verification of piece-wise linear feed-forward neural
  networks.
\newblock \emph{Automated Technology for Verification and Analysis}, 2017.

\bibitem[Frank and Wolfe(1956)]{Frank1956}
Marguerite Frank and Philip Wolfe.
\newblock An algorithm for quadratic programming.
\newblock \emph{Naval Research Logistics Quarterly}, 1956.

\bibitem[Gidel et~al.(2017)Gidel, Jebara, and Lacoste-Julien]{Gidel17}
Gauthier Gidel, Tony Jebara, and Simon Lacoste-Julien.
\newblock {Frank-Wolfe} algorithms for saddle point problems.
\newblock In \emph{International Conference on Artificial Intelligence and
  Statistics}, 2017.

\bibitem[Gowal et~al.(2018)Gowal, Dvijotham, Stanforth, Bunel, Qin, Uesato,
  Mann, and Kohli]{Gowal2018}
Sven Gowal, Krishnamurthy Dvijotham, Robert Stanforth, Rudy Bunel, Chongli Qin,
  Jonathan Uesato, Timothy Mann, and Pushmeet Kohli.
\newblock On the effectiveness of interval bound propagation for training
  verifiably robust models.
\newblock In \emph{SECML NeurIPS}, 2018.

\bibitem[Gurobi~Optimization(2020)]{gurobi-custom}
LLC Gurobi~Optimization.
\newblock Gurobi optimizer reference manual, 2020.
\newblock URL \url{http://www.gurobi.com}.

\bibitem[Jaggi(2013)]{Jaggi2013}
Martin Jaggi.
\newblock Revisiting {Frank-Wolfe}: Projection-free sparse convex optimization.
\newblock In \emph{International Conference on Machine Learning}, 2013.

\bibitem[Jeroslow(1987)]{Jeroslow}
R.~G. Jeroslow.
\newblock Representability in mixed integer programming, i: Characterization
  results.
\newblock \emph{Discrete Applied Mathematics}, 1987.

\bibitem[Katz et~al.(2017)Katz, Barrett, Dill, Julian, and
  Kochenderfer]{Katz2017}
Guy Katz, Clark Barrett, David Dill, Kyle Julian, and Mykel Kochenderfer.
\newblock Reluplex: An efficient {SMT} solver for verifying deep neural
  networks.
\newblock In \emph{International Conference on Computer-Aided Verification},
  2017.

\bibitem[Kingma and Ba(2015)]{Kingma2015}
Diederik~P. Kingma and Jimmy Ba.
\newblock Adam: A method for stochastic optimization.
\newblock \emph{International Conference on Learning Representations}, 2015.

\bibitem[Krizhevsky and Hinton(2009)]{CIFAR10}
A.~Krizhevsky and G.~Hinton.
\newblock Learning multiple layers of features from tiny images.
\newblock \emph{Master's thesis, Department of Computer Science, University of
  Toronto}, 2009.

\bibitem[Kurakin et~al.(2017)Kurakin, Goodfellow, and Bengio]{Kurakin2017}
Alexey Kurakin, Ian~J. Goodfellow, and Samy Bengio.
\newblock Adversarial machine learning at scale.
\newblock In \emph{International Conference on Learning Representations}, 2017.

\bibitem[LeCun et~al.(1998)LeCun, Bottou, Bengio, and Haffner]{LeCun1998}
Yann LeCun, L{\'e}on Bottou, Yoshua Bengio, and Patrick Haffner.
\newblock Gradient-based learning applied to document recognition.
\newblock \emph{IEEE}, 1998.

\bibitem[Lemar{\'e}chal(2001)]{lemarechal2001lagrangian}
Claude Lemar{\'e}chal.
\newblock Lagrangian relaxation.
\newblock In \emph{Computational combinatorial optimization}, pages 112--156.
  Springer, 2001.

\bibitem[Lu and Kumar(2020)]{Lu2020Neural}
Jingyue Lu and M~Pawan Kumar.
\newblock Neural network branching for neural network verification.
\newblock In \emph{International Conference on Learning Representations}, 2020.

\bibitem[Madry et~al.(2018)Madry, Makelov, Schmidt, Tsipras, and
  Vladu]{Madry2018}
Aleksander Madry, Aleksandar Makelov, Ludwig Schmidt, Dimitris Tsipras, and
  Adrian Vladu.
\newblock Towards deep learning models resistant to adversarial attacks.
\newblock In \emph{International Conference on Learning Representations}, 2018.

\bibitem[Paszke et~al.(2017)Paszke, Gross, Chintala, Chanan, Yang, DeVito, Lin,
  Desmaison, Antiga, and Lerer]{Paszke2017}
Adam Paszke, Sam Gross, Soumith Chintala, Gregory Chanan, Edward Yang, Zachary
  DeVito, Zeming Lin, Alban Desmaison, Luca Antiga, and Adam Lerer.
\newblock Automatic differentiation in pytorch.
\newblock In \emph{NIPS Autodiff Workshop}, 2017.

\bibitem[Raghunathan et~al.(2018)Raghunathan, Steinhardt, and
  Liang]{Raghunathan2018}
Aditi Raghunathan, Jacob Steinhardt, and Percy~S Liang.
\newblock Semidefinite relaxations for certifying robustness to adversarial
  examples.
\newblock In \emph{Neural Information Processing Systems}, 2018.

\bibitem[Salman et~al.(2019)Salman, Yang, Zhang, Hsieh, and Zhang]{Salman2019}
Hadi Salman, Greg Yang, Huan Zhang, Cho-Jui Hsieh, and Pengchuan Zhang.
\newblock A convex relaxation barrier to tight robustness verification of
  neural networks.
\newblock In \emph{Neural Information Processing Systems}, 2019.

\bibitem[Sherali and Choi(1996)]{Sherali1996}
Hanif~D. Sherali and Gyunghyun Choi.
\newblock Recovery of primal solutions when using subgradient optimization
  methods to solve lagrangian duals of linear programs.
\newblock \emph{Operations Research Letters}, 1996.

\bibitem[Singh et~al.(2018)Singh, Gehr, Mirman, P{\"u}schel, and
  Vechev]{Singh2018a}
Gagandeep Singh, Timon Gehr, Matthew Mirman, Markus P{\"u}schel, and Martin
  Vechev.
\newblock Fast and effective robustness certification.
\newblock In \emph{Neural Information Processing Systems}, 2018.

\bibitem[Singh et~al.(2019{\natexlab{a}})Singh, Ganvir, P{\"u}schel, and
  Vechev]{Singh2019b}
Gagandeep Singh, Rupanshu Ganvir, Markus P{\"u}schel, and Martin Vechev.
\newblock Beyond the single neuron convex barrier for neural network
  certification.
\newblock In \emph{Neural Information Processing Systems}, 2019{\natexlab{a}}.

\bibitem[Singh et~al.(2019{\natexlab{b}})Singh, Gehr, P{\"u}schel, and
  Vechev]{Singh2019a}
Gagandeep Singh, Timon Gehr, Markus P{\"u}schel, and Martin Vechev.
\newblock Boosting robustness certification of neural networks.
\newblock In \emph{International Conference on Learning Representations},
  2019{\natexlab{b}}.

\bibitem[Singh et~al.(2020)Singh, Maurer, M{\"u}ller, Mirman, Gehr, Hoffmann,
  Tsankov, Drachsler~Cohen, P{\"u}schel, and Vechev]{eran}
Gagandeep Singh, Jonathan Maurer, Christoph M{\"u}ller, Matthew Mirman, Timon
  Gehr, Adrian Hoffmann, Petar Tsankov, Dana Drachsler~Cohen, Markus
  P{\"u}schel, and Martin Vechev.
\newblock {ETH} robustness analyzer for neural networks ({ERAN}).
\newblock 2020.
\newblock URL \url{https://github.com/eth-sri/eran}.

\bibitem[Szegedy et~al.(2014)Szegedy, Zaremba, Sutskever, Bruna, Erhan,
  Goodfellow, and Fergus]{Szegedy2014}
Christian Szegedy, Wojciech Zaremba, Ilya Sutskever, Joan Bruna, Dumitru Erhan,
  Ian Goodfellow, and Rob Fergus.
\newblock Intriguing properties of neural networks.
\newblock In \emph{International Conference on Learning Representations}, 2014.

\bibitem[Tjandraatmadja et~al.(2020)Tjandraatmadja, Anderson, Huchette, Ma,
  Patel, and Vielma]{Tjandraatmadja}
Christian Tjandraatmadja, Ross Anderson, Joey Huchette, Will Ma, Krunal Patel,
  and Juan~Pablo Vielma.
\newblock The convex relaxation barrier, revisited: Tightened single-neuron
  relaxations for neural network verification.
\newblock \emph{arXiv preprint arXiv:2006.14076}, 2020.

\bibitem[Tjeng et~al.(2019)Tjeng, Xiao, and Tedrake]{Tjeng2019}
Vincent Tjeng, Kai Xiao, and Russ Tedrake.
\newblock Evaluating robustness of neural networks with mixed integer
  programming.
\newblock In \emph{International Conference on Learning Representations}, 2019.

\bibitem[VNN-COMP(2020)]{vnn-comp}
VNN-COMP.
\newblock International verification of neural networks competition
  ({VNN-COMP}).
\newblock \emph{Verification of Neural Networks workshop at the International
  Conference on Computer-Aided Verification}, 2020.
\newblock URL \url{https://sites.google.com/view/vnn20/vnncomp}.

\bibitem[Wang et~al.(2018)Wang, Pei, Whitehouse, Yang, and Jana]{Wang2018b}
Shiqi Wang, Kexin Pei, Justin Whitehouse, Junfeng Yang, and Suman Jana.
\newblock Efficient formal safety analysis of neural networks.
\newblock In \emph{Neural Information Processing Systems}, 2018.

\bibitem[Webb et~al.(2019)Webb, Rainforth, Teh, and Kumar]{Webb2019}
Stefan Webb, Tom Rainforth, Yee~Whye Teh, and M~Pawan Kumar.
\newblock A statistical approach to assessing neural network robustness.
\newblock In \emph{International Conference on Learning Representations}, 2019.

\bibitem[Weng et~al.(2018)Weng, Zhang, Chen, Song, Hsieh, Boning, Dhillon, and
  Daniel]{Weng2018}
Tsui-Wei Weng, Huan Zhang, Hongge Chen, Zhao Song, Cho-Jui Hsieh, Duane Boning,
  Inderjit~S Dhillon, and Luca Daniel.
\newblock Towards fast computation of certified robustness for relu networks.
\newblock In \emph{International Conference on Machine Learning}, 2018.

\bibitem[Wong and Kolter(2018)]{Wong2018}
Eric Wong and Zico Kolter.
\newblock Provable defenses against adversarial examples via the convex outer
  adversarial polytope.
\newblock In \emph{International Conference on Machine Learning}, 2018.

\bibitem[Xu et~al.(2021)Xu, Zhang, Wang, Wang, Jana, Lin, and
  Hsieh]{xu2021fast}
Kaidi Xu, Huan Zhang, Shiqi Wang, Yihan Wang, Suman Jana, Xue Lin, and Cho-Jui
  Hsieh.
\newblock Fast and complete: Enabling complete neural network verification with
  rapid and massively parallel incomplete verifiers.
\newblock In \emph{International Conference on Learning Representations}, 2021.

\end{thebibliography}

\end{document}